\documentclass[twoside,11pt]{article}

\usepackage{jmlr2e}

\usepackage{graphicx,epstopdf,amsfonts,,amsmath,algorithmic,hyperref, enumerate, paralist, bm,mathabx}
\usepackage{latexsym}
\usepackage{bm} 
\usepackage{amsmath,amssymb,amsfonts}
\usepackage{url}
\usepackage{enumitem}

\usepackage{bm}
\usepackage{algorithmic}
\usepackage{algorithm}
\usepackage{wrapfig}

\newcommand{\sX}{{\mathcal X}}

\newcommand{\sE}{\mathcal{E}}
\newcommand{\sN}{\mathcal{N}}
\newcommand{\sP}{\mathcal{P}}

\newcommand{\sC}{\mathcal{C}}

\newcommand{\bmtau}{\bm{\tau}}
\newcommand{\btau}{\bm{\tau}}

\newcommand{\bmeta}{\bm{\eta}}

\newcommand{\bgamma}{\bm{\gamma}}
\newcommand{\bpi}{\bm{\pi}}
\newcommand{\bmu}{\bm{\mu}}
\newcommand{\brho}{\bm{\rho}}
\newcommand{\blambda}{\bm{\lambda}}
\newcommand{\bPi}{\bm{\Pi}}

\newcommand{\ba}{\bm{a}}

\newcommand{\be}{\bm{e}}

\newcommand{\bn}{\bm{n}}

\newcommand{\bv}{\bm{v}}
\newcommand{\bw}{\bm{w}}
\newcommand{\bx}{\bm{x}}

\newcommand{\bC}{\bm{C}}

\newcommand{\bF}{\bm{F}}

\newcommand{\bH}{\bm{H}}

\newcommand{\bP}{\bm{P}}

\newcommand{\bbR}{\mathbb{R}} 

\newcommand{\bbZ}{\mathbb{Z}}

\newcommand{\si}{\Delta}

\newcommand{\supp}{\mathop{\mathrm{supp}}} 
\newcommand{\argmin}{\operatornamewithlimits{arg\ min}}

\newcommand{\aff}{\mathop{\text{aff}}} 
\newcommand{\co}{\mathop { \textnormal{conv}}} 
\newcommand{\spa}{\mathop { \textnormal{span}}} 
\newcommand{\rank}{\mathop { \textnormal{rank}}}
\newcommand{\unif}{\mathop { \text{unif}}}
\newcommand{\range}{\mathop { \text{range}}}

\newcommand{\sol}{\mathop { \textnormal{MultiResidue}}} 

\newcommand{\ind}[1]{\bm{1}_{\{#1\}}} 

\newcommand{\set}[1]{\{#1\}}

\newcommand{\paren}[1]{\left(#1\right)}
\newcommand{\norm}[1]{\left\|#1\right\|}

\renewcommand{\vec}[1]{\bm{#1}}
\renewcommand{\set}[1]{\{#1\}}

\newcommand{\p}[1]{\tilde{P}_#1}

\newcommand{\est}[1]{\widehat{#1}}
\newcommand{\ed}[1]{#1^\dagger} 
\newcommand{\vc}{\gamma_{\vec{n}}} 

\newtheorem{thm}{Theorem}
\newtheorem{prop}{Proposition}
\newtheorem{cor}{Corollary}
\newtheorem{defn}{Definition}

\usepackage{lastpage}
\jmlrheading{20}{2019}{1-\pageref{LastPage}}{9/17; Revised
11/18}{1/19}{17-576}{Julian Katz-Samuels, Gilles Blanchard and Clayton Scott}
\ShortHeadings{Decontamination of Mutual Contamination Models}{Katz-Samuels, Blanchard, and Scott}

\firstpageno{1}

\begin{document}

\title{Decontamination of Mutual Contamination Models}

\author{\name Julian Katz-Samuels \email jkatzsam@umich.edu \\
			\addr Department of Electrical Engineering and Computer Science \\
			University of Michigan \\
			Ann Arbor, MI 48109-2122 USA
			\AND
			\name Gilles Blanchard \email gilles.blanchard@math.uni-potsdam.de \\
			\addr Universit\"{a}t Potsdam, Institut f\"{u}r Mathematik  \\
			D-14476 Potsdam, Germany 
			\AND
			\name Clayton Scott \email clayscot@umich.edu \\
			\addr Department of Electrical Engineering and Computer Science \\
			University of Michigan \\
			Ann Arbor, MI 48109-2122 USA}

\editor{Inderjit Dhillon}

\maketitle

\begin{abstract}
Many machine learning problems can be characterized by \emph{mutual contamination models}. In these problems, one observes several random samples from different convex combinations of a set of unknown base distributions and the goal is to infer these base distributions. This paper considers the general setting where the base distributions are defined on arbitrary probability spaces. We examine three popular machine learning problems that arise in this general setting: multiclass classification with label noise, demixing of mixed membership models, and classification with partial labels. In each case, we give sufficient conditions for identifiability and present algorithms for the infinite and finite sample settings, with associated performance guarantees.
\end{abstract}


\section{Introduction}
In many machine learning problems, the learner observes several random samples from different mixtures of unknown base distributions, with unknown mixing weights, and the goal is to infer these base distributions. Examples include binary classification with label noise, multiclass classification with label noise, classification with partial labels, and topic modeling. The goal of this paper is to develop a unified framework and set of tools to study statistical properties of these problems in a very general setting.

To this end, we use the general framework of mutual contamination models  \citep{blanchard2014}. In a mutual contamination model, there are $L$ distributions $P_1,\ldots, P_L$ called \emph{base distributions}. The learner observes $M$ random samples
\begin{align}
X^i_1, \ldots, X_{n_i}^i \overset{i.i.d.}{\sim} \tilde{P}_i = \sum_{j =1}^L \pi_{i, j} P_j \label{data_model}
\end{align}
where $i=1,\ldots, M$, $\pi_{i,j} \geq 0$, and $\sum_{j} \pi_{i,j} = 1$. Here $\pi_{i,j}$ is the probability that an instance of the \emph{contaminated distribution} $\tilde{P}_i$ is a realization of $P_j$. The $\pi_{i,j}$s and $P_j$s are unknown and the $\tilde{P}_i$s are observed through data. In this work, we avoid parametric models and assume that the sample space is arbitrary. The model can be stated concisely as
\begin{align}
\tilde{\vec{P}} = \vec{\Pi} \vec{P} \label{model}
\end{align}
where $\vec{P} = (P_1, \ldots, P_L)^T$, $\tilde{\vec{P}} = (\p{1},\ldots, \p{M})^T$, and $\vec{\Pi} = \begin{pmatrix}
\pi_{i,j}
\end{pmatrix}$ is an $M \times L$ matrix that we call the \emph{mixing matrix}.

In this paper we study {\em decontamination}  of mutual contamination models, which is the problem of recovering, or estimating, the base distributions $\vec{P}$ from the contaminated distributions $\tilde{\vec{P}}$ from which data are observed, without knowledge of the mixing matrix $\vec{\Pi}$. We focus our attention on three specific types of mutual contamination models, all of which describe modern problems in machine learning: multiclass classification with label noise, demixing of mixed membership models and classification with partial labels. We will demonstrate that these three decontamination problems can be addressed using a common set of concepts and techniques. Before elaborating our contributions in detail, we first offer an overview of the three specific mutual contamination models, and associated decontamination problems, that we study.

\textbf{Multiclass Classification with Label Noise:} In multiclass classification with label noise, $M = L$ and the goal is to recover $\vec{P}$. Each $P_i$ represents the distribution of a class of examples. The learner observes training examples with noisy labels, that is, realizations from the $\p{j}$s. This problem arises in nuclear particle classification \citep{scott2013}. When one draws samples of a specific particle, it is impossible to remove other types of particles from the background. Thus, each example is drawn from a mixture of the different types of particles.

\textbf{Demixing of Mixed Membership Models:} We consider the following decontamination problem in mixed membership models: given a sample from each $\p{i}$, recover $\vec{P}$ up to a permutation. We refer to this decontamination problem as \emph{demixing of mixed membership models}. This problem arises in the task of automatically uncovering the thematic topics of a corpus of documents. Under the mixed membership model approach, the words of each document are thought of as being drawn from a document-specific mixture of topics. Specifically, documents correspond to the $\p{i}$s and the topics to the $P_i$s. This approach is also referred to as topic modeling. As we discuss in the next section, our theory significantly generalizes existing topic modeling theoretical guarantees.

\textbf{Classification with Partial Labels:}\footnote{Classification with partial labels has also been referred to as the ``superset learning problem" or the ``multiple label problem" \citep{liu2014}.} In classification with partial labels, each data point is labeled with a \emph{partial label} $Y \subset \set{1, \ldots, L}$; the true label is in $Y$, but it is not known which label is the true one. In our setup, we view the $i$th random sample as having partial label $Y_i \coloneq \set{j: \bpi_{i,j}>0}$ and being distributed according to $\tilde{P}_i = \sum_{j \in Y_i} \pi_{i,j} P_j$. Thus, the learner observes training examples from the contaminated distributions $\tilde{\bP}$ and the \emph{partial label matrix} $\bPi^+ = (\ind{\vec{\Pi}_{i,j} > 0})$, and the goal is to recover $\vec{P}$. 

There are many applications of classification with partial labels because often abundant sources of data are naturally associated with information that can be interpreted as partial labels. For example, consider the task of face recognition. On the internet, there are many images with captions that indicate who is in the picture but do not indicate which face belongs to which person. A partial label could be formed by associating each face with the names of the individuals appearing in the same image \citep{cour2011}. 

Although our work emphasizes recovery of $\vec{P}$, it is also possible to think of decontamination of mutual contamination models as concerned with estimation of the mixing matrix $\bPi$. This estimate of $\bPi$ could be used as a plug-in for recently developed debiased losses for multiclass classification with label noise and classification with partial labels, which require knowledge of $\bPi$ \citep{cidsueiro2012, menon15icml, vanrooyan2015arxix, patrini2017}. 

In this paper, we make the following contributions: (i) We give sufficient conditions on $\vec{P}$, $\vec{\Pi}$, and $\bPi^+$ for identifiability of the three problems. (ii) We establish necessary conditions that in some cases match or are similar to the sufficient conditions. (iii) We introduce novel algorithms for the infinite and finite sample settings. These algorithms are nonparametric in the sense that they do not model $P_i$ as a probability vector or other parametric model. Our algorithmic contributions show that while all three problems can be described in a unified way, the special structure of multiclass classification with label noise allows for a substantially simpler algorithm. (iv) We develop novel estimators for distributions obtained by iteratively applying the $\kappa^*$ operator (defined below). (v) Finally, our framework gives rise to several novel geometric insights about each of these three problems and leverages concepts from affine geometry, multilinear algebra, and probability.

\subsection{Notation}
Let $\bbZ^+$ denote the positive integers. For $n \in \bbZ^+$, let $[n] = \set{1, \ldots, n}$. If $\vec{x} \in \bbR^K$, let $x_i$ denote the $i$th entry of $\vec{x}$. If $\vec{x}_j \in \bbR^K$, then $x_{j,i}$ denotes the $i$th entry of $\vec{x}_j$. Let $\vec{e}_i$ denote the length $L$ vector with $1$ in the $i$th position and zeros elsewhere. Let $\vec{\pi}_i \in \si_L \subset \bbR^L$ be the transpose of the $i$th row of $\vec{\Pi}$ where $\si_L$ denotes the $(L-1)$-dimensional simplex, i.e., $\si_L = \{ \vec{\mu} = (\mu_1,\ldots, \mu_L)^T \in \bbR^L \,| \, \sum_{i=1}^L \mu_i = 1 \text{ and }  \forall i:  \mu_i   \geq 0  \}$. Let $\Delta_L^M$ denote the product of $M$ $(L-1)$-dimensional simplices, viewed as the space of $M \times L$ row-stochastic matrices. Let $\sP$ denote the space of probability distributions on a measurable space $(\sX, \sC)$. Let $\supp(F)$ denote the support of a distribution $F$ on a Borel space.

\section{Related Work}

Our work makes various contributions to the statistical understanding of multiclass classification with label noise, demixing of mixed membership models, and classification with partial labels. In the following subsections, we discuss how our results improve upon and relate to previous results in the literature.

\subsection{Multiclass Classification with Label Noise}
\label{multiclass_label_noise_related}

There has not been much work on classification with multiclass label noise. By contrast, label noise in the binary setting has received a fair amount of attention. For a review of work prior to 2013, see \citet{scott2013}. More recently, \citet{natarajan2013} considered the binary label noise case where the label noise rates are known (in our case, the label noise rates are unknown). \citet{vanrooyan2015arxix} generalized the work of \citet{natarajan2013} to the multiclass case, but again assumed that the mixing proportions are known. Recent work has proposed various algorithms for the binary setting where the label noise rates are unknown \citep{scott2015,vanrooyen2015, aditya2015}, but these algorithms have not been generalized to the multiclass case. \citet{menon2016} consider the binary setting with instance-dependent corruption, but they assume that the class probability functions take the form of a single-index model, whereas we make no parametric assumptions on the $P_i$s. \citet{ghosh2017} consider multiclass label noise, but they make two restrictive assumptions: \emph{(i)} in the infinite sample setting, they assume that there exists some function belonging to the chosen hypothesis class that attains $0$ risk and \emph{(ii)} in the finite sample setting, they assume that the label noise is symmetric, i.e., there exists a constant $c \in (0,1)$ such that $\pi_{i,j} = \frac{c}{L-1}$ for all $i \neq j$. \citet{patrini2017} also study the multiclass setting, but they assume that if their neural network has access to sufficiently many samples, it can perfectly model $\Pr(\tilde{Y} = k \, | \, \bx)$ where  $\bx$ is a given feature vector and $\tilde{Y}$ is a corrupted label.  Unlike most previous work that aims to learn a classifier, our focus is on estimating the base distributions. Given these estimates, one could then design a classifier to optimize some performance measure. See, for example, Section 4.3 of our initial work on this subject \citep{blanchard2014}.

Another approach for modeling random label noise, in addition to the mutual contamination model, is the label flipping model. Indeed, several of the above-cited papers adopt this setting. In this model, the label $Y$ of a data point is flipped independently of its features $X$ and
\begin{align*}
\mu_{l,k} \coloneq \Pr(\tilde{Y} = k \, | \, Y = l)
\end{align*}
gives the probability that a data point with true label $Y=l$ is corrupted to have an observed label $\tilde{Y} = k$. Under the assumption that $Y$ and $X$ are jointly distributed, the $\mu_{l,k}$s can be related to the $\pi_{i,j}s$ via Bayes' rule. We choose to study the mutual contamination model because we find it more convenient to study the question of identifiability. 

In this paper, we extend \citet{scott2013}, which examined binary classification with label noise (the case where $M = L = 2$). The multiclass setting is significantly more challenging and, as such, requires novel sufficient conditions and mathematical notions. In particular, \citet{scott2013} use the notion of \emph{irreducibility} of distributions as one of their sufficient conditions.
\begin{defn}
For distributions $G$ and $H$, we say that $G$ is \emph{irreducible} with respect to $H$ if it is not possible to write $G = \gamma H + (1-\gamma) F$ where $F$ is a distribution and $0 < \gamma \leq 1$. 
\end{defn}
\begin{defn}
For distributions $G$ and $H$, we say that $G$ and $H$ are \emph{mutually irreducible} if $G$ is irreducible with respect to $H$ and $H$ is irreducible with respect to $G$. We denote
\begin{align*}
\text{IR} & = \set{(G,H) : G \text{ and } H \text{ are mutually irreducible distributions}}.
\end{align*}
\end{defn}
\citet{scott2013} require that $P_1$ and $P_2$ are mutually irreducible. To treat the multiclass setting, we introduce a generalization of mutual irreducibility, namely \emph{joint irreducibility}. 

The work presented below on multiclass label noise originally appeared in a conference paper \citep{blanchard2014}. The purpose of the present paper is to demonstrate that the framework developed in that paper can be extended to the other two decontamination problems, and to provide a unified presentation of the three settings. In particular, the joint irreducibility assumption plays a pivotal role in all three settings, as does the task of mixture proportion estimation. However, the decontamination procedures for the latter two problems are substantially more complicated than for multiclass classification with label noise.

\subsection{Demixing Mixed Membership Models}
\label{demix_problem_related}

Mixed membership models have become a powerful modeling tool for data where data points are associated with multiple distributions. Applications have appeared in a wide range of fields including image processing \citep{li}, population genetics \citep{pritchard}, document analysis \citep{blei2003}, and surveys \citep{berkman}. One particularly popular application is topic modeling on a corpus of documents, such as the articles published in the journal Science. Topic modeling is closely related to demixing of mixed membership models and our work may be viewed as studying topic modeling on general domains. 

In topic modeling, the base distributions $P_i$ correspond to topics and the contaminated distributions $\p{i}$ to documents, which are regarded as mixtures of topics. In most cases, the $P_i$s are assumed to have a finite sample space. A variety of approaches have been proposed for topic modeling. The most common approach assumes a generative model for a corpus of documents and determines the maximum likelihood fit of the model given data. However, because maximum likelihood is NP-hard, these approaches must rely on heuristics that can get stuck in local minima \citep{arora_beyond_2012}. 

Recently, a trend towards algorithms for topic modeling with provable guarantees has emerged. Most of these methods rely on the separability assumption \textbf{(SEP)} and its variants \citep{donoho, arora_beyond_2012, arora_practical_2012, ding_2013, ding2014, recht_2012, huang2016}. According to \textbf{(SEP)}, $P_1,\ldots, P_L$ are distributions on a finite sample space and for every $i \in \set{1,\ldots,L}$, there exists a word $x \in \supp(P_i)$ such that $x \not \in \cup_{j \neq i} \supp(P_j)$. Our requirement that $P_1,\ldots, P_L$ are jointly irreducible is a natural generalization of separability of $P_1,\ldots, P_L$, as we will argue below. Specifically, if $P_1,\ldots, P_L$ have discrete sample spaces, separability and joint irreducibility coincide; however, if $P_1,\ldots, P_L$ are continuous, under joint irreducibility, $P_1,\ldots, P_L$ can have the same support. 

A key ingredient in these algorithms is to use the assumption of a finite sample space to view the distributions as probability vectors in Euclidean space; this leads to approaches based on non-negative matrix factorization (NMF), linear programs, and random projections \citep{donoho, arora_beyond_2012, arora_practical_2012, ding_2013, ding2014, recht_2012, huang2016}. However, more general distributions cannot be viewed as finite-dimensional vectors. Therefore, topic modeling on general domains requires new techniques. Our work seeks to provide such techniques.

Topic modeling on general domains has several applications, including in high-energy physics \citep{metodiev2018, metodiev2018jet}. In collider data, quantum chromodynamics causes data samples to be a mixture of different types of particles, where the underlying fraction of the particle type is unknown. In this setting, it is of interest to recover information about each of the particles. Recently, \citet{metodiev2018jet} applied the Demix algorithm, Algorithm \ref{demix_alg} in the current paper, to this problem in the case $M=L=2$. 

Topic modeling on general domains is also relevant to recent empirical research on topic modeling with word embeddings, e.g., \citep{das2015gaussian, li2016, lichenliang2016, xun2017correlated, pmlr-v80-zhao18a}. Word embeddings map words to vectors in $\bbR^d$ in a semantically and syntactically meaningful way. Their use has been pivotal to the state-of-art performance of many algorithms in NLP \citep{luong2013}. Several algorithms for topic modeling with word embeddings model the topics as multivariate Gaussian distributions in order to handle words that do not belong to the vocabulary of the training dataset \citep{das2015gaussian, xun2017correlated}. Whereas current topic modeling algorithms with theoretical guarantees do not cover such a modeling approach, the generality of our algorithms does.

\subsection{Classification with Partial Labels}
Classification with partial labels has had two main formulations in previous work \citep{liu2014}. In  one formulation \textbf{(PL-1)}, instances from each class are drawn independently and the partial label for each instance is drawn independently from a set-valued distribution. In another formulation \textbf{(PL-2)}, training data are in the form of bags where each bag is a set of instances and the bag has a set of labels. Each instance belongs to a single class, and the set of labels associated with the bag is given by the union of the labels of the instances in the bag. Our framework is similar to \textbf{(PL-2)}, although it does not assume a joint distribution on the features of instances and the partial labels. 

Most work takes an empirical risk minimization approach to classification with partial labels \citep{jin2002, nyugen2008, cour2011, liu2012}. Typically, these algorithms aim to pick a classifier that minimizes the \emph{partial label error}: the probability that a given classifier assigns a label to a training instance that is not contained in the partial label associated with the training instance. By contrast, our approach is to estimate the base distributions. One could then use these estimates to train a classifier under some performance measure.

There has not been much theoretical work on developing a statistical understanding of classification with partial labels. \citet{cidsueiro2012} and \citet{vanrooyan2015arxix} develop methods for classification with partial labels that require knowledge of the mixing proportions, e.g., the probability that a label is in a partial label, given the true label. In this work, we make the more realistic assumption that the mixing proportions are unknown. 

\citet{liu2014} consider the question of learnability where the mixing proportions are unknown. They consider two main sufficient conditions for learnability of a partial label problem. First, they require that for every label $l \in [L]$, the probability that $l$ occurs with any particular distinct label $l^\prime$ is less than $1$. Our condition on the partial label (described in the next Section) is considerably weaker. For example, it permits the case where there are two labels $l \neq l^\prime$ such that whenever $l$ occurs in a partial label, $l^\prime$ also occurs. 

The second sufficient condition of \citet{liu2014} is based on the class distributions, partial label distributions \emph{and} the hypothesis class of choice. It requires that every hypothesis that attains zero partial label error also attains zero true error. While this condition may be useful for the selection of a suitable hypothesis class for an ERM approach, it is important to develop interpretable sufficient conditions that only depend on the characteristics of a partial label problem. Our work provides such conditions.

We also note that \citet{liu2014} consider the realizable case, that is, the case where the supports of $P_1,\ldots, P_L$ do not overlap. By contrast, we make the significantly weaker assumption that $P_1,\ldots, P_L$ are jointly irreducible, which allows $P_1,\ldots, P_L$ to have the same support. Thus, our work addresses the agnostic case in classification with partial labels.

\section{Sufficient Conditions for Identifiability}

We can think of each problem as requiring a specific factorization of $\tilde{\vec{P}}$ in terms of $\vec{P}$ and $\vec{\Pi}$. We say $\tilde{\vec{P}}$ is \emph{factorizable} if there exists $(\vec{\Pi}, \vec{P}) \in \Delta_L^M \times \sP^L$ such that $\tilde{\vec{P}} = \vec{\Pi} \vec{P}$; we call $(\vec{\Pi}, \vec{P})$ a \emph{factorization} of $\tilde{\vec{P}}$. Multiclass classification with label noise requires a specific ordering of the elements of $\vec{P}$; classification with partial labels requires that $\vec{\Pi}$ is consistent with $\bPi^+$ and a specific ordering of the elements of $\vec{P}$.

A factorization is not guaranteed to exist. For example, there is no factorization in the case where $M = 3$, $L=2$, and $\p{1}, \p{2}, \p{3}$ are linearly independent. When a factorization exists, in general it is not unique. For instance, consider the case where $L=M$, $(\vec{\Pi},\vec{P})$ solves (\ref{model}), and $\vec{\Pi}$ is not a permutation matrix. Then, another solution is $\tilde{\vec{P}} = \vec{I} \tilde{\vec{P}}$. Furthermore, there are infinitely many solutions in the following general case.

\begin{prop}
\label{many_factorizations}
Suppose that $\tilde{\vec{P}}$ has at least two distinct $\p{j}$s and has a factorization $(\vec{\Pi}, \vec{P})$. If there is some $\p{i}$ in the interior of $\co(P_1, \ldots, P_L)$, then there are infinitely many distinct non-trivial factorizations of $\tilde{\vec{P}}$.
\end{prop}

\begin{proof}
Without loss of generality, suppose that $i = 1$ and $\p{1} \neq \p{2}$. Then, since $\p{1}$ is in the interior of $\co(P_1, \ldots, P_L)$, there is some $\delta > 0$ such that for any $\alpha \in (1, 1 + \delta)$, $Q_\alpha = \alpha \p{1} + (1- \alpha) \p{2}$ is a distribution. Then, $\co(\p{1}, \ldots, \p{L}) \subseteq \co(Q_\alpha, \p{2}, \ldots, \p{L})$   and, consequently, there is some $\vec{\Pi^\prime} \in \Delta_L^L$ such that $(\vec{\Pi^\prime}, (Q_\alpha, \p{2}, \ldots, \p{L})^T)$ solves (\ref{model}). Clearly, by varying $\alpha$, there are infinitely many solutions to (\ref{model}).
\end{proof}

Identifiability of each problem is equivalent to the existence of a unique factorization for that problem. Therefore, to establish identifiability for the three problems, we must impose conditions on $(\vec{\Pi}, \vec{P})$ and $\bPi^+$. To this end, we use the notion of joint irreducibility of distributions.
\begin{defn}
The distributions $\set{P_i}_{1 \leq i \leq L}$ are \emph{jointly irreducible} iff the following equivalent conditions hold
\begin{enumerate}
\item[(a)] For all $I \subset [L]$ such that $1 \leq |I| < L$, and $\epsilon_i$ such that $\epsilon_i \geq 0$ and $\sum_{i \in I} \epsilon_i = \sum_{i \not \in I} \epsilon_i = 1$, 
\begin{align*}
(\sum_{i \in I} \epsilon_i P_i, \sum_{i \not \in I} \epsilon_i P_i) \in \text{IR}.
\end{align*}
\item[(b)] $\sum_{i=1}^L \gamma_i P_i$ is a distribution implies that $\gamma_i \geq 0 \, \forall i$.
\end{enumerate}
\end{defn}
\noindent Conditions \emph{(a)} and \emph{(b)}, whose equivalence was established by \citet{blanchard2014}, give two ways to think about joint irreducibility. Condition \emph{(a)} says that every convex combination of a subset of the $P_i$s is irreducible (see Section \ref{multiclass_label_noise_related}) with respect to every convex combination of the other $P_i$s. Condition \emph{(b)} says that if a distribution is in the span of $P_1, \ldots, P_L$, it is in their convex hull. Joint irreducibility holds when each $P_i$ has a region of positive probability that does not belong to the support of any of the other $P_i$s; thus, separability (see Section \ref{demix_problem_related}) of the $P_i$s entails joint irreducibility of $P_1,\ldots, P_L$. However, the converse is not true: the $P_i$s can have the same support and still be jointly irreducible (e.g., $P_i$s Gaussian with a common variance and distinct means \citep{scott2013}). 

For all three problems, we assume that
\begin{description}
\item[(A)] $P_1, \ldots, P_L$ are jointly irreducible.
\end{description}
Henceforth, unless we say otherwise, $P_1, \ldots, P_L$ are assumed to be jointly irreducible. In Appendix \ref{experiments_section}, we provide experiments on real-world datasets that suggest that this assumption is reasonable.

We make different assumptions on $\vec{\Pi}$ for each of the three problems. For multiclass classification with label noise, we assume that 
\begin{description}
\item[(B1)] $\vec{\Pi}$ is invertible and $\vec{\Pi}^{-1}$ is a matrix with strictly positive diagonal entries and nonpositive off-diagonal entries. 
\end{description}
According to Lemma \ref{2_blanchard2014} below, this assumption essentially says that the problem has low noise in the sense that for each $i$, $\p{i}$ mostly comes from $P_i$. In particular, each $P_i$ can be recovered by subtracting small multiples of $\tilde{P}_j$, $j \neq i$ from $\tilde{P}_i$. For example, consider the following case where $\vec{\Pi}$ satisfies \textbf{(B1)}. Suppose that there is a ``common background noise" $\vec{c} \in \Delta_L$ that appears in different proportions in each of the distributions; formally, we have $\vec{\pi}_i = \gamma_i \vec{c} + (1-\gamma_i) \vec{e}_i$ with $\gamma_i \in [0,1)$. In other words, we shift each of the vertices $\vec{e}_i$ towards a common point $\vec{c}$ (see panel (iii) of Figure \ref{fig:A}). See \citet{blanchard2014} for a proof that this setup satisfies \textbf{(B1)}. In the binary case where $M = L = 2$, \textbf{(B1)} is equivalent to the simple condition that $\pi_{1,1} + \pi_{2,2} < 1$. This assumption roughly says that in expectation the majority of labels are correct. In Section \ref{multiclass_classification_alg}, we present Lemma~\ref{2_blanchard2014}, which gives a geometric interpretation of \textbf{(B1)}. 

For the demixing problem, we assume that
\begin{description}
\item[(B2)] $\vec{\Pi}$ has full column rank.
\end{description}
We note that \textbf{(B2)} is considerably weaker than \textbf{(B1)}, e.g., it allows $M > L$. Of course, it is natural to demand a weaker sufficient condition for demixing the mixed membership problem than muliticlass classification with label noise because the goal of the former problem is to recover any permutation of $\vec{P}$ while the goal of the latter is to recover $\vec{P}$ exactly. Nevertheless, the identifiability analysis to establish \textbf{(B2)} as a sufficient condition is also significantly more involved than the analysis of \textbf{(B1)}.

For classification with partial labels, we assume that  
\begin{description}
\item[(B3)] $\vec{\Pi}$ has full column rank and the columns of $\bPi^+$ are unique.
\end{description}
The assumption that the columns of $\bPi^+$ are unique says that there are no two classes that always appear together in the partial labels. In Appendix \ref{fact_res_app}, we argue that several of the above conditions are also necessary, or are not much stronger than what is necessary.

\section{Algorithms for the Population Case}
In this section, to establish that the above conditions are indeed sufficient for identifiability, we give a population case analysis of the three problems. The results on multiclass classification with label noise appeared in a conference paper \citep{blanchard2014}; we refer the reader to that paper for the proofs.

\subsection{Background}
This paper relies on the following quantity from \citet{blanchard2010}.
\begin{defn}
Given probability distributions $F_0, F_1$, define
\begin{align*}
\kappa^*(F_0 \, | \, F_1)  = \max \{ \kappa \in [0,1] | \, \exists \text{ a distribution } G \text{ s.t. } F_0 = (1-\kappa)G + \kappa F_1 \}.
\end{align*}
\end{defn}
The following Proposition from  \citet{blanchard2010} establishes some useful properties of $\kappa^*$.
\begin{prop}
\label{label_noise_bin_case}
Given probability distributions $F_0, F_1$ on a measurable space $(\sX, \sC)$, if $F_0 \neq F_1$, then $\kappa^*(F_0 \, | \, F_1) < 1$ and the above maximum is attained for a unique distribution $G$ (which we refer to as the \emph{residue of $F_0$ wrt. $F_1$}). Furthermore, the following equivalent characterization holds:
\begin{align*}
\kappa^*(F_0 \, | \, F_1) = \inf_{C \in \sC, F_1(C) > 0} \frac{F_0(C)}{F_1(C)}.
\end{align*}
\end{prop}
\noindent \sloppy  Note that $\kappa^*( F_0 \, | \, F_1) = 0$ iff $F_0$ is irreducible wrt $F_1$. $\kappa^*(F_0 \, | \, F_1)$ can be thought of as the maximum possible proportion of $F_1$ in $F_0$. We can think of $1-\kappa^*(F_0 \, | \, F_1)$ as a statistical distance since it is non-negative and equal to zero if and only if $F_0 = F_1$. We refer to $\kappa^*$ as the two-sample $\kappa^*$ operator. To obtain the residue of $F_0$ wrt $F_1$, one computes Residue($F_0 \, | \, F_1$) (see Algorithm \ref{two_residue}); this is well-defined under Proposition \ref{label_noise_bin_case} when $F_0 \neq F_1$.

In order to gain intuition about $\kappa^*$, we briefly discuss how it can be used to recover $\bPi^{-1}$ in the case $L=2$. Under conditions discussed above \citep{scott2013}, it holds that 
\begin{align*}
\p{1} & = (1- \kappa_1) P_1 + \kappa_1 \p{2}, \text{ and} \\
\p{2} & = ( 1- \kappa_2) P_2 + \kappa_2 \p{1}.
\end{align*}
and $\kappa_1 = \kappa^*(\p{1} \, | \, \p{2})$ and $\kappa_2 = \kappa^*(\p{2} \, | \, \p{1})$. By rearranging this system of equations, we can write
\begin{align*}
\bP & = \bPi^{-1} \tilde{\bP} = \begin{pmatrix}
\frac{1}{1-\kappa_1} & - \frac{\kappa_1}{1-\kappa_1} \\
- \frac{\kappa_2}{1-\kappa_2} & \frac{1}{1-\kappa_2}
\end{pmatrix} \tilde{\bP}.
\end{align*}

Next, we turn to the multi-sample generalization of $\kappa^*$, which we call the multi-sample $\kappa^*$ operator.
\begin{defn}
Given distributions $F_0, \ldots, F_K$, define 
\begin{align}
& \kappa^*(F_0 \, | \, F_1, \ldots, F_K) = \max_{\bmu \in \si_K } \kappa^*(F_0 \, | \, \sum_{i=1}^K \mu_i F_i) \nonumber \\
= & \max \Big(  \sum_{i=1}^K \nu_i :   \nu_i \geq 0, \sum_{i=1}^K \nu_i \leq 1, \exists \text{ distribution } G \text{ s.t. } F_0 = (1 - \sum_{i=1}^K \nu_i) G + \sum_{i=1}^K \nu_i F_i \Big). \label{multisample_kappa}
\end{align}
\end{defn}
\noindent \citet{blanchard2014} establish the equivalence in line \eqref{multisample_kappa}, as well as Lemma \ref{A_1}, which shows that the outer maximum is always attained at some $\bmu  \in \si_K$, i.e., $\kappa^*$ is well-defined. Although there is always a $G$ achieving the max, it is not necessarily unique. Any $G$ attaining the maximum is called a \emph{maximizer} of $\kappa^*(F_0 \, | \, F_1, \ldots, F_K)$. The algorithm $\sol$($F_0 \, | \, \set{F_1, \ldots, F_K}$) returns one of these $G$ (see Algorithm \ref{multi_kappa_sol}). If $G$ is unique, we call $G$ the \emph{multi-sample residue of $F_0$ with respect to $\set{F_1, \ldots, F_K}$}. Under our proposed sufficient conditions, certain residues are shown to exist, and our decontamination methods compute such residues via Algorithm \ref{multi_kappa_sol}. In Section \ref{multiclass_classification_alg}, we discuss Lemma \ref{2_blanchard2014}, which establishes useful conditions under which a multi-sample residue exists and is equal to one of the vertices of $\si_L$.

In general, one cannot express the multi-sample version of $\kappa^*$ in terms of the two-sample version. However, it is possible in some special cases. For example, if one had access to feasible $\nu_1, \ldots, \nu_K$ that attain the optimum in \eqref{multisample_kappa}, then it holds that $\kappa^*(F_0 \, | \, F_1, \ldots, F_K) = \kappa^*(F_0 \, | \,\frac{ \sum_{i =1}^K \nu_i F_i}{\sum_{i=1}^K \nu_i})$. Further, it is possible to replace the multi-sample $\kappa^*$ with several calls of the two-sample $\kappa^*$ when $K = L-1$, $F_i = P_i$ for all $i \neq 0$ and $F_0 = \sum_{i=1}^L \alpha_i P_i$ where  $\sum_i \alpha_i = 1$ and $\forall i$ $\alpha_i > 0$ (see Lemmas \ref{single} and \ref{mod_ident_lemma}). 

We remark that in previous work that assumes $P_i$ are probability vectors, distributions are compared using $l_p$ distances. By contrast, in our setting of general probability spaces, we use $\kappa^*$ to compare different distributions. 

\begin{algorithm}[t]
\caption{Residue($F_0 \, | \, F_1$)}
\begin{algorithmic}[1]
\label{two_residue}
\STATE $\kappa \longleftarrow \kappa^*(F_0 \, | \, F_1)$
\RETURN $\frac{F_0 - \kappa F_1}{1 - \kappa}$
\end{algorithmic}
\end{algorithm}

\begin{algorithm}[t]
\caption{$\sol$($F_0 \, | \, \set{F_1, \ldots, F_K}$)}
\begin{algorithmic}[1]
\label{multi_kappa_sol}
\STATE 	$(\nu_1, \ldots, \nu_K)^T \longleftarrow (\nu^\prime_1, \ldots, \nu^\prime_K)^T \text{ achieving the maximum in } \kappa^*(F_0 \, | \, F_1, \ldots, F_K)$\\
\RETURN $\frac{F_0 - \sum_{i=1}^K \nu_i F_i}{1 - \sum_{i=1}^K \nu_i}$
\end{algorithmic}
\end{algorithm}

\subsection{Mixture Proportions}

Recall that we assume that $P_1,\ldots, P_L$ are jointly irreducible. If $\vec{\eta} \in \bbR^L$ and $Q = \vec{\eta}^T \vec{P}$, we say that $\vec{\eta}$ is the \emph{mixture proportion} of $Q$. Since by Lemma \ref{B_1}, joint irreducibility of $P_1,\ldots, P_L$ implies linear independence of $P_1,\ldots, P_L$, mixture proportions are well-defined, i.e., the mixture proportions are unique.

An important feature of our decontamination strategy is recovering various mixture proportions in the simplex $\si_L$. To make this precise, we introduce the following definitions. If $i \in [L]$, we say that $\co(\set{\vec{e}_j : j \neq i})$ is a \emph{face} of the simplex $\si_L$; if $A \subset [L]$ and $|A| = k$, we also say that $\co(\set{\vec{e}_j : j \in A})$ is a \emph{$k$-face} of $\si_L$. If $\vec{\eta} \in \bbR^L$, $Q$ is a distribution, and $Q = \vec{\eta}^T \vec{P}$, we say that $\sN(Q) = \sN(\vec{\eta}) = \set{j : \eta_j > 0}$ is the \emph{support set} of $\vec{\eta}$ or the support set of $Q$. Note that by joint irreducibity, $\sN(\vec{\eta})$ consists of the indices of all the nonzero entries in the mixture proportion $\vec{\eta}$. Finally, for $\vec{\eta}_i \in \si_L$, and $Q_i = \vec{\eta}_i^T \vec{P}$ for $i =1,2$, we say that the distributions $Q_1$ and $Q_2$ (or the mixture proportions $\vec{\eta}_1$ and $\vec{\eta}_2$) are \emph{on the same face} of the simplex $\si_L$ if there exists $j \in [L]$ such that $\vec{\eta}_1, \vec{\eta}_2 \in \co (\set{\vec{e}_k : k \neq j})$. 

The heart of our approach is that under joint irreducibility, one can interchange distributions $Q_1, \ldots, Q_K$ and their mixture proportions $\vec{\eta}_1, \ldots, \vec{\eta}_K$, as indicated by the following Proposition. We note that it is valid to to apply the $\kappa^*$ operator to $\vec{\eta}_1, \ldots, \vec{\eta}_K$ since they can be viewed as discrete probability distributions over $[L]$.
\begin{prop}
\label{equiv_opt}
Let $Q_i = \vec{\eta}_i^T \vec{P}$ for $i \in [L]$ and $\vec{\eta}_i \in \si_L$. Suppose $\vec{\eta}_1, \ldots, \vec{\eta}_L$ are linearly independent and $P_1, \ldots, P_L$ are jointly irreducible. Then,
\begin{enumerate}
\item for any $i \in [L]$ and $A \subseteq [L] \setminus \set{i}$, $\kappa^*(Q_i \, | \, \set{Q_j : j \in  A}) =  \kappa^*(\vec{\eta}_i \, | \, \set{\vec{\eta}_j : j \in  A})<1$,

\item for any $i \in [L]$ and $A \subseteq [L] \setminus \set{i}$, a maximizer of $\kappa^*( Q_i \, | \,  \set{Q_j : j \in  A})$ exists, and

\item $\vec{\gamma} \in \si_L$ is a maximizer to $\kappa^*(\vec{\eta}_i \, | \, \set{\vec{\eta}_j : j \in  A})$ if and only if $G = \vec{\gamma}^T \vec{P}$ is a maximizer to $\kappa^*( Q_i \, | \,  \set{Q_j : j \in  A})$. 
\end{enumerate}

\end{prop}
\noindent In words, this proposition says that the optimization problem given by $\kappa^*(Q_i \, | \, \set{Q_j : j \in  A})$ is equivalent to the optimization problem given by $\kappa^*(\vec{\eta}_i \, | \, \set{\vec{\eta}_j : j \in  A})$. \emph{Thus, joint irreducibility of $P_1, \ldots, P_L$ and linear independence of the mixture proportions enable a reduction of each of the three problems to a geometric problem where the goal is to recover the vertices of a simplex by applying $\kappa^*$ to points (i.e., the mixture proportions) in the simplex.} This makes the figures below valid for general distributions (see Figures  \ref{fig:A}, \ref{2_case}, \ref{3_case}, and \ref{partial_label_figure}).

\subsection{Multiclass Classification with Label Noise}
\label{multiclass_classification_alg}

Our algorithm for multiclass classification with label noise is by far the simplest of the three. It simply computes a maximizer of $\kappa^*(\tilde{P}_i \, | \, \set{\tilde{P}_j : j \neq i})$ for every $i \in [L]$. 
\begin{thm}
\label{multiclass_identification}
Let $P_1, \ldots, P_L$ be jointly irreducible and $\vec{\Pi}$ satisfy \textbf{(B1)}. Then, Multiclass($\p{1}, \ldots, \p{L}$) returns $\vec{Q} \in \sP^L$ such that $\vec{Q} = \vec{P}$.
\end{thm}

\begin{algorithm}[t]
\caption{Multiclass($\p{1}, \ldots, \p{L}$)}
\begin{algorithmic}[1]
\label{multiclass_alg}
\FOR{$i = 1, \ldots, L$}
\STATE $Q_i \longleftarrow \sol(\p{i} \, | \, \set{ \p{j} : j \neq i})$ 
\ENDFOR
\RETURN $(Q_1, \ldots, Q_L)^T$
\end{algorithmic}
\end{algorithm}

\noindent The proof of this result has two main ideas. First, it applies the one-to-one correspondence established in Proposition \ref{equiv_opt} between the maximizers of $\kappa^*( \p{i} \, | \,\set{ \p{j} : j \neq i})$ and the maximizers of $\kappa^*( \vec{\pi}_i \, | \, \set{ \vec{\pi}_j  : j \neq i})$. 

Second, the proof shows that $\kappa^*( \vec{\pi}_i \, | \, \set{ \vec{\pi}_j  : j \neq i})$ is well-behaved in the sense that the residue of $\bpi_i$ wrt $\set{\bpi_j: j \neq i}$ is $\be_i$. The key idea is encapsulated in the following Lemma from \citet{blanchard2014}.  
\begin{lemma}\emph{\textbf{\citep{blanchard2014}}}
\label{2_blanchard2014}
The following conditions on $\bpi_1, \ldots, \bpi_L$ are equivalent:
\begin{enumerate}
\item For each $i$, the residue of $\bpi_i$ with respect to $\set{\bpi_j, j \neq i}$ is $\be_i$.

\item For every $i$ there exists a decomposition $\bpi_i = \kappa_i \be_i + (1-\kappa_i) \bpi_i^\prime$ where $\kappa_i > 0$ and $\bpi_i^\prime$ is a convex combination of $\bpi_j$ for $j \neq i$.

\item $\vec{\Pi}$ is invertible and $\vec{\Pi}^{-1}$ is a matrix with strictly positive diagonal entries and nonpositive off-diagonal entries.

\end{enumerate}
\end{lemma}
\noindent This lemma establishes that under \textbf{(B1)}, for each $i$, the residue of $\bpi_i$ with respect to $\set{\bpi_j, j \neq i}$ is $\be_i$. The main step in the proof of this Lemma is establishing that $\emph{3}$ implies $\emph{1}$. The argument identifies the residue of $\bpi_i$ with respect to $\set{\bpi_j}_{j \neq i}$ by reformulating the linear program in $\kappa^*(\bpi_i \, | \, \set{\bpi_j}_{j \neq i})$ such that the objective is to maximize $\be_i^t\bPi^{-1} \gamma$ subject to some appropriately defined constraint. By the structure of $\bPi^{-1}$ assumed in \textbf{(B1)}, it follows that the $\gamma \in \si_L$ that maximizes this objective is $\be_i$, and it can further be shown that this maximizer satisfies the other constraints.

Thus, combining the above two ideas yields the result. In addition, Lemma \ref{2_blanchard2014} provides geometric intuition as to when \textbf{(B1)} is satisfied through condition \emph{2}. Figure \ref{fig:A} illustrates the case $L = 3$. See Panel (i) for an example where condition (b) is satisfied and Panel (ii) for an example where (b) is not satisfied.

\begin{figure*}[t]
\centering
\begin{tabular}{ccc}
\includegraphics[width=1.5in]{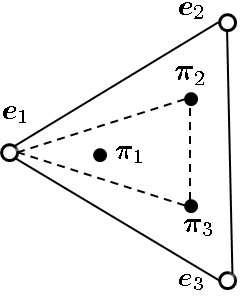} \hspace{5mm} &
\hspace{5mm} \includegraphics[width=1.5in]{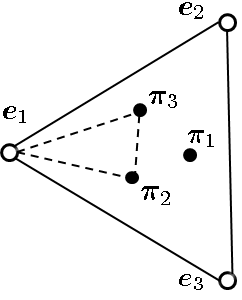} \hspace{5mm} &
\hspace{5mm} \includegraphics[width=1.5in]{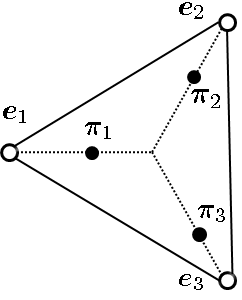} \\
(i) & (ii) & (iii)
\end{tabular}
\caption{Illustration of the \textbf{(B1)} when there are $L=3$ classes where $\be_i$ denotes the $i$th unit vector. Panel
(i): Low noise, $\vec{\Pi}$ recoverable. Each $\vec{\pi}_l$ can be written as a
convex combination of $\vec{e}_l$ and the other two $\vec{\pi}_j$ (with a
positive weight on $\vec{e}_l$), depicted here for $l = 1$. Panel
(ii): High noise, $\vec{\Pi}$ not recoverable. Panel (iii): The setting of
``common background noise" described in the text.}
\label{fig:A}
\end{figure*}

\subsection{Demixing Mixed Membership Models}
\label{demixing_mixed_membership_models_population_section}
%

In this section, we assume that $M=L$; we consider the nonsquare case in the appendix. For certain simple cases of mixture proportions, a straightforward resampling strategy can be used to reduce the problem of demixing mixed membership models to multiclass classification with label noise. For example, suppose that there are $L = 3$ classes and
\begin{align}
\vec{\Pi} & = \begin{pmatrix}
\frac{1}{2} & \frac{1}{2} & 0 \\
\frac{1}{2} & 0 & \frac{1}{2} \\
0 & \frac{1}{2} & \frac{1}{2} 
\end{pmatrix}.
\label{example_problem}
\end{align}
Inspection shows that the inverse of $\bPi$ does not satisfy the condition in \textbf{(B1)} and, therefore, one cannot simply apply Algorithm \ref{multiclass_alg}. A simple procedure to circumvent this issue is to resample from the contaminated distributions to obtain the following distributions:
\begin{align*}
\tilde{Q}_1 & = \frac{1}{2} \p{1} + \frac{1}{2} \p{2}, \, \,
\tilde{Q}_2  = \frac{1}{2} \p{1} + \frac{1}{2} \p{3}, \, \, \text{and} \, \,
\tilde{Q}_3  = \frac{1}{2} \p{2} + \frac{1}{2} \p{3}.
\end{align*}
Then, it can be shown that the resulting mixing matrix
\begin{align*}
\tilde{\bPi} & = \begin{pmatrix}
\frac{1}{2} & \frac{1}{4} & \frac{1}{4} \\
\frac{1}{4} & \frac{1}{2} & \frac{1}{4} \\
\frac{1}{4} & \frac{1}{4} & \frac{1}{2}
\end{pmatrix}
\end{align*}
associated with the $\tilde{Q}_i$s satisfies the conditions of Lemma \ref{2_blanchard2014} so that Multiclass($\tilde{Q}_1, \tilde{Q}_2, \tilde{Q}_3$) gives the desired solution. However, this approach breaks down for most possible mixing matrices. Thus, the challenge is to develop an algorithm that works for a large class of mixture proportions and does not rely on knowledge of the mixture proportions. To meet this challenge, we propose the Demix algorithm.

The Demix algorithm is recursive. Let $S_1, \ldots, S_K$ denote $K$ contaminated distributions. In the base case, the algorithm takes as its input two contaminated distributions $S_1$ and $S_2$. It returns Residue($S_1 \, | \, S_2$) and Residue($S_2 \, | \, S_1$), which are a permutation of the two base distributions (see Figure \ref{2_case}). When $K > 2$, Demix uses a subroutine FindFace (see Algorithm \ref{findface_alg}) to find $K-1$ distributions $R_2,\ldots, R_K$ on the same $(K-1)$-face. FindFace iteratively generates candidates for distributions on the same $(K-1)$-face, which it tests using FaceTest (see Algorithm \ref{face_test_alg}). FaceTest$(S_1, \ldots, S_{K-1})$ determines whether a set of distributions $S_1, \ldots, S_{K-1}$ are on the interior of the same face by using the two-sample $\kappa^*$ operator; equivalently, it tests whether there exists a pair of distributions $S_i$ and $S_j$ such that $S_i$ is irreducible with respect to $S_j$. Once Demix finds $K-1$ distributions $R_2,\ldots, R_K$ on the same $(K-1)$-face, it recursively applies Demix to $R_2,\ldots, R_K$ to obtain distributions $Q_1,\ldots, Q_{K-1}$ that are a permutation of $K-1$ of the base distributions. Subsequently, the algorithm computes a maximizer $Q_K$ of $\kappa^*(\frac{1}{K} \sum_{i=1}^K S_i \, | \, Q_1, \ldots, Q_{K-1})$. Since $Q_1, \ldots, Q_{K-1}$ are a permutation of $K-1$ of the base distributions, the maximizer $Q_K$ is guaranteed to be unique and to be the remaining base distribution (see Figure \ref{3_case} for an execution of the algorithm).
\begin{algorithm}[t]
\caption{Demix($S_1, \ldots, S_K$)}
\textbf{Input: }$S_1, \ldots, S_K$ are distributions
\label{demix_alg}
\begin{algorithmic}[1]
\IF{$K = 2$}
\RETURN $(\text{Residue}(S_1 \, | \, S_2),  \text{Residue}(S_2 \, | \, S_1))^T$
\ELSE
\STATE $(R_2, \ldots, R_K)^T \longleftarrow \text{FindFace}(S_1, \ldots,S_K)$
\STATE $(Q_1, \ldots, Q_{K-1})^T \longleftarrow \text{Demix}(R_2, \ldots, R_K)$
\STATE $Q_K \longleftarrow \frac{1}{K} \sum_{i=1}^K S_i $ 
\STATE $Q_K \longleftarrow \sol(Q_K \, | \, Q_1, \ldots, Q_{K-1})$ \label{multi_sample_kappa_demix}
\RETURN $(Q_1, \ldots, Q_K)^T$
\ENDIF
\end{algorithmic}
\end{algorithm}

\begin{algorithm}[t]
\caption{FindFace($S_1, \ldots, S_K$)}
\textbf{Input: }$S_1, \ldots, S_K$ are distributions
\label{findface_alg}
\begin{algorithmic}[1]
\STATE $Q \longleftarrow \text{ uniformly distributed element in} \co (S_2,\ldots,S_K)$
\FOR{$n=1,2,\ldots$}
\STATE Set $R_i \longleftarrow \text{Residue}(\frac{1}{n} S_i + \frac{n-1}{n} Q \, | \, S_1)$ for all $i \in \set{2,\ldots,K}$ \label{findface_res_line}
\IF{$\text{FaceTest}(R_2,\ldots,R_K)$}
\RETURN $(R_2, \ldots, R_K)^T$
\ENDIF
\ENDFOR
\end{algorithmic}
\end{algorithm}

\begin{algorithm}[t]
\caption{FaceTest($S_1,\ldots, S_K$)}
\begin{algorithmic}[1]
\label{face_test_alg}
\STATE Set $\vec{Z}_{i,j} \coloneq \vec{1}\{\kappa^*(S_i \, | \, S_j) > 0\}$ for all $i$ and $j$
\IF{$\vec{Z}$ has a zero off-diagonal entry}
\RETURN $0$
\ELSE
\RETURN $1$
\ENDIF
\end{algorithmic}
\end{algorithm}

A number of remarks are in order regarding the Demix algorithm. First, although we compute the residue of $\frac{1}{n} S_i + \frac{n-1}{n} Q$ wrt $S_1$ for each $i \neq 1$, there is nothing special about the distribution $S_1$. We could replace $S_1$ with any $S_j$ where $j \in [K]$, provided that we adjust the rest of the algorithm accordingly. Second, we can replace the sequence $\set{\frac{n-1}{n}}_{n=1}^\infty$ with any sequence $\alpha_n \nearrow 1$. Finally, we could replace line \ref{multi_sample_kappa_demix} with the following sequence of steps: for $i = 1, \ldots, K-1$, compute $Q_K \longleftarrow \text{Residue}(Q_K \, | \, Q_i)$ (see Lemma \ref{mod_ident_lemma}). Then, the algorithm would only use the two-sample $\kappa^*$ operator. We use such an algorithm in the finite-sample setting.

We also remark that a simplified version of Demix solves the demixing mixed membership models problem if we assume \textbf{(B1)} from the multiclass label noise setting. In that case, finding $L-1$ distributions on the same face can be accomplished by simply computing $Q_i \longleftarrow \sol$($\p{i} \, | \, \set{\p{j}}_{j \neq i}$) for $i = 2, \ldots, L$. Indeed, then, each $Q_i$ is equal to $P_i$ and $P_1$ can be obtained by computing $\sol(\p{1} \, | \, \set{Q_j}_{j =2, \ldots, L})$. 

We establish the following theorem.
\begin{thm}
\label{demix_identification}
Let $P_1, \ldots, P_L$ be jointly irreducible and $\vec{\Pi}$ have full column rank. Then, with probability $1$, Demix$(\tilde{\vec{P}})$ returns a permutation of  $\vec{P}$.
\end{thm}
We briefly sketch three key aspects of the proof. First, in the FindFace subroutine, sampling $Q$ uniformly at random from $\co (S_2,\ldots,S_K)$ ensures that w.p. 1 $\text{Residue}( Q \, | \, S_1)$ is on the interior of a face of the simplex. Then, conditional on this event, we show that by a continuity property of $\text{Residue}( \cdot \, | \, S_1)$ there is a large enough $n$ such that $R_2, \ldots, R_K$ are on the same face of the simplex $\si_{K-1}$ (see panels (c) and (d) of Figure \ref{3_case}). Second, Proposition \ref{face_test} in Appendix \ref{ident_app} establishes that the subroutine $ \text{FaceTest}(R_2,\ldots,R_K)$ returns $1$ if and only if $R_2, \ldots, R_K$ are on the same face of the simplex. Combining the above two observations implies that eventually FindFace($S_1, \ldots, S_K$) terminates at which point $\set{R_k}_{k \in [K] \setminus \set{1}} \subset \set{P_k}_{k \in [K] \setminus \set{l}}$ for some $l \in [K]$.  The final key observation is that $\set{R_k}_{k \in [K] \setminus \set{1}}$ and $\set{P_k}_{k \in [K] \setminus \set{l}}$ form an instance of the demixing problem that satisfies the sufficient conditions \textbf{(A)} and \textbf{(B2)} (see Figure \ref{2_case}). Therefore, this instance can be solved recursively.

\begin{figure}[t]
\centering
\setlength{\textfloatsep}{5pt}
\includegraphics[scale = .55]{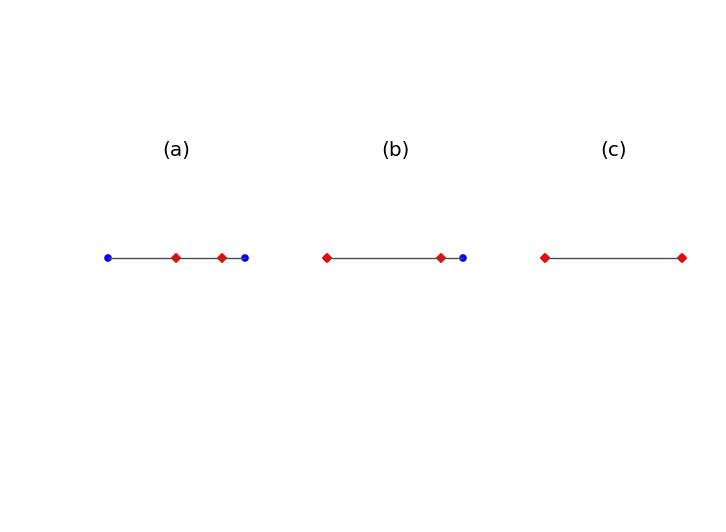}
\vspace{-8em}
\caption{In (a), we consider a demixing problem where there are two classes and $M=L$ (the base case of Algorithm \ref{demix_alg}). The diamonds represent the mixture proportions of $\p{1}$ and $\p{2}$. The circles represent the base distributions. In (b), the residue of a contaminated distribution wrt the other contaminated distribution is computed (line 3), yielding a base distribution. In (c), the residue is computed again switching the roles of the contaminated distributions (line 4); this yields the remaining base distribution.}
\label{2_case}
\end{figure}
\begin{figure}[t]
\centering
\setlength{\textfloatsep}{5pt}
\includegraphics[scale = .55]{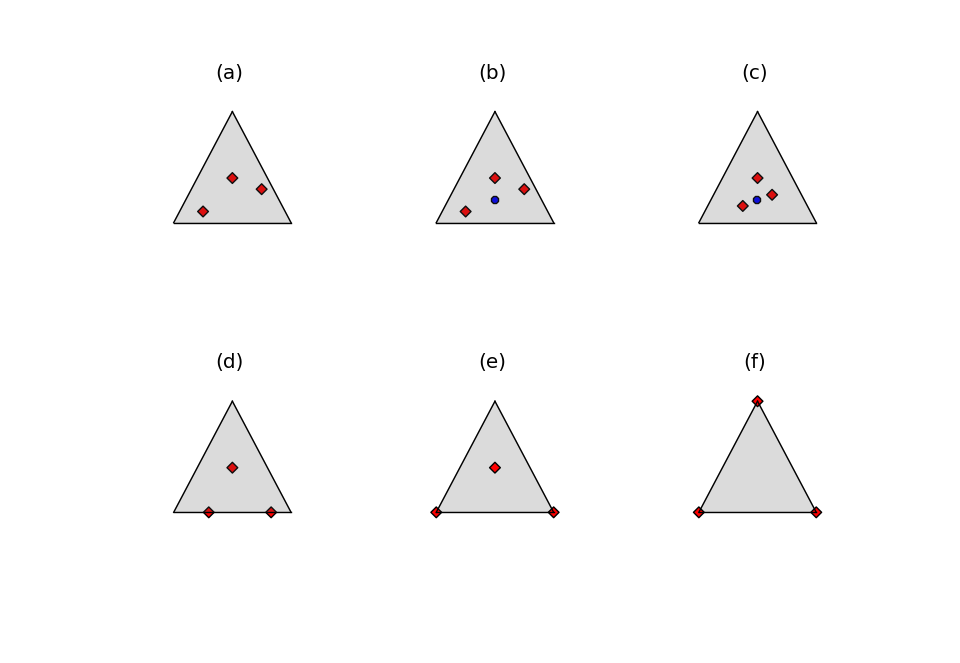}
\vspace{-4em}
\caption{In (a), we consider a demixing problem where there are three classes and $M=L$. The diamonds represent the mixture proportions of $\p{1},\p{2}$ and $\p{3}$. In (b), the blue circle is a random distribution chosen in the convex hull of two of the distributions (line 7). In (c), two of the distributions are resampled so that their residues wrt the other distribution are on the same face of the simplex (lines 12-15). In (d), these particular residues are computed (lines 12-15). In (e), two of the distributions are demixed (lines 3-5). In (f), the residue of the final distribution wrt the final two demixed distributions is computed to obtain the final demixing (line 18-21).}
\label{3_case}
\end{figure}

\subsection{Classification with Partial Labels}

As in the case of demixing mixed membership models, a simple resampling strategy works in certain nice settings of classification with partial labels. For example, consider an instance of classification with partial labels with the mixing matrix from equation \eqref{example_problem}. The resampling procedure that yields $\tilde{Q}_1, \tilde{Q}_2, \tilde{Q}_3$ (described in Section \ref{demixing_mixed_membership_models_population_section}) also works here. Nevertheless, as in demixing mixed membership models, this approach does not meet our goal of an algorithm that solves a broad class of mixing matrices and partial labels. 

Indeed, we observe that the partial labels do not provide enough information for choosing the resampling weights. Consider another instance of the problem with the same partial label matrix:
\begin{align}
\vec{\Pi} & = \begin{pmatrix}
\frac{1}{10} & \frac{9}{10} & 0 \\
\frac{9}{10} & 0 & \frac{1}{10} \\
0 & \frac{1}{10} & \frac{9}{10} 
\end{pmatrix}.
\label{example_problem_2}
\end{align}
Applying the resampling approach to \eqref{example_problem_2} can be shown to fail by observing that the inverse of the resampled mixing matrix does not satisfy condition \emph{3} of Lemma \ref{2_blanchard2014}. Thus, although the problem instances \eqref{example_problem} and \eqref{example_problem_2} have the same partial label matrix, the resampling procedure only works for one of these. 

\begin{algorithm}[t]
\caption{PartialLabel($\bPi^+, (\p{1}, \ldots, \p{M})^T$)}
\begin{algorithmic}[1]
\label{partial_label_alg}
\FOR{$i = 1, \ldots, L$}
\STATE $Q_i \longleftarrow \text{ uniformly random distribution in } \co(\p{1}, \ldots, \p{M})$ \label{randomization_partial_labels_line}
\ENDFOR
\FOR{$k=2,3,\ldots$}
\STATE $(W_1, \ldots, W_L)^T \longleftarrow \text{GenerateCandidates}(k, (Q_1, \ldots, Q_L)^T)$
\STATE $(\text{FoundVertices, } \vec{C}) \longleftarrow \text{VertexTest}(\bPi^+, \p{1}, \ldots, \p{M}, W_1, \ldots, W_L)$
\IF{FoundVertices}
\RETURN $\vec{C} (W_1, \ldots, W_L)^T$
\ENDIF
\ENDFOR
\end{algorithmic}
\end{algorithm}

\begin{algorithm}[t]
\caption{GenerateCandidates($k, (Q_1, \ldots, Q_L)^T$)}
\begin{algorithmic}[1]
\label{generate_candidate_alg}
\STATE Set $W_i \longleftarrow Q_i$ for all $i \in [K]$
\FOR{$i = 1, \ldots, L$}
\STATE $\bar{Q}_i \longleftarrow \frac{1}{L-1} [\sum_{j > i} Q_j + \sum_{j < i} W_j]$ \label{partial_label_alg_average}
\STATE $W_i \longleftarrow \sol(\frac{1}{k} Q_i + (1- \frac{1}{k}) \bar{Q}_i \, | \,   \set{Q_j }_{j > i} \cup \set{W_j}_{j < i})$ \label{solve_kappa_partial_labels_line}
\ENDFOR
\RETURN $ (W_1, \ldots, W_L)^T$
\end{algorithmic}
\end{algorithm}

\begin{algorithm}[t]
\caption{$\text{VertexTest}(\bPi^+, (\p{1}, \ldots, \p{M})^T, (Q_1, \ldots, Q_L)^T)$}
\begin{algorithmic}[1]
\label{vertex_test_alg}
\STATE Form the matrix $Y_{i,j} \coloneq \vec{1}\{\kappa^*(Q_i \, | \, Q_j) >0 \}$
\IF{$\vec{Y}$ has a non-zero off-diagonal entry}
\RETURN $(0, \vec{0})$
\ENDIF
\STATE Form the matrix $Z_{i,j} \coloneq \vec{1}\{\kappa^*(\p{i} \, | \, Q_j) > 0\}$
\STATE Use any algorithm that finds a permutation matrix $\vec{C}$ such that $\vec{Z} \vec{C} = \bPi^+$ (if it exists)
\IF{such a permutation matrix $\vec{C}$ exists}
\RETURN $(1, \vec{C}^T)$
\ELSE
\RETURN $(0, \vec{0})$
\ENDIF
\end{algorithmic}
\end{algorithm}
Next, we turn to presenting an algorithm that solves classification with partial labels for a wide class of mixing matrices and partial labels. We propose the PartialLabel algorithm (see Algorithm \ref{partial_label_alg}). PartialLabel proceeds by iteratively creating sets of candidate distributions $\vec{W} \coloneq (W_1, \ldots, W_L)^T$ via the subroutine CreateCandidates (see Algorithm \ref{generate_candidate_alg}). Given each $\vec{W}$, it runs an algorithm VertexTest (see Algorithm \ref{vertex_test_alg}) that uses $\tilde{\vec{P}}$ and the partial label matrix $\bPi^+$ to determine whether $\vec{W}$ is a permutation of the base distributions $\vec{P}$. If $\vec{W}$ is a permutation of $\vec{P}$, VertexTest constructs the corresponding permutation matrix for relating these distributions. If not, it reports failure and the PartialLabel algorithm increments $k$ and finds another set of candidate distributions. 

The VertexTest algorithm proceeds as follows on a vector of candidate distributions $\vec{Q} \coloneq (Q_1, \ldots, Q_L)^T$. First, it determines whether there are two distinct distributions $Q_i, Q_j$ such that $Q_i$ is not irreducible wrt $Q_j$, in which case $\vec{Q}$ cannot be a permutation of $\vec{P}$. If there is such a pair, it reports failure. Otherwise, it forms the matrix $Z_{i,j} \coloneq \vec{1}\{\kappa^*(\p{i} \, | \, Q_j) > 0\}$ and uses any algorithm that finds a permutation $\vec{C}$ (if it exists) of the columns of $\vec{Z}$ to match the columns of $\bPi^+$. If such a permutation $\vec{C}$ exists, it returns $\vec{C}^T$ and, as we show in Lemma \ref{vertex_test}, $\vec{C}^T \vec{Q} = \vec{P}$; otherwise, VertexTest reports failure.

We remark that finding the permutation in line 6 of Algorithm \ref{vertex_test_alg} is not NP-hard. One algorithm (but most likely not the most efficient) proceeds as follows: define a total ordering on the columns of binary matrices. Sort the columns of $\vec{Z}$ and $\bPi^+$ according to this total ordering. Check whether the resulting matrices are equal.

The following theorem gives our identification result for classification with partial labels.
\begin{thm}
\label{partial_identification}
Suppose that $P_1, \ldots, P_L$ are jointly irreducible, $\vec{\Pi}$ has full column rank, and the columns of $\bPi^+$ are unique. Then, $\text{PartialLabel}(\bPi^+, (\tilde{\vec{P}})^T)$ returns $\vec{R} \in \sP^L$ such that $\vec{R} = \vec{P}$. 
\end{thm}

There are two key ideas to the proof of Theorem \ref{partial_identification}. First, the randomization in line \ref{randomization_partial_labels_line} of Algorithm \ref{partial_label_alg} ensures through a linear independence argument that with probability $1$, the operation $\sol(\frac{1}{k} Q_i + (1- \frac{1}{k}) \bar{Q}_i \, | \,   \set{Q_j }_{j > i} \cup \set{W_j}_{j < i})$ in line \ref{solve_kappa_partial_labels_line} of Algorithm \ref{generate_candidate_alg} is well-defined. Second, in the GenerateCandidates algorithm, let $Q_j = \btau_j^T \bP$ and $W_j = \bgamma_j^TP$. We make the simple observation that the affine hyperplane given by $\bgamma_1, \ldots, \bgamma_{i-1}, \btau_{i+1}, \ldots, \btau_L$ bisects $\Delta_L$ such that $\btau_i$ and a nonempty subset of $\set{\be_1, \ldots, \be_L} \setminus \set{\bgamma_1,\ldots, \bgamma_{i-1}}$ are in the same halfspace. Using this observation, we show that for large enough $k$, $W_i$ is one of the base distributions and is distinct from all $W_j $ with $j < i$. 

The VertexTest algorithm connects the demixing problem and classification with partial labels by showing that any algorithm that solves the demixing problem can be used as a subroutine to solve classification with partial labels. For example, consider the following algorithm for classification with partial labels. First, use the Demix algorithm to obtain a permutation $\vec{Q}$ of the base distributions $\vec{P}$. Second, use VertexTest to find the permutation matrix relating $\vec{Q}$ and $\vec{P}$. This alternate algorithm is the basis of our finite sample algorithm for classification with partial labels (see Section \ref{partial_label_est_sec} for a more thorough discussion).

\begin{figure}[t]
\centering
\includegraphics[scale = .6]{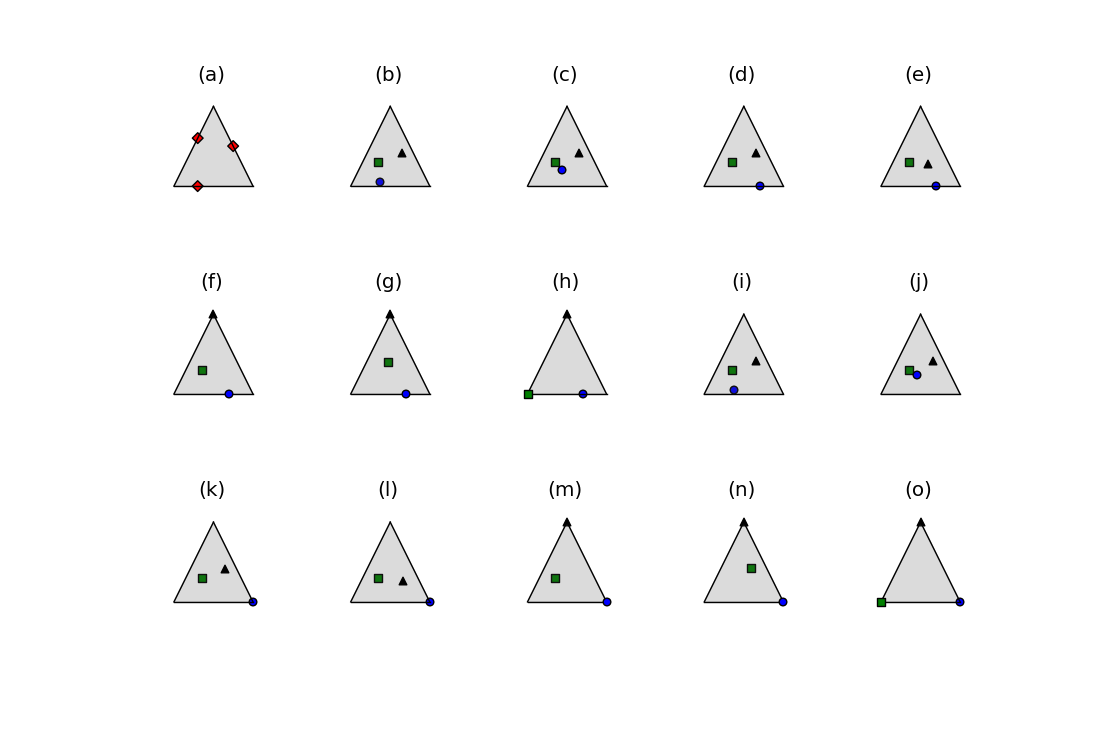}
\vspace{-4em}
\caption{(a) depicts an instance of the partial label problem where there are $L=3$ classes, $M=3$ partial labels, and each partial label only contains two of the classes. In (a), the red diamonds represent the mixture proportions of the distributions $\p{1}, \p{2}, \p{3}$. In (b), three distributions $Q_1, Q_2, Q_3$ are sampled uniformly randomly from the convex hull of $\p{1}, \p{2}, \p{3}$; the blue circle, black triangle, and green square represent their mixture proportions. Figures (c)-(h) show how the algorithm generates a set of candidate distributions $(W^{(2)}_1,W^{(2)}_2,W^{(2)}_3)^T$ with $k=2$. In (h), PartialLabel runs VertexTest on $(W^{(2)}_1,W^{(2)}_2,W^{(2)}_3)^T$ and determines that $(W^{(2)}_1,W^{(2)}_2,W^{(2)}_3)^T$ is not a permutation of $(P_1, P_2,P_3)^T$. In (i)-(o), PartialLabel begins again with $Q_1, Q_2, Q_3$ and executes the same series of steps with $k=3$, generating $(W^{(3)}_1,W^{(3)}_2,W^{(3)}_3)^T$. In (o), it runs VertexTest on $(W^{(3)}_1,W^{(3)}_2,W^{(3)}_3)^T$ and determines that $(W^{(3)}_1,W^{(3)}_2,W^{(3)}_3)^T$ is a permutation of $(P_1, P_2, P_3)^T$.}
\label{partial_label_figure}
\end{figure}

\section{Estimators for the Finite Sample Setting}

In this section, we develop the estimation theory to treat the three problems in the finite sample setting. Let $\sX = \bbR^d$ be equipped with the standard Borel $\sigma$-algebra $\sC$ and $\p{1}, \ldots, \p{L}$ be probability distributions on this space.  Suppose that we observe for $i = 1, \ldots, L$, 
\begin{align*}
X_1^i, \ldots, X_{n_i}^i \overset{i.i.d.}{\sim} \p{i}.
\end{align*}
Let $\sE$ be any Vapnik-Chervonenkis (VC) class with VC-dimension $V < \infty$, containing the set of all open balls, all open rectangles, or some other collection of sets that generates the Borel $\sigma$-algebra $\sC$. For example, $\sE$ could be the set of all open balls wrt the Euclidean distance, in which case $V = d+1$. Define $\epsilon_i(\delta_i) \equiv 3 \sqrt{\frac{V \log(n_i +1) - \log(\delta_i /2)}{n_i}}$ for $i = 1, \ldots, L$. Our estimators are based on the VC inequality \citep{devroye96}. This inequality says that for each $i \in [L]$, and $\delta >0$, the following holds with probability at least $1 - \delta$:
\begin{align*}
\sup_{E \in \sE}|\p{i}(E) - \ed{\tilde{P}}_i(E) | \leq \epsilon_i(\delta)
\end{align*}
where the \emph{empirical distribution} is given by $\ed{\tilde{P}}_i(E) = \frac{1}{n_i} \sum_{j=1}^{n_i} \ind{X^i_{j} \in E}$. 

\subsection{Multiclass Classification with Label Noise}

Let $\ed{F}_0, \ldots, \ed{F}_M$ denote the empirical distributions based on i.i.d. random samples from respective distributions $F_0, \ldots, F_M$. We introduce the following estimator of the multi-sample $\kappa^*$:
\begin{align}
\label{multisample_kappa_estimator}
\widehat{\kappa}(\ed{F}_0 \, | \, \ed{F}_1, \ldots, \ed{F}_M) & = \max_{\vec{\mu} \in \Delta_M} \inf_{E \in \sE} \frac{\ed{F}_0(E) + \epsilon_0(\frac{1}{n_0}) }{(\sum_{i=1}^M \mu_i \ed{F}_i(E) - \sum_i \mu_i \epsilon_i(\frac{1}{n_i}) )_+}
\end{align}
where the ratio is defined to be $\infty$ when the denominator is zero. This estimator arises from applying the VC inequality to the following expression:
\begin{align*}
\kappa^*(F_0 \, | \, F_1, \ldots, F_M) & = \max_{\bmu \in \si_M} \kappa^*(F_0 \, | \, \sum_{i=1}^M \mu_i F_i)  = \max_{\bmu \in \si_M} \inf_{E \in \sE,  \sum_{i=1}^M \mu_i F_i(E) > 0} \frac{F_0(E)}{ \sum_{i=1}^M \mu_i F_i(E)},
\end{align*} 
where the last equality uses Proposition \ref{label_noise_bin_case}. Let $\widehat{\vec{\mu}}$ denote a point where the maximum is achieved in \eqref{multisample_kappa_estimator}. Then, $\widehat{\vec{\nu}} \coloneq \widehat{\kappa} \widehat{\vec{\mu}}$ estimates the vector $(\nu_1, \ldots, \nu_M)$ attaining the maximum in (\ref{multisample_kappa}). See Proposition 2 of \citet{blanchard2014} to find a proof that the proposed estimator is consistent.

Based on this estimator, we introduce estimators that under the assumptions of Theorem \ref{multiclass_identification} converge to the base distributions uniformly in probability. Let $\vec{n} \equiv (n_1,\ldots, n_L)$; we write $\vec{n} \longrightarrow \infty$ to indicate that $\min_{i} n_i \longrightarrow \infty$.
\begin{thm}
\label{multiclass_estimation}
Let $(\widehat{\nu}_{i,j})_{j \neq i}$ be a vector attaining the maximum in the definition of $\widehat{\kappa}_i \coloneq \widehat{\kappa}(\ed{\tilde{P}}_i \, | \, \set{\ed{\tilde{P}}_j : j \neq i})$ and
\begin{align*}
\est{Q}_i = \frac{\ed{\tilde{P}}_i - \sum_{j \neq i} \widehat{\nu}_{i,j} \ed{\tilde{P}}_j }{1- \widehat{\kappa}_i}.
\end{align*}
Then, under the assumptions of Theorem \ref{multiclass_identification}, $\forall i = 1, \ldots, L$, $\sup_{E \in \sE} |\est{Q}_i(E) - P_i(E)| \overset{i.p.}{\longrightarrow} 0$ as $\vec{n} \longrightarrow \infty$. 
\end{thm}

\subsection{Demixing Mixed Membership Models}

In this section, we develop a novel estimator that can be used to extend the Demix algorithm to the finite sample case. Uniform convergence results typically assume access to i.i.d. samples. The challenge of developing an estimator for Demix is that because of the recursive nature of the Demix algorithm, we cannot assume access to i.i.d. samples to estimate every distribution that arises. Nonetheless, we show that uniform convergence of distributions propagates through the algorithm if we employ an estimator of $\kappa^*$ with a known rate of convergence. 

Let $\est{F}$ and $\est{H}$ be estimates of distributions $F$ and $H$, respectively. We introduce the following estimator: 
\begin{align*}
\est{\kappa}( \est{F} \, | \, \est{H}) = \inf_{E \in \sE} \frac{\est{F}(E) + \vc }{(\est{H}(E) - \vc)_+}
\end{align*}
where $\vc = \sum_{i=1}^L \epsilon_i(\frac{1}{n_i})$. Our estimator is closely related to the estimator from \citet{blanchard2010}: if $\est{F}$ and $\est{H}$ are empirical distributions, e.g., $\est{F} = \ed{\tilde{P}}_i$ and $\est{H} = \ed{\tilde{P}}_j$, then their estimator for $\kappa^*(F \, | \, H)$ is $\inf_{E \in \sE} \frac{\est{F}(E) + \epsilon_i(\frac{1}{n_i})  }{(\est{H}(E) - \epsilon_j(\frac{1}{n_j}))_+}$. Note that our proofs only require that $\vc$ include the terms $\epsilon_i(\frac{1}{n_i})$ corresponding to $\p{i}$ that the estimators $\est{F}$ and $\est{H}$ use samples from; to simplify presentation, however, we include all the terms, which leads to bounds that are looser by only a constant factor. 

Based on the estimator $\est{\kappa}$, we introduce the following estimator of the residue of $F$ wrt $H$.
\begin{algorithm}[t]
\caption{ResidueHat($\est{F} \, | \, \est{H}$)}
\label{residue_hat}
\textbf{Input: }$\est{F}, \est{H} \text { are estimates of } F,H$
\begin{algorithmic}[1]
\STATE $\est{\kappa} \longleftarrow \est{\kappa}(\est{F} \, | \, \est{H})$
\RETURN $\frac{\est{F} - \est{\kappa}(1 - \est{H})}{1 - \est{\kappa}}$
\end{algorithmic}
\end{algorithm}

\begin{defn}
Let $\est{F}$ and $\est{H}$ be estimators of $F$ and $H$, respectively, where $F \neq H$ and let $G \longleftarrow \text{Residue}(F \, | \, H)$ and $\est{G} \longleftarrow \text{ResidueHat}(\est{F} \, | \, \est{H})$. We call $\est{G}$ a \emph{ResidueHat estimator} of order $k \geq 1$ if (i) $F, H \in \co(P_1, \ldots, P_L)$, and (ii) at least one of $\est{F}$ and $\est{H}$ is a ResidueHat estimator of order $k-1$ and the other is either an empirical distribution or a ResidueHat estimator of order less than or equal to $k-1$. We call $\est{G}$ a \emph{ResidueHat estimator} of order $0$ if (i) holds, and $\est{F}$ and $\est{H}$ are empirical distributions.
\end{defn}
\noindent Note that the above definition is recursive and matches the recursive structure of the Demix algorithm. We suppress the qualifier ``of order $k$" when it is not relevant. 

To use ResidueHat estimators to estimate the $P_i$s, we build on the rate of convergence result from \citet{scott2015}. In \citet{scott2015}, a rate of convergence was established for an estimator of $\kappa^*$ using empirical distributions; we extend these results to our setting of recursive estimators and achieve the same rate of convergence. To ensure that this rate of convergence holds for every estimate in our algorithm, we introduce the following condition.
\begin{description}
\item[(A$''$)] $P_1, \ldots, P_L$ are such that $\forall i$ $\supp(P_i) \not \subseteq \cup_{j \neq i} \supp(P_j) $. 
\end{description}
Note that this assumption implies joint irreducibility and is a natural generalization of the separability assumption.

\begin{algorithm}[t]
\caption{DemixHat($\est{S}_1, \ldots, \est{S}_K \, | \, \epsilon$)}
\textbf{Input: }$\est{S}_1, \ldots, \est{S}_K$ are ResidueHat estimates
\begin{algorithmic}[1]
\label{demix_hat_alg}
\IF{$K = 2$}
\RETURN $(\text{ResidueHat}(\est{S}_1 \, | \, \est{S}_2), \text{ResidueHat}(\est{S}_2 \, | \, \est{S}_1))^T$
\ELSE
\STATE $(R_2,\ldots, R_K)^T \longleftarrow \text{FindFaceHat}(\est{S}_1, \ldots, \est{S}_K \, | \, \epsilon)$
\STATE $(\est{Q}_1, \cdots, \est{Q}_{K-1})^T \longleftarrow \text{DemixHat}(\est{R}_2,\cdots,\est{R}_K)$\\
\STATE $\est{Q}_K \longleftarrow \frac{1}{K} \sum_{i=1}^K \est{S}_i$\\
\FOR{$i = 1, \ldots, K-1$}
\STATE $\est{Q}_K \longleftarrow \text{ResidueHat}(\est{Q}_K \, | \, \est{Q}_i)$
\ENDFOR
\RETURN $(\est{Q}_1, \cdots, \est{Q}_K)^T$
\ENDIF
\end{algorithmic}
\end{algorithm}

The following result establishes sufficient conditions under which ResidueHat estimates converge uniformly.
\begin{prop}
\label{uniform_convergence}
If $P_1, \ldots, P_L$ satisfy \textbf{(A$''$)} and $\est{G}$ is a ResidueHat estimator of a distribution $G \in \co(P_1, \ldots, P_L)$, then $\sup_{E \in \sE} | \est{G}(E) - G(E)| \overset{i.p.}{\longrightarrow} 0$ as $\vec{n} \longrightarrow \infty$.
\end{prop}
Based on the ResidueHat estimators, we introduce an empirical version of the Demix algorithm---DemixHat (see Algorithm \ref{demix_hat_alg}). The main differences are that \emph{(i)} we replace the Residue function with the ResidueHat function, \emph{(ii)} we replace line \ref{multi_sample_kappa_demix} in the Demix algorithm with a sequence of applications of the two-sample $\kappa^*$ operator, as mentioned just before Theorem \ref{demix_identification}, and \emph{(iii)} DemixHat requires specification of a hyperparameter $\epsilon \in (0,1)$. We replace the multi-sample $\kappa^*$ with the two-sample $\kappa^*$ because there is no known estimator with a rate of convergence for the multi-sample $\kappa^*$, and the rate of convergence is essential to our consistency proof. The hyperparameter $\epsilon$ gives a tradeoff between runtime and accuracy. The runtime increases with increasing $\epsilon$, but the amount of uncertainty about whether DemixHat executes successfully decreases with increasing $\epsilon$.

\begin{algorithm}[t]
\caption{FindFaceHat($\est{S}_1, \ldots, \est{S}_K \, | \, \epsilon$)}
\textbf{Input: }$\est{S}_1, \ldots, \est{S}_K$ are ResidueHat estimates
\begin{algorithmic}[1]
\label{findface_hat_alg}
\STATE $\est{Q} \longleftarrow \text{ uniformly distributed random element from } \co (\est{S}_2, \ldots, \est{S}_K)$
\FOR{$n=2,3, \ldots$} 
\STATE Set $\est{R}_i \longleftarrow \text{ResidueHat}(\frac{1}{n}\est{S}_i +\frac{n-1}{n} \est{Q}) \, | \, \est{S}_1)$ for all $i \in \set{2,\ldots, K}$\label{R_hat_demixhat}
\IF{$\text{FaceTestHat}(\est{R}_2, \cdots, \est{R}_K \, | \, \epsilon)$}	
\RETURN $(\est{R}_2, \cdots, \est{Q}_K)^T$
\ENDIF
\ENDFOR
\end{algorithmic}
\end{algorithm}

\begin{algorithm}[t]
\caption{FaceTestHat($\est{Q}_1,\cdots, \est{Q}_K \, | \, \epsilon$)}
\begin{algorithmic}[1]
\label{face_test_hat_alg}
\STATE Set $\vec{Z}_{i,j} \coloneq \vec{1}\{\est{\kappa}(\est{Q}_i \, | \, \est{Q}_j) > \epsilon \}$ for $i\neq j$
\IF{$\vec{Z}$ has a zero off-diagonal entry}
\RETURN $0$
\ELSE
\RETURN $1$
\ENDIF
\end{algorithmic}
\end{algorithm}

We now state our main estimation result. 
\begin{thm}
\label{demix_estim}
Let $\delta > 0$ and $\epsilon \in (0,1)$. Suppose that $P_1, \ldots, P_L$ satisfy \textbf{(A$''$)} and $\vec{\Pi}$ has full rank. Then, with probability tending to $1$ as $\vec{n} \longrightarrow \infty$, DemixHat($\ed{\tilde{P}}_1, \ldots, \ed{\tilde{P}}_L \, | \, \epsilon$) returns $(\est{Q}_1,\ldots, \est{Q}_L)$ for which there exists a permutation $\sigma: [L] \longrightarrow [L]$ such that for every $i \in [L]$,
\begin{align*}
\sup_{E \in \sE} |\est{Q}_i(E) - P_{\sigma(i)}(E)| < \delta.
\end{align*}
\end{thm}

\subsection{Classification with Partial Labels}
\label{partial_label_est_sec}

In this section, we present a finite sample algorithm for the decontamination of a partial label model (see Algorithm \ref{partial_label_hat_alg}). This algorithm is based on a different approach from PartialLabel (Algorithm \ref{partial_label_alg}): it combines DemixHat with an empirical version of the VertexTest algorithm (see Algorithm \ref{vertex_test_hat_alg}). The reason for this is that we have an estimator with a rate of convergence for the two-sample $\kappa^*$, whereas there is no known estimator with a rate of convergence for the multi-sample $\kappa^*$. We leverage this rate of convergence to prove the consistency of our algorithm.

We make an assumption that simplifies our algorithm: $\bPi^+$ satisfies
\begin{description}
\item[(D)] there does not exist $i,j \in [L]$ such that $\bPi^+_{i,:} = \vec{e}_j^T$.
\end{description}

\noindent In words, this says that there is no contaminated distribution $\p{i}$ and base distribution $P_j$ such that $\p{i} = P_j$. We emphasize that we make this assumption only to simplify the presentation and development of the algorithm; one can reduce any instance of a partial label model satisfying \textbf{(B3)} and \textbf{(A)} to an instance of a partial label model that also satisfies \textbf{(D)}. We defer the sketch of this reduction to Section \ref{partial_label_estimation}.

\begin{algorithm}[t]
\caption{$\text{PartialLabelHat}(\bPi^+, (\ed{\tilde{P}}_1, \ldots, \ed{\tilde{P}}_M)^T \, | \, \epsilon)$}
\begin{algorithmic}[1]
\label{partial_label_hat_alg}
\STATE $(\est{Q}_1, \ldots, \est{Q}_L)^T \longleftarrow \text{DemixHat}(\ed{\tilde{P}}_1, \ldots, \ed{\tilde{P}}_M \, | \, \epsilon)$ 
\STATE $(\text{FoundVertices}, \vec{C}) \longleftarrow \text{VertexTestHat}(\bPi^+, (\ed{\tilde{P}}_1, \ldots, \ed{\tilde{P}}_M)^T, (\est{Q}_1, \ldots, \est{Q}_L)^T)$
\RETURN $\vec{C} (\est{Q}_1, \ldots, \est{Q}_L)^T$
\end{algorithmic}
\end{algorithm}

We now state our main estimation result for classification with partial labels.

\begin{thm}
\label{partial_label_hat}
Let $\delta > 0$ and $\epsilon \in (0,1)$. Suppose that $P_1, \ldots, P_L$ satisfy \textbf{(A$''$)}, $\vec{\Pi}$ has full rank, the columns of $\bPi^+$ are unique and  $\bPi^+$ satisfies \textbf{(D)}. Then, with probability tending to $1$ as $\vec{n} \longrightarrow \infty$, PartialLabelHat($ \bPi^+, (\ed{\tilde{P}}_1, \ldots, \ed{\tilde{P}}_M)^T \, | \, \epsilon$) returns $(\est{Q}_1,\ldots, \est{Q}_L)^T$ such that for every $i \in [L]$,
\begin{align*}
\sup_{E \in \sE} |\est{Q}_i(E) - P_i(E)| < \delta.
\end{align*}
\end{thm}

\subsection{Sieve Estimators}

In the preceding, we have assumed a fixed VC class to simplify the presentation. However, these results easily extend to the setting where $\sE = \sE_k$ and $k \longrightarrow \infty$ at a suitable rate depending on the growth of the VC dimensions $V_k$. This allows for the $P_i$s to be estimated uniformly on arbitrarily complex events, e.g., $\sE_k$ is the set of unions of $k$ open balls.

\section{Discussion}

In this paper, we have studied the problem of how to recover the base distributions $\vec{P}$ from the contaminated distributions $\tilde{\vec{P}}$ without knowledge of the mixing matrix $\vec{\Pi}$. We used a common set of concepts and techniques to solve three popular machine learning problems that arise in this setting: multiclass classification with label noise, demixing of mixed membership models, and classification with partial labels. Our technical contributions include: (i) We provide sufficient and sometimes necessary conditions for identifiability for all three problems. (ii) We give nonparametric algorithms for the infinite and finite sample settings. (iii) We provide a new estimator for iterative applications of $\kappa^*$ that is of independent interest. (iv) Finally, our work provides a novel geometric perspective on each of the three problems.

Our results improve on what was previously known for all three problems. For multiclass classification with label noise and unknown $\vec{\Pi}$, previous work had only considered the case $M=L=2$. Our work achieves a generalization to arbitrarily many distributions. For demixing of mixed membership models, previous algorithms with theoretical guarantees required a finite sample space. Our work allows for a much more general set of distributions. Finally, for classification with partial labels, previous work on learnability assumed the realizable case (non-overlapping $P_1, \ldots, P_L$) and assumed strong conditions on label co-occurence in partial labels. Our analysis covers the agnostic case and a much wider set of partial labels.

Our work has also highlighted the advantages and disadvantages associated with the two-sample $\kappa^*$ operator and multi-sample $\kappa^*$ operator, respectively. Algorithms that only use the two-sample $\kappa^*$ operator have the following two advantages: (i) the geometry of the two-sample $\kappa^*$ operator is simpler than the geometry of the multi-sample $\kappa^*$ operator and, as such, can be more tractable. Indeed, in recent years, several practical algorithms for estimating the two-sample $\kappa^*$ have been developed (see \citet{jain2016} and references therein). (ii) We have estimators with established rates of convergence for the two-sample $\kappa^*$ operator, but not for the multi-sample $\kappa^*$ operator. On the other hand, algorithms that use the multi-sample $\kappa^*$ operator have fewer steps.

The aims of this work are mainly theoretical, but we believe that our work can inform practical algorithms. First, we note that while we have emphasized recovery of $\vec{P}$, another interpretation is that our work deals with estimating $\bPi$. One can then plug our estimate of $\bPi$ into corrected losses for multiclass classification with label noise and classification with partial labels that require knowledge of $\bPi$ \citep{cidsueiro2012, menon15icml, vanrooyan2015arxix, patrini2017}. Thus, in general, our work can be applied in this two-stage approach. Second, when $L$ or $M$ are small, the Demix algorithm is practical. For example, \citet{metodiev2018jet} apply the Demix algorithm to a high-energy physics application where $L=M=2$. It is of interest to examine more generally whether variants of Demix work when $M$ or $L$ are small. Third, we conjecture that our analysis suggests novel principles for designing algorithms. For example, an alternative approach to the three problems in question is to embed the contaminated distributions in a reproducing kernel Hilbert Space and to estimate the $\bPi$ matrix by setting up an optimization problem (e.g., see \citet{ramaswamy2016} for the setting where there are two distributions). One of our necessary conditions, maximality (see Appendix \ref{fact_res_app}), suggests formulating the optimization problem to search for base distributions that \emph{(i)} explain the observed contaminated distributions and \emph{(ii)} whose convex hull is as large as possible. In this way, we believe that our general treatment of these three problems that arise in mutual contamination models provides intuitions that could be useful for designing new algorithms.

Although our experiments in Appendix \ref{experiments_section} suggest that joint irreducibility of the base distributions is a reasonable assumption, it is nevertheless worthwhile to consider what questions arise if the base distributions are not exactly jointly irreducible. We see two possible research directions. First, one could perform a stability analysis: when the base distributions are not jointly irreducible, but are nearly jointly irreducible (in some sense that would need to be defined precisely), does the estimate of $\bPi$ remain close to the true $\bPi$? A second research question is to reinterpret the problem of demixing of mixed membership models as a dimensionality reduction problem. That is, given a large set of distributions, one could seek to represent them as convex combinations of a small set of irreducible base distributions. Then, the challenge would be to define an appropriate measure of approximation quality and to determine whether our approach could be useful for designing a consistent algorithm for the best approximation.

\subsubsection*{Acknowledgements}

We thank the anonymous reviewers for their extremely helpful and insightful comments. This work was supported in part by NSF grants 1422157 and 1838179. Gilles Blanchard acknowledges support from the DFG, under the Research Unit FOR-1735 ``Structural Inference in Statistics - Adaptation and Efficiency'', and under the Collaborative Research Center SFB-1294 ``Data Assimilation''.

\appendix

\section{Outline of Appendices}

To begin, we introduce additional notation for the appendices. In Section \ref{fact_res_app}, we discuss how strong our sufficient conditions are and present factorization results that suggest that they are reasonable. In Section \ref{ident_app}, we give our identifiability analysis of demixing mixed membership models and classification with partial labels. In Section \ref{est_app}, we prove our results on the ResidueHat estimator, as well as the finite sample algorithms for demixing mixed membership models and classification with partial labels. In Section \ref{prev_app}, we state some lemmas from related papers that we use in our arguments.

\section{Notation for Appendices}
Let $A \subset \bbR^d$ be a set. Let $\aff A$ denote the affine hull of $A$, i.e., 
$\aff A = \set{\sum_{i=1}^K \theta_i \vec{x}_i | \vec{x}_1, \ldots , \vec{x}_K \in A, \sum_{i=1}^K \theta_i = 1}$. $A^\circ$ denotes the relative interior of $A$, i.e., $A^\circ = \set{x \in A | B(x,r) \cap \aff A \subseteq A \text{ for some } r > 0}$. Then, $\partial A$ denotes the relative boundary of $A$, i.e., $\partial A = A \setminus A^\circ$. In addition, let $\norm{\cdot}$ denote an arbitrary finite-dimensional norm on $\bbR^L$. For two vectors, $\vec{x}, \vec{y} \in \bbR^K$, define
\begin{align*}
\min(\vec{x}^T, \vec{y}^T) & = (\min(x_1, y_1), \ldots, \min(x_K, y_K)).
\end{align*}
$\vec{x} \geq \vec{y}$ means $x_i \geq y_i$ $\forall i \in [K]$.

For distributions $Q_1, \ldots, Q_k$, we use $\co(Q_1, \ldots, Q_K)^\circ$ to denote the relative interior of their convex hull and have that
\begin{align*}
\co(Q_1, \ldots, Q_K)^\circ & = \set{\sum_{i=1}^K \alpha_i Q_i : \alpha_i > 0, \sum_{i=1}^K \alpha_i = 1}.
\end{align*}
Note that when $Q_1, \ldots, Q_K$ are discrete distributions, this definition coincides with the definition of the relative interior of a set of Euclidean vectors.

We use the following affine mapping throughout the paper: $m_\nu(\vec{x}, \vec{y}) = (1-\nu) \vec{x} + \nu\vec{y}$ where $\vec{x}, \vec{y} \in \bbR^L$ and $\nu \in [0,1]$. Overloading notation, when $Q_1$ and $Q_2$ are distributions, we define $m_\nu(Q_1, Q_2) = (1-\nu)Q_1 + \nu Q_2$. Note that if $\vec{\eta}_1, \vec{\eta}_2 \in \si_L$ and $Q_1 = \vec{\eta}_1^T \vec{P}$ and $Q_2 = \vec{\eta}_2^T \vec{P}$, then $m_\nu(\vec{\eta}_1, \vec{\eta}_2)$ is the mixture proportion for the distribution $m_\nu(Q_1, Q_2)$.

\section{Factorization Results}
\label{fact_res_app}

In this section, we discuss whether our sufficient conditions are necessary. For the problems of demixing mixed membership models and classification with partial labels, we provide factorization results that suggest that our sufficient conditions are not much stronger than what is necessary.

\subsection{Multiclass Classification with Label Noise}

Our sufficient condition \textbf{(B1)} for multiclass classification with label noise is not necessary. Rather, \textbf{(B1)} is one of several possible sufficient conditions, and one that reflects a low noise assumption as illustrated in Figure \ref{fig:A}. Consider the case $L=M=2$ where \textbf{(B1)} is equivalent to $\pi_{1,2} + \pi_{2,1} < 1$. Recovery is still possible if $\pi_{1,2} + \pi_{2,1} > 1$ since one can simply swap $\p{1}$ and $\p{2}$ in a decontamination procedure. $\pi_{1,2} + \pi_{2,1} < 1$ is only necessary if one assumes that most of the training labels are correct, which is what $\pi_{1,2} + \pi_{2,1} < 1$ essentially says. For larger $L = M$, \textbf{(B1)} says in a sense that most of the data from $\p{i}$ come from $P_i$ for every $i$. Other sufficient conditions are possible (as in the binary case), but these would require at least one $\tilde{P}_i$ to contain a significant portion of some $P_j$, $j \neq i$. Regarding \textbf{(A)}, \citet{blanchard2016} study the question of necessity for joint irreducibility in the case $L=M=2$ and show that under mild assumptions on the decontamination procedure, joint irreducibility is necesssary.

\subsection{Demixing Mixed Membership Models}

Recall the definition of a factorization: $\tilde{\vec{P}}$ is factorizable if there exists $(\vec{\Pi}, \vec{P}) \in \Delta_L^M \times \sP^L$ such that $\tilde{\vec{P}} = \vec{\Pi} \vec{P}$; we call $(\vec{\Pi}, \vec{P})$ a factorization of $\tilde{\vec{P}}$.  

Our sufficient conditions are not much stronger than what is required by factorizations that satisfy the two forthcoming desirable properties. 
\begin{defn}
\label{max_demix}
We say a factorization $(\vec{\Pi}, \vec{P})$ of $\tilde{\vec{P}}$ is \emph{maximal} \textbf{(M)} iff for all factorizations $(\vec{\Pi}^\prime, \vec{P}^\prime)$ of $\tilde{\vec{P}}$ with $\vec{P}^\prime = (P^\prime_1,\ldots, P^\prime_L)^T \in \mathcal{P}^L$, it holds that $\set{P^\prime_1,\ldots, P^\prime_L} \subseteq \co(P_1,\ldots, P_L)$.
\end{defn} 
\noindent In words, $\vec{P} = (P_1, \ldots, P_L)^T$ is a maximal collection of base distributions if it is not possible to move any of the $P_i$s outside of $\co(P_1, \ldots,P_L)$ and represent $\tilde{\vec{P}}$. 
\begin{defn}
\label{lin_demix}
 We say a factorization $(\vec{\Pi}, \vec{P})$ of $\tilde{\vec{P}}$ is \emph{linear} \textbf{(L)} iff $\set{P_1, \ldots, P_L} \subseteq \spa(\p{1}, \ldots, \p{M})$. 
\end{defn}
\noindent We believe that \textbf{(L)} is a reasonable requirement because it holds in the common situation in which there exist $\vec{\pi}_{i_1}, \ldots, \vec{\pi}_{i_L}$ that are linearly independent. Then for $I = \set{i_1, \ldots, i_L}$, we can write $\tilde{\vec{P}}_I = \vec{\Pi}_I \vec{P}$ where $\vec{\Pi}_I$ is the submatrix of $\vec{\Pi}$ containing only the rows indexed by $I$ and  $\tilde{\vec{P}}_I$ is similarly defined. Then, $\vec{\Pi}_I$ is invertible and $\vec{P} = \vec{\Pi}_I^{-1} \tilde{\vec{P}}_I$.

Factorizations that satisfy \textbf{(A)} and \textbf{(B2)} are maximal and linear.
\begin{prop}
\label{max_ji_consistency}
Let $(\vec{\Pi}, \vec{P})$ be a factorization of $\tilde{\vec{P}}$. If $(\vec{\Pi}, \vec{P})$ satisfies \textbf{(A)} and \textbf{(B2)}, then $(\vec{\Pi}, \vec{P})$ satisfies \textbf{(M)} and \textbf{(L)}. 
\end{prop}
\begin{proof}
We first show that $(\vec{\Pi}, \vec{P})$ satisfies \textbf{(L)}. By hypothesis, $P_1, \ldots, P_L$ are jointly irreducible. By Lemma \ref{B_1}, $P_1, \ldots, P_L$ are linearly independent. Since by hypothesis $\vec{\Pi}$ has full rank, there exist $L$ rows in $\vec{\Pi}$, $\vec{\pi}_{i_1}, \ldots, \vec{\pi}_{i_L}$, that are linearly independent. By Lemma \ref{B_1}, $\p{{i_1}}, \ldots, \p{{i_L}}$ are linearly independent. Since $\vec{\Pi} \vec{P} = \tilde{\vec{P}}$, $\spa(\p{1}, \ldots, \p{M}) \subseteq \spa(P_1, \ldots, P_L) $. Since $\dim \spa(\p{1}, \ldots, \p{M}) \geq L$, we have $\spa(\p{1}, \ldots, \p{M}) = \spa(P_1, \ldots, P_L)$. Therefore, $(\vec{\Pi}, \vec{P})$ satisfies \textbf{(L)}.

Now, we show that $(\vec{\Pi}, \vec{P})$ satisfies \textbf{(M)}. Suppose that there is another solution $(\vec{\Pi}^\prime, \vec{P}^\prime)$ with $\vec{P}^\prime = (P^\prime_1, \ldots, P^\prime_L)^T$ such that $\vec{\Pi}^\prime \vec{P}^\prime = \tilde{\vec{P}}$ and with some $P^\prime_i$ such that $P_i^\prime \not \in \co(P_1, \ldots, P_L)$. We claim that $P_i^\prime \not \in \spa(P_1, \ldots, P_L)$. Towards a contradiction, suppose that $P_i^\prime \in \spa(P_1, \ldots, P_L)$ so that we can write $P_i^\prime = \sum_{i=1}^L a_i P_i$. Then, at least one of the $a_i$ is negative since, by assumption, $P_i^\prime \not \in \co(P_1, \ldots, P_L)$. But, by joint irreducibility of $P_1, \ldots, P_L$, $P_i^\prime$ is not a distribution, which is a contradiction. So, the claim follows. But, then, since $\spa(P_1, \ldots, P_L) = \spa(\p{1}, \ldots, \p{M})$, we must have that $\spa(\p{1}, \ldots, \p{M}) \subseteq \spa(P_1^\prime, \ldots, P_{i-1}^\prime, P_{i+1}^\prime, \ldots, P_L)$, which is impossible since $\dim \spa(\p{1}, \ldots, \p{M}) = L$.
\end{proof}
Maximal and linear factorizations imply conditions that are not much weaker than our sufficient conditions.
\begin{thm}
\label{nec_cond}
Let $(\vec{\Pi}, \vec{P})$ be a factorization of $\tilde{\vec{P}}$. If $(\vec{\Pi}, \vec{P})$ satisfies \textbf{(M)}, then
\begin{description}
\item[(A$'$)] $\forall i$, $P_i$ is irreducible with respect to every distribution in $\co(\set{P_j : j \neq i})$.
\end{description}
If $(\vec{\Pi}, \vec{P})$ satisfies \textbf{(L)}, then 
\begin{description}
\item[(B$'$)] $\rank(\vec{\Pi}) \geq \dim \spa(P_1, \ldots, P_L)$.
\end{description}
\end{thm}
\begin{proof}
\begin{description}
\item[(A$'$)] We prove the contrapositive. Suppose that there is some $P_i$ and $Q \in \co(\set{P_j : j \neq i})$ with $Q = \sum_{j \neq i} \beta_j P_j$ such that $P_i$ is not irreducible wrt $Q$. Then, there is some distribution $G$ and $\gamma \in (0,1]$ such that $P_i = \gamma Q + (1-\gamma) G$. 

Suppose $\gamma = 1$. Then, $P_i = Q \in \co(\set{P_k : k \neq i})$. But, then $\p{1}, \ldots, \p{M} \in \co(\set{P_j : j \neq i} \cup \set{R})$ for any distribution $R \not \in \co(P_1, \ldots, P_L)$. This shows that $(\vec{\Pi}, \vec{P})$ does not satisfy \textbf{(M)}.  

Therefore, assume that $\gamma \in (0,1)$. Either $G \in \co(P_1, \ldots, P_L)$ or $G \not \in \co(P_1, \ldots, P_L)$. Suppose that $G \in \co(P_1, \ldots, P_L)$. Then, there exist $\alpha_1, \ldots, \alpha_L$ all nonnegative and summing to $1$ such that 
\begin{align*}
P_i & = \gamma Q + (1-\gamma)(\alpha_1 P_1 + \ldots + \alpha_L P_L).
\end{align*}
\noindent Therefore, $P_i \in \co(\set{P_k : k \neq i})$. Then, by the argument in the previous paragraph, $(\vec{\Pi}, \vec{P})$ does not satisfy \textbf{(M)}.

Now, suppose that $G \not \in \co(P_1, \ldots, P_L)$. Since $P_i \in \co(G, Q)$ and $Q \in \co(\set{P_j : j \neq i})$, we have that $\co(\set{P_j : j \neq i} \cup \set{G}) \supset \co(P_1, \ldots, P_L)$. Then, $\p{1}, \ldots, \p{M} \in \co(\set{P_j : j \neq i} \cup \set{G})$. This shows that $(\vec{\Pi}, \vec{P})$ does not satisfy \textbf{(M)}. The result follows.

\item[(B$'$)] Clearly, $\dim \spa (\p{1}, \ldots , \p{M}) \leq \dim \spa(P_1, \ldots, P_L)$ since $\p{i} = \vec{\pi}_i^T \vec{P}$ for all $i \in [M]$. Since $(\vec{\Pi}, \vec{P})$ satisfies \textbf{(L)}, $\spa (P_1, \ldots, P_L) \subset \spa (\p{1}, \ldots, \p{L})$, which implies that
\begin{align*}
\dim \spa(P_1, \ldots, P_L) \leq \dim \spa (\p{1}, \ldots , \p{M}).
\end{align*}
Therefore, $\dim \spa (\p{1}, \ldots , \p{M}) = \dim \spa(P_1, \ldots, P_L)$. Then, since $\p{1}, \ldots, \p{M} \in \range(\vec{\Pi})$, $\dim \spa (P_1, \ldots , P_L) \leq \dim \range \vec{\Pi}$. By Result 3.117 of \citet{axler}, 
\begin{align*}
\rank (\vec{\Pi}) = \dim \text{range} \, \vec{\Pi} \geq \dim \spa(P_1, \ldots, P_L).
\end{align*}
\end{description}
\end{proof}
\noindent As a corollary, Theorem \ref{nec_cond} implies that if there is a linear factorization $(\vec{\Pi}, \vec{P})$ of $\tilde{P}$ and $P_1, \ldots, P_L$ are linearly independent, then there must be at least as many contaminated distributions as base distributions, i.e., $M \geq L$. Also, note that \textbf{(A$'$)} appears as a sufficient condition in \citet{sanderson2014}.

By comparing \textbf{(A)} with \textbf{(A$'$)} and \textbf{(B2)} with \textbf{(B$'$)}, we see that the proposed sufficient conditions are not much stronger than \textbf{(M)} and \textbf{(L)} require. Since joint irreducibility of $P_1, \ldots, P_L$ entails their linear independence by Lemma \ref{B_1}, under \textbf{(A)}, \textbf{(B2)} and \textbf{(B$'$)} are the same. \textbf{(A)} differs from \textbf{(A$'$)} in that it requires that a slightly larger set of distributions are irreducible with respect to convex combinations of the remaining distributions. Specifically, under \textbf{(A)}, \emph{every convex combination} of a subset of the $P_i$s is irreducible with respect to every convex combination of the other $P_i$s whereas \textbf{(A$'$)} only requires that every $P_i$ be irreducible with respect to every convex combination of the other $P_i$s. 

\subsection{Classification with Partial Labels}

Most of our definitions and results for classification with partial labels parallel those of demixing mixed membership models. We say that $\tilde{\vec{P}}$ is $\bPi^+$-factorizable if there exists a pair $(\vec{\Pi}, \vec{P}) \in \Delta_L^M \times \sP^L$ that solves (\ref{model}) such that $\vec{\Pi}$ is consistent with $\bPi^+$; we call $(\vec{\Pi}, \vec{P})$ an $\bPi^+$-factorization of $\tilde{\vec{P}}$. We say a partial label model is identifiable if given $(\tilde{\vec{P}}, \bPi^+)$, $\tilde{\vec{P}}$ has a unique $\bPi^+$-factorization $(\vec{\Pi}, \vec{P})$.

Our definitions of maximal and linear $\bPi^+$-factorizations resemble definitions \ref{max_demix} and \ref{lin_demix}.

\begin{defn}
We say a $\bPi^+$-factorization $(\vec{\Pi}, \vec{P})$ of $\tilde{\vec{P}}$ is \emph{maximal} \textbf{(M)}  iff for all $\bPi^+$-factorizations $(\vec{\Pi}^\prime, \vec{P}^\prime)$ of $\tilde{\vec{P}}$ with $\vec{P}^\prime = (P^\prime_1,\ldots, P^\prime_L)^T \in \mathcal{P}^L$, it holds that $\set{P^\prime_1,\ldots, P^\prime_L} \subseteq \co(P_1,\ldots, P_L)$.
\end{defn}

\begin{defn}
\sloppy We say a $\bPi^+$-factorization $(\vec{\Pi}, \vec{P})$ of $\tilde{\vec{P}}$ is \emph{linear} \textbf{(L)} iff $\set{P_1, \ldots, P_L} \subseteq \spa(\p{1}, \ldots, \p{M})$.
\end{defn}

Similarly, $\bPi^+$-factorizations that satisfy \textbf{(A)} and \textbf{(B3)} are maximal and linear. The proof is identical and, accordingly, omitted.

\begin{prop}
Let $(\vec{\Pi}, \vec{P})$ be a $\bPi^+$-factorization of $\tilde{\vec{P}}$. If $(\vec{\Pi}, \vec{P})$ satisfies \textbf{(A)} and \textbf{(B3)}, then $(\vec{\Pi}, \vec{P})$ satisfies \textbf{(M)} and \textbf{(L)}.
\end{prop}

Linear $\bPi^+$-factorizations must satisfy \textbf{(B$'$)}; indeed, the proof is identical to the proof for linear factorizations. However, maximal $\bPi^+$-factorizations need not satisfy \textbf{(A$'$)}. Consider the following counterexample. Let $Q_1 \sim \unif(0,2)$ and $Q_2 \sim \unif(1,3)$. Let $P_1 = \frac{2}{3} Q_1 + \frac{1}{3} Q_2$, $P_2 = \frac{1}{3} Q_1 + \frac{2}{3} Q_2$, $\p{1} = P_1$, and $\p{2} = P_2$. Then, $\bPi^+ = \vec{I}_2$---the identity matrix. Then, any $(\vec{\Pi}^\prime, \vec{P}^\prime)$ that satisfies (\ref{model}) and is consistent with $\bPi^+$ must be such that $(P_1, P_2)^T = \vec{P}^\prime$. Therefore, \textbf{(M$'$)} is satisfied. But, clearly, \textbf{(A$'$)} is not satisfied.

In summary, we are unable to offer a necessary condition that is close to \textbf{(A)}. On the other hand, \textbf{(B3)} is necessary.
\begin{prop}
\label{partial_necessity}
Let $(\vec{\Pi}, \vec{P})$ be an $\bPi^+$-factorization of $\tilde{\vec{P}}$. If $(\tilde{\vec{P}}, \bPi^+)$ is identifiable, then the columns of $\bPi^+$ are distinct.
\end{prop}
\begin{proof}
First, suppose $(\tilde{\vec{P}}, \bPi^+)$ is identifiable. Then, we can write $\tilde{\vec{P}} = \vec{\Pi} \vec{P}$ where $\vec{\Pi}$ is consistent with $\bPi^+$. We claim that for all $i \neq j$, $P_i \neq P_j$. To the contrary, suppose that there exists $i \neq j$ such that $P_i = P_j$. Without loss of generality, suppose $i = 1, j = 2$. Then, we can write
\begin{align*}
\tilde{\vec{P}} = \begin{pmatrix}
2 \vec{\Pi}_{:,1} & \vec{\Pi}_{:,3} & \ldots & \vec{\Pi}_{:,L}
\end{pmatrix}
\begin{pmatrix}
P_1 \\
P_3 \\
\vdots \\
P_L
\end{pmatrix},
\end{align*}
which contradicts the uniqueness of $\vec{P}$ and $\vec{\Pi}$.

Next, we give a proof by contraposition. Suppose that there exists $i \neq j$ such that $\bPi^+_{:, i} = \bPi^+_{:,j}$. Without loss of generality, let $i = 1$ and $j= 2$. Suppose that $(\vec{\Pi}, \vec{P})$ is consistent with $\bPi^+$ and solves $\tilde{\vec{P}} = \vec{\Pi} \vec{P}$. Then, the pair $(\vec{\Pi}^\prime, \vec{P}^\prime)$ given by
\begin{align*}
\vec{\Pi}^\prime = & \begin{pmatrix}
\vec{\Pi}_{:,2}  & \vec{\Pi}_{:,1} & \vec{\Pi}_{:,3} & \ldots & \vec{\Pi}_{:,L}
\end{pmatrix} \\
\vec{P}^\prime & = \begin{pmatrix}
P_2 \\
P_1 \\
P_3 \\
\vdots \\
P_L \\
\end{pmatrix}
\end{align*}
solves $\tilde{\vec{P}} = \vec{\Pi}^\prime \vec{P}^\prime$ and is consistent with $\bPi^+$. If $P_1 = P_2$, then $(\tilde{\vec{P}}, \bPi^+)$ is not identifiable, so we may rule out this case. Therefore, $\vec{P}^\prime \neq \vec{P}$, yielding the result.
\end{proof}

\section{Identification}
\label{ident_app}

In this section, we establish our identification results, i.e., Theorems \ref{demix_identification} and \ref{partial_identification}. We begin by proving Proposition \ref{equiv_opt}. Second, we prove a set of useful lemmas in Section \ref{lemmas_for_identification}. Third, we present our results on demixing mixed membership models in Section \ref{demixing_mixed_membership_models_section}. Finally, we present our results on classification with partial labels in Section \ref{classification_with_partial_labels_section}.

\subsection{Proof of Proposition \ref{equiv_opt}}

\begin{proof}
We prove the claims in order.
\begin{enumerate}
\item Without loss of generality, suppose $i =1$ and let $A = [L] \setminus \set{1}$. There is at least one point attaining the maximum in the optimization problem $\kappa^*(\vec{\eta}_1 \, | \, \set{\vec{\eta}_j : j \neq 1})$ by Lemma \ref{A_1}. Take a $G$ that achieves the maximum in $\kappa^*(Q_1 \, | \, \set{Q_j : j \neq 1})$, which exists also by Lemma \ref{A_1}. Then, we can write:
\begin{align}
Q_1 & = (1 - \sum_{j \geq 2} \mu_j) G + \sum_{j \geq 2} \mu_j Q_j. \label{prop_3_eqtn_1}
\end{align}
\sloppy Note that since $\vec{\eta}_1, \ldots, \vec{\eta}_L$ are linearly independent and $P_1, \ldots, P_L$ are jointly irreducible, $Q_1, \ldots, Q_L$ are linearly independent by Lemma \ref{B_1}. Therefore, $\kappa^*(Q_1 \, | \, \set{Q_j : j \neq 1}) = \sum_{j \geq 2} \mu_j < 1$ because, if not, $Q_1 = \sum_{j \geq 2} \mu_j Q_j$. 

Further, any $G$ that satisfies  \eqref{prop_3_eqtn_1} has the form $\sum_{i=1}^L \gamma_i P_i$ because \eqref{prop_3_eqtn_1} implies that $G \in \spa(Q_1, \ldots, Q_L)$ and each $Q_i \in \co(P_1, \ldots, P_L)$ by hypothesis. The $\gamma_i$ must sum to one, and we have that they are nonnegative by joint irreducibility. That is, $\vec{\gamma} \equiv \begin{pmatrix}
\gamma_1, \ldots, \gamma_L
\end{pmatrix}^T$
is a discrete distribution. Then, the above equation is equivalent to 
\begin{align}
\label{one_to_one_1}
\vec{\eta}^T_1 \vec{P} & = (1- \sum_{j \geq 2} \mu_j) \vec{\gamma}^T \vec{P} + \sum_{j \geq 2} \mu_j \vec{\eta}_j^T \vec{P}.
\end{align}
Since $P_1, \ldots, P_L$ are jointly irreducible, $P_1, \ldots, P_L$ are linearly independent by Lemma \ref{B_1}. By linear independence of $P_1, \ldots, P_L$, we obtain 
\begin{align}
\label{one_to_one_2}
\vec{\eta}_1 & = (1- \sum_{j \geq 2} \mu_j) \vec{\gamma} + \sum_{j \geq 2} \mu_j \vec{\eta}_j.
\end{align}
\noindent Consequently, $\kappa^*(Q_1 \, | \, \set{Q_j : j \neq 1}) = \kappa^*(\vec{\eta}_1 \, | \,\set{\vec{\eta}_j : j \neq 1}) < 1$. This completes the proof of statement 1.

\item This result follows immediately from Lemma \ref{A_1}.

\item By equations (\ref{one_to_one_1}) and (\ref{one_to_one_2}), there is a one-to-one correspondence between the maximizer $G$ to $\kappa^*(Q_1 \, | \, \set{Q_j : j \neq 1})$ and the maximizer $\vec{\gamma}$ to $\kappa^*(\vec{\eta}_1 \, | \, \set{\vec{\eta}_j : j \neq 1})$. The one-to-one correspondence is given by $G = \vec{\gamma}^T \vec{P}$. 

\end{enumerate}
\end{proof}

\subsection{Lemmas for Identification}
\label{lemmas_for_identification}

We present some technical results that are used repeatedly for our identification results. Lemma \ref{facts} gives us some useful properties of the two-sample $\kappa^*$ that we exploit in the PartialLabel and Demix algorithms. Statement \emph{1} gives an alternative form of $\kappa^*$. Statement \emph{2} gives the intuitive result that the residues lie on the boundary of the simplex. Statement \emph{3} gives a useful relation for determining whether two mixture proportions are on the same face; we use this relation extensively in our algorithms.
\begin{lemma}
\label{facts}
Let $F_1, \ldots, F_K$ be jointly irreducible distributions with $\vec{F} = (F_1, \ldots, F_K)^T$, $Q_1, Q_2$ be two distributions such that $Q_i = \vec{\eta}_i^T \vec{F}$ where $\vec{\eta}_i \in \si_K$ for $i =1,2$ and $\vec{\eta}_1 \neq \vec{\eta}_2$. Let $R$ be the residue of $Q_1$ wrt $Q_2$ and $R = \vec{\mu}^T \vec{F}$.
\begin{enumerate}
\item There is a one-to-one correspondence between the optimization problem in $\kappa^*(Q_1 \, | \, Q_2)$ and the optimization problem
\begin{align*}
\max(\alpha \geq 1 |  \exists \text{ a distribution } G \text{ s.t } G = Q_2 + \alpha(Q_1 - Q_2))
\end{align*}
via $\alpha = (1- \kappa)^{-1}$.
\item $ \vec{\mu} \in \partial \si_K$. 

\item $\sN(\vec{\eta}_2) \not \subseteq \sN(\vec{\eta}_1)$ if and only if $R=Q_1$ if and only if $\kappa^*(Q_1 \, | \, Q_2) = 0$.
\end{enumerate}
\end{lemma}
\begin{proof}
We note that we may assume that $R = \bmu^T \bF$ since by definition of the residue, $R \in \spa(Q_1, Q_2)$ and $Q_1, Q_2 \in \co(F_1, \ldots, F_K)$. 
\begin{enumerate}
\item Consider the linear relation: $Q_1 = (1-\kappa) G + \kappa Q_2$ where $\kappa \in [0,1]$. Since $F_1, \ldots, F_K$ are jointly irreducible and $\vec{\eta}_1$ and $\vec{\eta}_2$ are linearly independent, $Q_1$ and $Q_2$ are linearly independent by Lemma \ref{B_1}. Therefore, $\kappa < 1$. We can rewrite the relation as
\begin{align*}
G = \frac{1}{1-\kappa}Q_1 - \frac{\kappa}{1-\kappa} Q_2 = \alpha Q_1 + (1-\alpha)Q_2
\end{align*}
\noindent where $\alpha = \frac{1}{1-\kappa}$. The equivalence follows.

\item Since $R$ is the residue of $Q_1$ wrt $Q_2$, by Proposition \ref{equiv_opt}, $\vec{\mu}$ is the residue of $\vec{\eta}_1$ wrt $\vec{\eta}_2$ and $\vec{\mu} \in \si_K$. Therefore, by statement \emph{1} in Lemma \ref{facts}, $\vec{\mu}$ is such that $\alpha^*$ is maximized subject to the following constraints:
\begin{align*}
\vec{\mu} & = (1- \alpha^*)\vec{\eta}_2 + \alpha^*\vec{\eta}_1 \\
\alpha^* & \geq 1 \\
\vec{\mu} & \in \si_K.
\end{align*}
\noindent Suppose that $\min_{i} \mu_i > 0$. Then,
\begin{align*}
\bmu = (1- \alpha^*)\vec{\eta}_2 + \alpha^*\vec{\eta}_1 = \vec{\eta}_2 + \alpha (\vec{\eta}_1 - \vec{\eta}_2) > 0
\end{align*}
so that there is some $\epsilon > 0$ such that 
\begin{align*}
\vec{\mu}^\prime & = (1- \alpha^* - \epsilon)\vec{\eta}_2 + (\alpha^* + \epsilon)\vec{\eta}_1 \\
\alpha^* + \epsilon & \geq 1 \\
\vec{\mu}^\prime & \in \si_K.
\end{align*}
\noindent But, this contradicts the definition of $\alpha^*$ and $\vec{\mu}$. Therefore, $\min_{i} \mu_i = 0$. Consequently, $\vec{\mu} \in \partial \si_K$. 

\item By definition of $\kappa^*$, it is clear that $R=Q_1$ if and only if $\kappa^*(Q_1 \, | \, Q_2) = 0$. Therefore, it suffices to show that $\sN(\vec{\eta}_2) \not \subseteq \sN(\vec{\eta}_1)$ if and only if $\kappa^*(Q_1 \, | \, Q_2) = 0$. Suppose $\sN(\vec{\eta}_2) \not \subseteq \sN(\vec{\eta}_1)$. Then, there must be $i \in [K]$ such that $\eta_{2,i} > 0$ and $\eta_{1,i} = 0$. For any $\alpha > 1$, 
\begin{align*}
\min_{i \in [K]} (1-\alpha) \eta_{2,i} + \alpha \eta_{1,i} < 0;
\end{align*}
\noindent but, this violates the constraint of the optimization problem. Therefore, $\alpha = 1$. By statement \emph{1} in Lemma \ref{facts}, $\kappa^*(\vec{\eta}_1 \, | \, \vec{\eta}_2) = 0$. By statement \emph{1} of Proposition \ref{equiv_opt}, $\kappa^*(Q_1 \, | \, Q_2) = \kappa^*(\vec{\eta}_1 \, | \, \vec{\eta}_2) = 0$.

Now, suppose $\sN(\vec{\eta}_2) \subseteq \sN(\vec{\eta}_1)$. Then, for any $i \in [K]$, if $\eta_{2,i} > 0$, then $\eta_{1,i} > 0$. Then, there is $\alpha > 1$ sufficiently close to $1$ such that
\begin{align*}
\min_{i \in [K]} \eta_{2,i} + \alpha (\eta_{1,i} - \eta_{2,i}) \geq 0.
\end{align*}
\noindent By statement \emph{1} in this Lemma, $\kappa^*(\vec{\eta}_1 \, | \, \vec{\eta}_2) > 0$. By statement \emph{1} of proposition \ref{equiv_opt}, $\kappa^*(Q_1 \, | \, Q_2) = \kappa^*( \vec{\eta}_1 \, | \, \vec{\eta}_2) > 0$.
\end{enumerate}
\end{proof}

\begin{lemma}
\label{random_lin_indep}
Let $0 \leq k < L$. If $\vec{v}_1, \ldots, \vec{v}_k \in \si_L$ are linearly independent and $\vec{w}_{k+1}, \ldots, \vec{w}_L \in \si_L$ are random vectors drawn independently from the uniform distribution on a set $A \subset \si_L$ with positive $(L-1)$-dimensional Lebesgue measure, then $\vec{v}_1, \ldots, \vec{v}_k, \vec{w}_{k+1}, \ldots, \vec{w}_L$ are linearly independent with probability $1$.
\end{lemma}
\begin{proof}
We prove the result inductively. To begin, we prove the base case, i.e.  $\vec{v}_1, \ldots, \vec{v}_k, \vec{w}_{k+1}$ are linearly independent w.p. 1. It suffices to show that $\vec{w}_{k+1} \not \in \spa(\vec{v}_1, \ldots, \vec{v}_k)$ w.p. 1. Thus, it is enough to show that $\spa(\vec{v}_1, \ldots, \vec{v}_k) \cap A$ has $(L-1)$-dimensional Lebesgue measure $0$. Since 
\begin{align*}
\spa(\vec{v}_1, \ldots, \vec{v}_k) \cap A \subset \spa(\vec{v}_1, \ldots, \vec{v}_k) \cap \si_L,
\end{align*}
it suffices to show that $\spa(\vec{v}_1, \ldots, \vec{v}_k) \cap \si_L$ has $(L-1)$-dimensional Lebesgue measure $0$. 

Next, we claim that $\spa(\vec{v}_1, \ldots, \vec{v}_k) \cap \si_L \subseteq \aff(\bv_1, \ldots, \bv_k)$. Let $\sum_{i=1}^k \alpha_i \bv_i \in \si_L$. Since $\bv_i \in \si_L$ for all $i \in [k]$, we can write $\bv_i = \sum_{j=1}^L \beta_{i,j} \be_i$ where $\beta_{i,j} \geq 0$ and $\sum_{j=1}^L \beta_{i,j} = 1$. Then, since $\sum_{i=1}^k \alpha_i \bv_i = \sum_{i=1}^k \alpha_i \sum_{j=1}^L \beta_{i,j} \be_j \in \si_L$, it holds that
\begin{align*}
1 & = \sum_{i=1}^k \alpha_i \sum_{j=1}^L \beta_{i,j} = \sum_{i=1}^k \alpha_i.
\end{align*}
Thus, $\sum_{i=1}^k \alpha_i \bv_i \in \aff(\bv_1, \ldots, \bv_k)$, establishing the claim.

Thus, it suffices to show that $\aff(\bv_1, \ldots, \bv_k)$ has $(L-1)$-dimensional Lebesgue measure $0$. $\aff(\bv_1, \ldots, \bv_k)$ has affine dimension at most $k-1$. Since it is not possible to fit a $(L-1)$ dimensional ball in an affine subspace of affine dimension $k-1 < L-1$, $\aff(\bv_1, \ldots, \bv_k)$ has $(L-1)$-dimensional Lebesgue measure $0$. Thus, with probability $1$, $\vec{w}_{k+1} \not \in \spa(\vec{v}_1,\ldots, \vec{v}_k)$. This establishes the base case.

The inductive step follows by a union bound and a similar argument to the base case. Thus, the result follows.
\end{proof}

\subsection{Demixing Mixed Membership Models}
\label{demixing_mixed_membership_models_section}

In this section, we prove our identification result for demixing mixed membership models, i.e., Theorem \ref{demix_identification}. First, we present technical lemmas in Section \ref{lemmas_demixing_section}. Second,  in Section \ref{facetest_algorithm_section}, we present the key subroutine FaceTest and prove that it behaves as desired. Third, we prove Theorem \ref{demix_identification} in Section \ref{demix_algorithm_section}. Finally, in Section \ref{non_square_demix_algorithm_section}, we extend our results to the nonsquare case (where $M > L$). 

\subsubsection{Lemmas}
\label{lemmas_demixing_section}

Lemma \ref{continuity} establishes an intuitive continuity property of the two-sample version of $\kappa^*$ and the residue. Recall that $\norm{\cdot}$ denotes an arbitrary finite-dimensional norm on $\bbR^L$.
\begin{lemma}
\label{continuity}
Let $\vec{\eta}_1, \vec{\eta}_2 \in \si_L$ be distinct vectors and let $\vec{\mu}$ be the residue of $\vec{\eta}_2$ wrt $\vec{\eta}_1$. Let $\bgamma_n \in \si_L$ be a sequence such that $\norm{\bgamma_n - \bmeta_2} \longrightarrow 0$ as $n \longrightarrow \infty$, and let $\vec{\tau}_n$ be the residue of $\bgamma_n$ wrt $\vec{\eta}_1$. Then,
\begin{enumerate}
\item $\lim_{n \longrightarrow \infty} \kappa^*( \bgamma_n \, | \, \vec{\eta}_1) = \kappa^*(\vec{\eta}_2 \, | \, \vec{\eta}_1)$, and

\item $\lim_{n \longrightarrow \infty} \norm{\vec{\tau}_n - \vec{\mu}} = 0$.

\item If, in addition, $\brho_n \in \si_L$ is a sequence such that $\norm{\brho_n - \bmeta_2} \longrightarrow 0$ as $n \longrightarrow \infty$ and $\sN(\bmeta_2) = \sN(\bgamma_n) = \sN(\brho_n)$ for all $n$. Then, $\lim_{n \longrightarrow \infty} \kappa^*( \bgamma_n \, | \, \brho_n) =1$.
\end{enumerate}
\end{lemma}
\begin{proof}

\begin{enumerate}
\item \sloppy In order to apply the residue operator $\kappa^*$ to $\vec{\eta}_1, \vec{\eta}_2, \bgamma_n$ we think of $\vec{\eta}_1,\vec{\eta}_2, \bgamma_n$ as discrete probability distributions. By Proposition \ref{label_noise_bin_case}, 
\begin{align*}
\kappa^*(\vec{\eta}_2 \, | \, \vec{\eta}_1) & = \min_{i, \, \eta_{1,i} >0} \frac{\eta_{2,i}}{\eta_{1,i}}.
\end{align*}
Clearly, there is a constant $\delta > 0$ such that $\min_{i, \, \vec{\eta}_{1,i} >0} \vec{\eta}_{1,i} > \delta$. Let $\epsilon >0$. By the equivalence of norms on finite-dimensional vector spaces, there exists a constant $C > 0$ such that $\norm{\cdot}_\infty \leq C \norm{\cdot}$ where $\norm{\cdot}_\infty$ denotes the supremum norm. Thus, since $\norm{\bgamma_n - \bmeta_2} \longrightarrow 0$ as $n \longrightarrow \infty$, we can let $n$ large enough such that $|\bgamma_{n,i} - \bmeta_{2,i}| \leq \epsilon$ for all $i \in [L]$. Then,
\begin{align*}
\kappa^*(\bgamma_n \, | \, \bmeta_1) & = \min_{i, \eta_{1,i} > 0} \frac{\gamma_{n,i}}{\eta_{1,i}} \leq \min_{i, \eta_{1,i} > 0} \frac{\eta_{2,i} + \epsilon}{\eta_{1,i}} \\
& \leq \kappa^*(\bmeta_2 \, | \, \bmeta_1) + \frac{\epsilon}{\delta}.
\end{align*}

Similarly, 
\begin{align*}
\kappa^*(\bgamma_n \, | \, \bmeta_1) & = \min_{i, \eta_{1,i} > 0} \frac{\gamma_{n,i}}{\eta_{1,i}} \\
  & \geq  \frac{\eta_{2,i} - \epsilon}{\eta_{1,i}} \\
& \geq \kappa^*(\bmeta_2 \, | \, \bmeta_1) - \frac{\epsilon}{\delta}.
\end{align*}

Since $\epsilon > 0$ was arbitrary, statement \emph{1} follows.

%

\item Write $\vec{\mu} = \kappa \vec{\eta}_2 + (1-\kappa) \vec{\eta}_1$ and $\vec{\tau}_n = \kappa_n \bgamma_n + (1-\kappa_n) \vec{\eta}_1$ where $\kappa = \kappa^*(\vec{\eta}_2 \, | \, \vec{\eta}_1)$ and $\kappa_n = \kappa^*(\bgamma_n \, | \, \vec{\eta}_1)$. Then, by the triangle inequality,
\begin{align*}
\norm{\bmu - \btau_n} & \leq \norm{\kappa \bmeta_2 -\kappa_n \bgamma_n} + \norm{(1-\kappa) \bmeta_1 - (1-\kappa_n) \bmeta_1} \\
& \leq |\kappa-\kappa_n| \norm{\bmeta_2} + |\kappa_n| \norm{\bmeta_2 - \bgamma_n} + |\kappa - \kappa_n|\norm{\bmeta_1} \longrightarrow 0
\end{align*}
as $n \longrightarrow \infty$.


\item W.l.o.g., suppose that $[K] = \sN(\bmeta_2) = \sN(\bgamma_n) = \sN(\brho_n)$ where $K \leq L$. Observe that
\begin{align*}
\kappa^*(\bgamma_n \, | \, \brho_n) & = \min_{i, \, \rho_{n,i} >0} \frac{\gamma_{n,i}}{\rho_{n,i}} = \min_{i \in [K]} \frac{\gamma_{n,i}}{\rho_{n,i}}
\end{align*}
There exists a constant $\delta > 0$ such that $\min_{i \in [K]} \eta_{2,i} \geq \delta$. Let $\delta > \epsilon > 0$. By the equivalence of norms on finite-dimensional vector spaces, we can let $n$ large enough such that $|\gamma_{n,i} - \eta_{2,i}| \leq \epsilon$ and $|\rho_{n,i} - \eta_{2,i}| \leq \epsilon$ for all $i \in [L]$. Then,
\begin{align*}
\kappa^*(\bgamma_n \, | \, \brho_n) & = \min_{i \in [K]} \frac{\gamma_{n,i}}{\rho_{n,i}} \\
& \leq  \min_{i \in [K]} \frac{\eta_{2,i} + \epsilon }{\eta_{2,i} - \epsilon} \\
& \leq \frac{\eta_{2,i} + \epsilon }{\eta_{2,i} - \epsilon}
\end{align*}
for any $i \in [K]$, which goes to $1$ as $\epsilon \longrightarrow 0$. Similarly,
\begin{align*}
\kappa^*(\bgamma_n \, | \, \brho_n) & = \min_{i \in [K]} \frac{\gamma_{n,i}}{\rho_{n,i}} \\
& \geq  \min_{i \in [K]} \frac{\eta_{2,i} - \epsilon }{\eta_{2,i} + \epsilon} \\
\end{align*}
Since for any $i \in [K]$,
\begin{align*}
\frac{\eta_{2,i} - \epsilon }{\eta_{2,i} + \epsilon} \longrightarrow 1
\end{align*}
as $\epsilon \longrightarrow 0$, the above lower bound goes to $1$ as $\epsilon \longrightarrow 0$. Thus, statement \emph{3} follows.

\end{enumerate}

\end{proof}
Lemma \ref{pres_lin_ind} guarantees that certain operations in the Demix algorithm preserve linear independence of the mixture proportions. The proof uses tools from multilinear algebra.
\begin{lemma}
\label{pres_lin_ind}
Let $\vec{\tau}_1, \ldots, \vec{\tau}_K \in \si_K$ be linearly independent and $P_1,\ldots, P_K$ be jointly irreducible. Let $Q_i = \vec{\tau}_i^T \vec{P}$ for $i \in [K]$. Then for any $i,j \in [K]$ such that $i \neq j$, 
\begin{enumerate}

\item If $\vec{\eta} = \sum_{k=1}^K a_k \vec{\tau}_k$ with $a_j \neq 0$, then $\vec{\tau}_1,\ldots, \vec{\tau}_{j-1}, \vec{\eta}, \vec{\tau}_{j+1}, \ldots, \vec{\tau}_K$ are linearly independent.

\item Let $R_k$ be the residue of $Q_k$ with respect to $Q_j$ for all $k \in [K] \setminus \set{j}$. Then, $R_k = \vec{\eta}^T_k \vec{P}$ where $\vec{\eta}_k \in \si_L$ and $\vec{\eta}_1,\ldots,\vec{\eta}_{j-1},\vec{\tau}_j,\vec{\eta}_{j+1},\ldots,\vec{\eta}_K$ are linearly independent.

\item Let $\vec{\tau}^* \in \co (\vec{\tau}_1,\ldots, \vec{\tau}_k)^\circ$ and $\vec{\eta}_i \in \co (\vec{\tau}_i, \vec{\tau}^*)^\circ$ for $i \in [k]$ where $k \leq K$. Then, 
\begin{align*}
\vec{\eta}_1, \vec{\eta}_2, \ldots, \vec{\eta}_k, \vec{\tau}_{k+1},\ldots, \vec{\tau}_K
\end{align*}
are linearly independent. 

\end{enumerate}

\end{lemma}
\begin{proof}
We use the multilinear expansion and usual properties of determinants.

\begin{enumerate}
\item Viewing each $\bmtau_i$ as a column vector,
\begin{align*}
\det(\bmtau_1,\ldots,\bmtau_{j-1},\sum_k a_k \bmtau_k,\bmtau_{j+1},\ldots,\bmtau_K)
= a_j \det(\bmtau_1,\ldots,\bmtau_K) \neq 0.
\end{align*}

\item Linear independence of $\bmtau_1, \ldots, \bmtau_K$ implies that the $Q_1, \ldots, Q_K$ are distinct. Hence, by Proposition \ref{label_noise_bin_case}, we can write $\bmeta_k = (1-\alpha_k) \bmtau_j + \alpha_k \bmtau_k$ where $\alpha_k \neq 0 \, \forall k \neq j$. Then, it holds that
\begin{align*}
\det(\bmeta_1,\ldots,\bmeta_{j-1},\bmtau_j,\bmeta_{j+1},\ldots,\bmeta_K) = \paren{\prod_{i \neq j} \alpha_i}
\det(\bmtau_1,\ldots,\bmtau_K) \neq 0.
\end{align*}

\item Since $\bmeta_j = (1-\alpha_j) \bmtau^* + \alpha_j \bmtau_j$ where $\alpha_j \in (0,1)$ for all $j\leq k$, 
and $\bmtau^*=\sum_i \beta_i \bmtau_i$, it holds
\begin{align*}
\det(\bmeta_1,\ldots,\bmeta_{k-1},\bmtau_k,\ldots,\bmtau_K) = \paren{1+\sum_{j=1}^k \frac{(1-\alpha_j)}{\alpha_j}\beta_j}
\paren{\prod_{i=1}^k \alpha_i}
\det(\bmtau_1,\ldots,\bmtau_K) \neq 0.
\end{align*}

\end{enumerate}
\end{proof}
Lemma \ref{single} gives a condition on the mixture proportions under which the multi-sample residue is unique. Lemma 2 in \citet{blanchard2014} is very similar and is proved in a very similar way. We give a useful generalization here that reproduces many of the same details. 
\begin{lemma}
\label{single}
Let $l, k \in [L]$. Let $\vec{\tau}_1, \ldots, \vec{\tau}_L \in \si_L$ be linearly independent. We have that condition 1 implies condition 2 and condition 2 implies condition 3. 
\begin{enumerate}

\item There exists a decomposition
\begin{align*}
\vec{\tau}_l & = \kappa\vec{e}_k + (1-\kappa) \vec{\tau}^\prime_l
\end{align*}
where $\kappa > 0$ and $\vec{\tau}^\prime_l \in \co(\set{\vec{\tau}_j : j \neq l})$. Further, for every $\vec{e}_i$ such that $i \neq k$, there exists a decomposition
\begin{align*}
\vec{e}_i & = \sum_{j=1}^L a_j \vec{\tau}_j
\end{align*}
such that $a_l < \frac{1}{\kappa}$. 

\item Let
\begin{align*}
\vec{T} & = \begin{pmatrix}
\vec{\tau}^T_1 \\
\vdots \\
\vec{\tau}^T_L
\end{pmatrix};
\end{align*}
the matrix $\vec{T}$ is invertible and $\vec{T}^{-1}$ is such that $(\vec{T}^{-1})_{l,k} > 0$ and $(\vec{T}^{-1})_{l,i} \leq 0$ for $i \neq k$ and $(\vec{T}^{-1})_{l,k} > (\vec{T}^{-1})_{j,k}$ for $j \neq l$. In words, the $(l,k)$th entry in $\vec{T}^{-1}$ is positive, every other entry in the $l$th row of $\vec{T}^{-1}$ is nonpositive and every other entry in the $k$th column of $\vec{T}^{-1}$ is strictly less than the $(l,k)$th entry. \footnote{$(\vec{T}^{-1})_{i,j}$ is the $i \times j$ entry in the matrix $\vec{T}^{-1}$.}

\item The residue of $\vec{\tau}_l$ with respect to $\set{\vec{\tau}_j, j \neq l}$ is $\vec{e}_k$. 
\end{enumerate}

\end{lemma}
\begin{proof}
Without loss of generality, let $l = 1$ and $k=2$. By relabeling the vectors $\be_1, \ldots, \be_L$, we can assume without loss of generality that $k = 1$. First, we show that condition \emph{1} implies condition \emph{2}. Suppose that condition \emph{1} holds. Then, there exists $\kappa > 0$ such that 
\begin{align*}
\vec{\tau}_1 & = \kappa \vec{e}_1 + (1-\kappa) \sum_{i=2}^L \mu_i \vec{\tau}_i
\end{align*}
with $\mu_i \geq 0$ for $i \in [L] \setminus \set{1}$. Then,
\begin{align*}
\vec{e}_1 =  \frac{1}{\kappa}(\vec{\tau}_1 - \sum_{i \geq 2} (1-\kappa) \mu_i \vec{\tau}_i).
\end{align*}
Hence, the first row of $\vec{T}^{-1}$ is given by $\frac{1}{\kappa}(1, -(1-\kappa)\mu_2, \cdots, -(1-\kappa)\mu_L)$. This shows that the first row is such that $(\vec{T}^{-1})_{1,1} > 0$ and $(\vec{T}^{-1})_{1,i} \leq 0$ for $i \neq 1$.

Consider $\vec{e}_i$ such that $i \neq 1$. Then, we have the relation: $\vec{e}_i = \sum_{j=1}^L a_j \vec{\tau}_j$, which gives the $i$th row of $\vec{T}^{-1}$. By assumption, $a_1 < \frac{1}{\kappa}$, so the $(i,1)$th entry is strictly less than the $(1,1)$th entry. Hence, \emph{2} follows. 

Now, we prove that condition \emph{2} implies condition \emph{3}. Suppose condition \emph{2} is true. Consider the optimization problem
\begin{align*}
\max_{\vec{\nu}, \vec{\gamma}} \sum_{i=2}^L \nu_i & \textit{ s.t. } \vec{\tau}_1 = (1 - \sum_{i \geq 2} \nu_i) \vec{\gamma} + \sum_{i=2}^L \nu_i \vec{\tau}_i 
\end{align*}
over $\vec{\gamma} \in \Delta_L$ and $\vec{\nu} = (\nu_2, \cdots, \nu_L) \in C_{L-1} = \set{(\nu_2, \cdots, \nu_L) : \nu_i \geq 0 ; \sum_{i=2}^L \nu \leq 1}$. 

By the same argument given in the proof of Lemma 2 of \citet{blanchard2014}, this optimization problem is equivalent to the program
\begin{align*}
\max_{\vec{\gamma} \in \si_L} \vec{e}_1^T(\vec{T}^T)^{-1} \vec{\gamma} \textit{ s.t. } \vec{\nu}((\vec{T}^T)^{-1} \vec{\gamma}) \in C_{L-1}
\end{align*}
where $\vec{\nu}(\vec{\eta}) \coloneq \eta_1^{-1}(-\eta_2, \cdots, -\eta_L)$. The above objective is of the form $\vec{a}^T \vec{\gamma}$ where $\vec{a}$ is the first column of $\vec{T}^{-1}$. Since $l =1$, by assumption, for every $i \neq 1$, $\vec{T}^{-1}_{1,1} > \vec{T}^{-1}_{i,1}$. Therefore, the unconstrained maximum over $\vec{\gamma} \in \si_L$ is attained uniquely by $\vec{\gamma} = \vec{e}_1$. Notice that $(\vec{T}^T)^{-1}\vec{e}_1$ is the first row of $\vec{T}^{-1}$. Denote this vector $\vec{b} = (b_1, \cdots, b_L)$. We show that $\vec{\nu}(\vec{b}) = b_1^{-1}(-b_2, \cdots, -b_L) \in C_{L-1}$. By assumption, $\vec{b}$ has its first coordinate positive and the other coordinates are nonpositive. Therefore, all of the components of $\vec{\nu}(\vec{b})$ are nonnegative. Furthermore, the sum of the components of $\vec{\nu}(\vec{b})$ is
\begin{align*}
\sum_{i=2}^L \frac{-b_i}{b_1} & = 1 - \frac{\sum_{i=1}^L b_i}{b_1} = 1 - \frac{1}{b_1}  \leq 1.
\end{align*}
The last equality follows because the rows of $\vec{T}^{-1}$ sum to $1$ since $\vec{T}$ is a stochastic matrix. Then, we have $\vec{\nu}((\vec{T}^T)^{-1}\vec{e}_1) \in C_{L-1}$. Consequently, the unique maximum of the optimization problem is attained for $\vec{\gamma} = \vec{e}_1$. This establishes \emph{3}. 
\end{proof}
\subsubsection{The FaceTest Algorithm}
\label{facetest_algorithm_section}

Next, we consider the main subroutine in the Demix algorithm: the FaceTest algorithm (see Algorithm \ref{face_test_alg}). Proposition \ref{face_test} establishes that FaceTest($Q_1, \ldots, Q_K$) returns $1$ if and only if $Q_1, \ldots, Q_K$ are in the relative interior of the same face of the simplex.


\begin{prop}
\label{face_test}
Let $Q_j = \vec{\eta}^T_j \vec{P}$ for $\vec{\eta}_j \in \si_K$ and all $j \in [K]$. Let $P_1,\ldots, P_K$ be jointly irreducible, $Q_1, \ldots, Q_K \in \co(P_1, \ldots, P_K)$ be distinct, and  for each $i \in [K]$, let $\vec{\eta}_i$ lie in the relative interior of one of the faces of $\si_K$. FaceTest($Q_1, \ldots, Q_K$) returns $1$ if and only if $\vec{\eta}_1,\ldots, \vec{\eta}_K$ lie in the relative interior of the same face of $\si_K$. 
\end{prop}
\begin{proof}
Suppose that $\vec{\eta}_1,\ldots, \vec{\eta}_K$ lie on the relative interior of the same face of $\si_K$. Then, $\sN(Q_1) = \ldots = \sN(Q_K)$. By statement 3 of Lemma \ref{facts}, $\kappa^*(Q_i \, | Q_j)>0$ for all $i \neq j$. Hence, FaceTest($ Q_1, \ldots, Q_K$) returns $1$.

Suppose that $Q_1,\ldots,Q_K$ do not all lie on the relative interior of the same face. Then, there exists $Q_i, Q_j$ ($i \neq j$) that do not lie on the relative interior of the same face. Without loss of generality, suppose that $\sN(Q_j) \not \subseteq \sN(Q_i)$. Then, by statement 3 of Lemma \ref{facts}, $\kappa^*(Q_i \, | \, Q_j) = 0$. Hence, FaceTest($Q_1, \ldots, Q_K$) returns $0$.
\end{proof}
\subsubsection{The Demix Algorithm}
\label{demix_algorithm_section}

\begin{proof}[Proof of Theorem \ref{demix_identification}]

Let $K \leq L$, $\bgamma_i \in \si_L$ for all $i \in [K]$, $S_i = \bgamma_i^T \vec{P}$ for all $i \in [K]$, and 
\begin{align*}
\vec{\Gamma} = \begin{pmatrix}
\bgamma_1^T \\
\vdots \\
\bgamma_K^T
\end{pmatrix}.
\end{align*}
We claim that for any $\set{i_1,\ldots,i_K} \subset [L]$ and $\set{S_1, \ldots, S_K} \subset \co(P_{i_1},\ldots, P_{i_K})$, if $P_1, \ldots, P_L$ are jointly irreducible, and $\vec{\Gamma}$ has full row rank, then w.p. $1$ Demix($S_1, \ldots,S_K$) returns a permutation of $(P_{i_1}, \ldots, P_{i_K})$. If the claim holds, then setting $K = L$ and putting $\p{i} = S_i$ yields the result. We prove the claim by induction on $K$.

Consider the base case: $K = 2$. Suppose that $\set{S_1, S_2} \subset \co(P_1, P_2)$ (the other cases are similar). Note that $\bgamma_1 \neq \bgamma_2$ by linear independence of $\bgamma_1$ and $\bgamma_2$. Either $\bgamma_1 \in \co(\be_1, \bgamma_2)$ or $\bgamma_1 \in \co(\be_2, \bgamma_2)$. Suppose $\bgamma_1 \in \co(\be_1, \bgamma_2)$. Condition \emph{2} of Lemma \ref{2_blanchard2014} is satisfied so that $\be_1$ is the residue of $\bgamma_1$ with respect to $\bgamma_2$ and $\be_2$ is the residue of $\bgamma_2$ with respect to $\bgamma_1$. Thus, by statement \emph{3} of Proposition \ref{equiv_opt}, $P_1$ is the residue of $S_1$ with respect to $S_2$ and $P_2$ is the residue of $S_2$ with respect to $S_1$. If $\bgamma_1 \in \co(\be_2, \bgamma_2)$, then similar reasoning establishes that $P_2$ is the residue of $S_1$ with respect to $S_2$ and $P_1$ is the residue of $S_2$ with respect to $S_1$. Thus, the base case follows.

Suppose $L \geq K > 2$. The inductive hypothesis is:
\begin{description}
\item[Inductive Hypothesis:] for any $\set{i_1,\ldots,i_{K-1}} \subset [L]$ and $\set{S_1, \ldots, S_{K-1}} \subset \co(P_{i_1},\ldots, P_{i_{K-1}})$, if $P_1, \ldots, P_L$ are jointly irreducible and $\vec{\Gamma}$ has full row rank, then w.p. $1$ Demix($S_1, \ldots,S_{K-1}$) returns a permutation of $(P_{i_1}, \ldots, P_{i_{K-1}})$.
\end{description}
\sloppy Suppose that $\set{S_1, \ldots, S_K} \subset \co(P_1, \ldots, P_K)$ (the other cases are similar). Set $\Xi = \co(\be_1, \ldots, \be_K)$. With probability $1$, $Q \in \co (S_2,\ldots, S_K)^\circ$. We can write $Q = \vec{\eta}^T \vec{P}$ where $\vec{\eta}$ is a uniformly distributed random vector in $\co(\bgamma_2,\ldots,\bgamma_K)$. Let $R$ be the residue of $Q$ with respect to $S_1$. By statement \emph{3} of Proposition \ref{equiv_opt}, we can write $R = \vec{\lambda}^T \vec{P}$ where $\vec{\lambda}$ is the residue of $\vec{\eta}$ with respect to $\bgamma_1$. By statement \emph{2} of Lemma \ref{facts}, $\vec{\lambda} \in \partial \Xi$. 
\begin{description}
\item[Step 1:] We claim that with probability $1$, there is $l \in [K]$ such that $\vec{\lambda} \in \co (\set{\vec{e}_j : j \in [K] \setminus \set{l}})^\circ$. Let $B_{i,j} = \co (\set{\bgamma_1} \cup \set{\vec{e}_k :k \in [K] \setminus \set{ i,j}})$ where $i,j \in [K]$ and $i \neq j$ and let $C = \co(\bgamma_2,\ldots, \bgamma_K)$. First, we argue that $C \cap B_{i,j}$ has affine dimension at most $K-3$.\footnote{Note that if $\vec{v}_1,\ldots,\vec{v}_n \in \bbR^L$ are linearly independent and $n \leq L$, then $\aff(\vec{v}_1,\ldots,\vec{v}_n)$ has affine dimension $n-1$.} Since $\bgamma_2,\ldots,\bgamma_K$ are linearly independent, $C$ has affine dimension $K-2$. Since $\set{\vec{e}_k :k \in [K] \setminus \set{ i,j}}$ are linearly independent, $B_{i,j}$ has affine dimension $K-2$ or $K-3$. If $B_{i,j}$ has affine dimension $K-3$, then $C \cap B_{i,j}$ has affine dimension at most $K-3$. So, suppose that $B_{i,j}$ has affine dimension $K-2$. If $C \cap B_{i,j}$ has affine dimension $K-2$, then $\aff C = \aff B_{i,j}$. Then, in particular, $\bgamma_1 \in \aff C$. But, this contradicts the linear independence of $\bgamma_1,\ldots, \bgamma_K$. Therefore, $C \cap B_{i,j}$ has affine dimension at most $K-3$.

Because $C$ has affine dimension $K-2$ and $\vec{\eta}$ is a uniformly distributed random vector in $C$, with probability $1$, $\vec{\eta} \not \in \cup_{i,j \in [K], i \neq j} B_{i,j}$. Since $\bgamma_1 \in B_{i,j}$ for all $i,j \in [K]$ and $\vec{\eta} \in \co(\vec{\lambda}, \bgamma_1)$ by definition, the convexity of $B_{i,j}$ implies that $\vec{\lambda} \not \in \cup_{i,j \in [K], i\neq j} B_{i,j}$. Since $\vec{\lambda} \in \partial \Xi$, the claim follows.

\item[Step 2:] Let $R^{(n)}_i$ be the residue of $m_\frac{n-1}{n}(S_i, Q)$ with respect to $S_1$. We claim that there is some finite integer $N \geq 2$ such that for all $n \geq N$,
\begin{align*}
\text{FaceTest}(R^{(n)}_2, \ldots, R^{(n)}_K)
\end{align*}
\noindent returns $1$. By Proposition \ref{face_test}, this is equivalent to the statement that there exists $N \geq 2$ such that for all $n \geq N$, the mixture proportions of $R^{(n)}_2, \ldots, R^{(n)}_K$ are on the relative interior of the same face. Let $m_\frac{n-1}{n}(S_i, Q) = (\vec{\tau}_{i}^{(n)})^T \vec{P}$ for $i \in [K] \setminus \set{1}$; note that $\vec{\tau}^{(n)}_{i} = \frac{1}{n} \bgamma_i + \frac{n-1}{n}\vec{\eta}$ and, consequently, $\vec{\tau}_{i}^{(n)} \in \Xi$. Since $\vec{\eta} \in \co(\bgamma_2,\ldots,\bgamma_K)^\circ$ with probability $1$, $\vec{\tau}^{(n)}_{i} \in \co(\bgamma_i, \vec{\eta})^\circ$ for all $i \in [K] \setminus \set{1}$ and $n \in \mathbb{N}$, and $\bgamma_1,\ldots, \bgamma_K$ are linearly independent, it follows that for all $n \in \mathbb{N}$ with probability $1$, $\bgamma_1, \vec{\tau}^{(n)}_{2}, \ldots, \vec{\tau}^{(n)}_{K}$ are linearly independent by statement \emph{3} in Lemma \ref{pres_lin_ind}. Fix $i \in [K] \setminus \set{1}$. It suffices to show that there is large enough $N$ such that for $n \geq N$, $\text{Residue}(m_\frac{n-1}{n}(S_i, Q) \, | \, S_1) = R^{(n)}_i$ is on the same face as $R$. Let $R_i^{(n)} = (\vec{\mu}_i^{(n)})^T \vec{P}$; by statement \emph{3} of Proposition \ref{equiv_opt}, $\vec{\mu}_i^{(n)}$ is the residue of $\vec{\tau}_i^{(n)}$ with respect to $\bgamma_1$ and by statement 2 of Lemma \ref{facts} $\vec{\mu}_i^{(n)} \in \Xi$. It suffices to show that $\sN( \vec{\mu}_i^{(n)}) = \sN(\vec{\lambda})$, i.e., every $\vec{\mu}_i^{(n)}$ is on the same face as $\vec{\lambda}$. As $n \longrightarrow \infty$, $\vec{\tau}_{i}^{(n)} = (1 - \frac{n -1}{n}) \bgamma_i + \frac{n-1}{n} \vec{\eta} \longrightarrow \vec{\eta}$, hence by statement \emph{2} in Lemma \ref{continuity}, $\norm{\vec{\mu}_i^{(n)} - \vec{\lambda}} \longrightarrow 0$. Since with probability $1$, $\vec{\lambda} \in \co (\set{\vec{e}_j : j \in [K] \setminus \set{l} })^\circ$ for some $l$ (step 1), it follows that for large enough $n$, $\vec{\mu}_i^{(n)} \in \co (\set{\vec{e}_j : j \in [K] \setminus \set{l}})^\circ$. 

\item[Step 3:] \sloppy Assume that $n$ is sufficiently large such that $R_2^{(n)}, \ldots, R_K^{(n)}$ are on the same face. The algorithm recurses on $R_2^{(n)},\ldots, R_K^{(n)}$. Since $\bgamma_1, \vec{\tau}^{(n)}_{2}, \ldots, \vec{\tau}^{(n)}_{K}$ are linearly independent, it follows by statement \emph{2} in Lemma \ref{pres_lin_ind} that $\vec{\mu}_2^{(n)}, \ldots, \vec{\mu}_K^{(n)}$ are linearly independent. Suppose wlog that $\set{R_2^{(n)},\ldots, R_K^{(n)}} \subset \set{P_1, \ldots, P_{K-1}}$. Then, by the inductive hypothesis, if $(Q_1, \ldots, Q_{K-1}) \longleftarrow \text{Demix}(R^{(n)}_2,\ldots,R^{(n)}_K)$, then $(Q_1, \ldots, Q_{K-1})$ is a permutation of $(P_1, \ldots, P_{K-1})$. Note that $\frac{1}{K} \sum_{i=1}^K S_i \in \co(P_1, \ldots, P_K)^\circ$ since $\vec{\Gamma}$ has full rank by assumption. 

Write $Q_i = \brho_i^T \vec{P}$ for $i \in [K]$. Then, there exists of a permutation $\sigma: [K-1] \longrightarrow [K-1]$ such that $\brho_i = \be_{\sigma(i)}$. Since $\brho_K \in \Xi^\circ$ and $\brho_i = \be_{\sigma(i)}$ for $i \leq K-1$, the conditions in statement \emph{1} of Lemma \ref{single} are satisfied. Therefore, by Lemma \ref{single}, the residue of $\brho_K$ with respect to $\set{\brho_1, \ldots, \brho_{K-1}}$ is $\be_K$. Then, by statement \emph{3} of Proposition \ref{equiv_opt}, the residue of $Q_K$ with respect to $\set{Q_1, \ldots, Q_{K-1}}$ is $P_K$. This completes the inductive step.

\end{description}
\end{proof}
\subsubsection{The Non-Square Demix Algorithm}
\label{non_square_demix_algorithm_section}

Now, we examine the non-square case of the demixing problem ($M > L$). Note that knowledge of $L$ is needed since one must resample exactly $L$ distributions in order to run the square Demix algorithm.
\begin{algorithm}[t]
\caption{NonSquareDemix($\p{1}, \ldots, \p{M}$)}
\label{non_square_demix_alg}
\begin{algorithmic}[1]
\STATE $R_1, \ldots, R_L \longleftarrow \text{ independently uniformly distributed elements in} \co (\p{1}, \ldots, \p{M})$
\STATE $(Q_1,\ldots, Q_L)^T \longleftarrow$ Demix($R_1,\ldots,R_L$)\\
\RETURN $(Q_1,\ldots, Q_L)^T$
\end{algorithmic}
\end{algorithm}
\begin{cor}
\label{non_square}
Suppose $M > L$. Let $P_1, \ldots, P_L$ be jointly irreducible and $\vec{\Pi}$ have full rank. Then, with probability $1$, NonSquareDemix$(\p{1}, \ldots, \p{M})$ returns $(Q_1,\ldots, Q_L)$ such that $(Q_1,\ldots, Q_L)$ is a permutation of $(P_1, \ldots, P_L)$.
\end{cor}
\begin{proof}
We can write $R_i = \vec{\tau}_i^T \vec{P}$ where $\vec{\tau}_i \in \si_L$ and $i=1,\ldots, L$. $\vec{\tau}_1, \ldots, \vec{\tau}_L$ are drawn uniformly independently from a set with positive $(L-1)$-dimensional Lebesgue measure since $\vec{\Pi}$ has full rank by hypothesis. By Lemma \ref{random_lin_indep}, $\vec{\tau}_1,\ldots, \vec{\tau}_L$ are linearly independent with probability $1$. Then, by Theorem \ref{demix_identification}, with probability 1, Demix($R_1,\ldots,R_L$) returns a permutation of $(P_1,\ldots, P_L)$.
\end{proof}
\subsection{Classification with Partial Labels}
\label{classification_with_partial_labels_section}

In this section, we present our identification result for classification with partial labels, i.e., Theorem \ref{partial_identification}. To begin, in Section \ref{vertex_test_section}, we prove an important lemma for the main subroutine of the algorithm PartialLabel: VertexTest (algorithm \ref{vertex_test_alg}). Second, in Section \ref{Theorem_partial_label_section}, we present the proof of Theorem \ref{partial_identification}.

\subsubsection{VertexTest Algorithm}
\label{vertex_test_section}

Lemma \ref{vertex_test} establishes that the VertexTest algorithm determines whether one vector of distributions is a permutation of another vector of distributions. 
\begin{lemma}
\label{vertex_test}
Let $\vec{\eta}_1, \ldots, \vec{\eta}_L \in \si_L$ and $Q_i = \vec{\eta}_i^T \vec{P}$ for $i \in [L]$ and $\vec{Q} = (Q_1, \ldots, Q_L)^T$. Suppose that $P_1, \ldots, P_L$ are jointly irreducible, $\vec{\Pi}$ has full column rank, and the columns of $\bPi^+$ are unique. Then, $\text{VertexTest}(\bPi^+, \tilde{\vec{P}},\vec{Q})$ returns $(1, \vec{C}^T)$ with $\bC$ a permutation matrix  if and only if $\vec{Q}$ is a permutation of $\vec{P}$. Further, if   $\text{VertexTest}(\bPi^+, \tilde{\vec{P}}, \vec{Q})$ returns $(1, \vec{C}^T)$, then $\vec{C}^T \vec{Q}= \vec{P}$.
\end{lemma}
\begin{proof}
If $\vec{Q} = (Q_1, \ldots, Q_L)^T$ is such that  $\vec{D}^T \vec{Q} = \vec{P}$ where $\vec{D}$ is a permutation matrix, then it is clear that $\text{VertexTest}(\bPi^+, (\p{1}, \ldots, \p{M})^T, (Q_1, \ldots, Q_L)^T)$ returns $(1, \vec{C}^T)$ for some permutation matrix $\bC$ since the entries of $\bPi^+$ are $\bPi^+_{i,j} = \ind{\kappa^*(\p{i} \, | \, P_j) > 0}$. Since $\vec{D}^T \vec{Q} = \vec{P}$, clearly, $\vec{Z} \vec{D} = \bPi^+$. But since the columns of $\bPi^+$ are unique, there is a unique permutation of the columns of $\vec{Z}$ to obtain the columns of $\bPi^+$. Therefore, $\vec{D} = \vec{C}$.

Consider the ``only if'' direction. We use the notation from Algorithm \ref{vertex_test_alg}. Suppose Algorithm \ref{vertex_test_alg} has returned $(1,\vec{C}^T)$ where $\bC$ is a permutation matrix. W.l.o.g. (reordering the $Q_i$) we can assume that $\vec{C}$ is the identity and thus $\vec{Z}=\bPi^+$.
  
In the sequel denote $\phi(x):=\ind{x > 0}$ and $\phi(\vec{M})$ the entry-wise application of $\phi$ to the matrix or vector $\vec{M}$. We denote $\vec{v} \preceq \vec{w}$ when all entries of $\vec{v}$ are less than or equal to the corresponding entries of $\vec{w}$ (where $\bv$ and $\bw$ are vectors). This is a partial order, which will be used only for $0-1$ vectors below
(essentially to denote support inclusion). W.l.o.g. (reordering the $P_i$)
we can assume that the columns of $\bPi^+$ are reordered in some sequence compatible with
$\preceq$ in decreasing order, i.e. such that if $\bPi^+_{:,j} \preceq \bPi^+_{:,i}$, then $i \leq j$.

Introduce the following additional notation: let $\vec{\Lambda}$ be the matrix with rows
$\vec{\Lambda}_{i,:} = \phi(\vec{\eta}_i^T)$. Observe that by statement 3 of Lemma \ref{facts}, for any $i,j,k$, $\kappa^*(\p{i}|Q_j) >0$ and $\kappa^*(Q_j|P_k)>0$ implies $\kappa^*(\p{i}|P_k)>0$. Note that we can write $\vec{\Lambda}_{j,k} = \ind{\kappa^*(Q_j \, | \, P_k) > 0}$ and $\bPi^+_{j,k} = \ind{\kappa^*(\tilde{P}_j \, | \, P_k) > 0}$. Thus, we must have $\phi(\vec{Z} \vec{\Lambda}) \preceq \bPi^+$.

We now argue that this implies that $\vec{\Lambda}$ is sub-diagonal, i.e., $\vec{\Lambda}_{ij}=0$ for $i<j$. Let $i < j$. If $\vec{\Lambda}_{ij}>0$, then $\vec{Z}_{:,i} \preceq \bPi^+_{:,j}$ by the above relation. Since $\vec{Z}=\bPi^+$, this implies $\bPi^+_{:,i} \preceq \bPi^+_{:,j}$, which implies $j \leq i$ by the assumed ordering of the columns of $\bPi^+$, a contradiction. Hence $\vec{\Lambda}_{ij}=0$ for $i<j$.

Now, since the matrix $\vec{Y}$ (line 1 of Algorithm \ref{vertex_test_alg})
is diagonal, Statement 3 of Lemma~1 gives that for any $i\neq j$ we have $\vec{\Lambda}_{i,:}
\not \preceq \vec{\Lambda}_{j,:}$. One can conclude by a straightforward recursion that
since $\vec{\Lambda}$ is sub-diagonal, this implies that $\vec{\Lambda}$ is in fact
diagonal. Start with the first row $\vec{\Lambda}_{1,:}$ which must be $(1,0,\ldots,0)^T$
(by sub-diagonality). Since $\vec{\Lambda}_{1,:} \not \preceq \vec{\Lambda}_{j,:}$
for $j>1$, this implies the first column $\vec{\Lambda}_{:,1}$ is also $(1,0,\ldots,0)$.
The subsequent columns/rows are handled in the same way.

Hence $\vec{\Lambda}$ is the identity, which implies that $\vec{Q}=\vec{P}$.
\end{proof}
\subsubsection{Proof of Theorem \ref{partial_identification}}
\label{Theorem_partial_label_section}

\begin{proof}
We adopt the notation from the description of Algorithm \ref{partial_label_alg} with the exception that we make explicit the dependence on $k$ by writing $W_i^{(k)}$ instead of $W_i$ and $\bar{Q}^{(k)}_i$ instead of $\bar{Q}_i$. We show that there is a $K$ such that for all $k \geq K$, $(W^{(k)}_1, \ldots, W^{(k)}_L)^T$ is a permutation of $(P_1, \ldots, P_L)^T$. Then, the result will follow from Lemma \ref{vertex_test}. 

Let $Q_i = \vec{\tau}^T_i \vec{P}$, $\bar{Q}^{(k)}_i = \bar{\vec{\tau}}^{{(k)}^T}_i \vec{P}$, and  $W_i^{(k)} = \vec{\gamma}_i^{{(k)}^T} \vec{P}$. Further, let $0 \leq n < L$, $\set{i_1, \ldots, i_n} \subset [L]$, $l \neq j \in [L]$, and define the following events wrt the randomness of $\btau_1,\ldots, \btau_L$:
\begin{align*}
E_{i_1, \ldots, i_n} & = \set{\be_{i_1},\ldots, \be_{i_n}, \btau_{n+1}, \ldots, \btau_L \text{ are linearly independent}} \\
E & = \cap_{\set{i_1,\ldots,i_n} \subset [L], 0 \leq n < L} E_{i_1, \ldots, i_n} \\
F_{l,j} & = \set{ \be_l - \be_j, \btau_2, \ldots, \btau_L \text{ are linearly independent}} \\
F & = \cap_{l \neq j \in [L]} F_{l,j} \\
G_{i_1, \ldots, i_{n-1}}^{l,j} & = \set{ \be_l - \be_j, \be_{i_1}, \ldots, \be_{i_{n-1}}, \btau_{n+1}, \ldots, \btau_L \text{ are linearly independent}} \\
G & = \cap_{ \set{i_1,\ldots,i_{n-1}} \subset [L], 0 \leq n < L, l \neq j \in [L]\setminus \set{i_1,\ldots,i_{n-1}}} G_{i_1, \ldots, i_{n-1}}^{l,j}.
\end{align*}
By Lemma \ref{random_lin_indep}, for any $0 \leq n < L$, $\set{i_1, \ldots, i_n} \subset [L]$, the event $E_{i_1, \ldots, i_n}$ occurs with probability $1$. Similarly, by Lemma \ref{random_lin_indep}, for any $l \neq j \in [L]$, the event $F_{l,j}$ occurs with probability 1. Finally, by Lemma \ref{random_lin_indep}, for any $0 \leq n < L$, $\set{i_1, \ldots, i_{n-1}} \subset [L]$ and any $l \neq j \in [L] \setminus \set{i_1,\ldots,i_{n-1}}$, the event $G_{i_1, \ldots, i_{n-1}}^{l,j}$ occurs with probability $1$. Hence, the event $E \cap F \cap G$ occurs with probability $1$. For the remainder of the proof, assume event $E \cap F \cap G$ occurs.

We prove the claim inductively. We show that for all $n \leq L$ there exists $K_n$ such that if $k \geq K_n$, then $W_1^{(k)}, \ldots, W_n^{(k)}$ are distinct base distributions. 

\medskip
\noindent \textbf{Base Case: $n=1$.} We will apply Lemma \ref{single}. By event $E$, $\vec{\tau}_1, \ldots, \vec{\tau}_L$ are linearly independent. Therefore, $\aff(\vec{\tau}_2, \ldots, \vec{\tau}_L)$ gives a hyperplane with an associated open halfspace $\vec{H}$ that contains $\vec{\tau}_1$ and at least one $\vec{e}_j$. Inspection of Line \ref{partial_label_alg_average} of Algorithm \ref{generate_candidate_alg} shows that $\bar{\vec{\tau}}^{(k)}_1$ is simply the average of $\bmtau_2, \ldots, \bmtau_L$ and does not depend on $k$. Thus, there exists $K_1$ such that for all $k \geq K_1$, if $\vec{e}_j \in \vec{H}$, then $\blambda_k \coloneq \frac{1}{k} \vec{\tau}_1 + \frac{k-1}{k} \bar{\vec{\tau}}^{(k)}_1 \in \co(\vec{e}_j, \vec{\tau}_2, \ldots, \vec{\tau}_L)^\circ$. Fix $k \geq K_1$. Then, by event $E$, for all $\vec{e}_j \in \vec{H}$, there exists a unique $\kappa_j > 0$ and unique $a_{j,2}, \ldots, a_{j,L}$ such that 
\begin{align*}
\vec{\lambda}_k & = \kappa_j \be_j + \sum_{i=2}^L a_{j,i} \btau_i \\
& = \kappa_j \be_j + (1-\kappa_j) \tilde{\btau}_j
\end{align*}
where $\tilde{\btau}_j \in \co(\vec{\tau}_2, \ldots, \vec{\tau}_L)$ is unique. We claim that for all $i \neq j$ and $\set{\be_i,\be_j} \subset \bH$, $\kappa_i \neq \kappa_j$. Suppose to the contrary that there is $i \neq j$ such that $\set{\be_i,\be_j} \subset \bH$ and $\kappa_i = \kappa_j = \kappa$. Then, 
\begin{align*}
\vec{\lambda}_k & = \kappa \vec{e}_i + (1- \kappa) \tilde{\btau}_i \\
\vec{\lambda}_k & = \kappa \vec{e}_j + (1- \kappa) \tilde{\btau}_j. 
\end{align*}
Then, $(1-\kappa) (\tilde{\vec{\tau}}_j - \tilde{\vec{\tau}}_i) - \kappa(\vec{e}_i - \vec{e}_j) = 0$, from which it follows that $\vec{e}_i - \vec{e}_j \in \spa(\vec{\tau}_2, \ldots, \vec{\tau}_L)$. But, by event $F$, $\vec{e}_i - \vec{e}_j, \vec{\tau}_2, \ldots, \vec{\tau}_L$ are linearly independent and, hence, we have a contradiction. Thus, the claim follows. 

Consequently, there is a unique $j$ that minimizes $\kappa_j$. Note that for all $\vec{e}_i \not \in \vec{H}$, if we write $\vec{e}_i = \sum_{l \geq 2 } a_l \vec{\tau}_l + a_1 \vec{\lambda}_k$, then $a_1 \leq 0$. Then, by Lemma \ref{single}, $\vec{e}_j$ is the residue of $\vec{\lambda}_k$ with respect to $\vec{\tau}_2, \ldots, \vec{\tau}_L$. Therefore, by Proposition \ref{equiv_opt}, $\sol (\frac{1}{k} Q_1 + (1- \frac{1}{k}) \bar{Q}_1 \, | \,   \set{Q_j }_{j > 1} )$ is well-defined and if $W^{(k)}_1 \longleftarrow \sol (\frac{1}{k} Q_1 + (1- \frac{1}{k}) \bar{Q}_1 \, | \,   \set{Q_j }_{j > 1} )$, $W^{(k)}_1$ is one of the base distributions. This establishes the base case.

\medskip

\noindent \sloppy \textbf{The Inductive Step:} The proof is similar to the base case. Suppose that there exists $K_{n-1}$ such that for all $k \geq K_{n-1}$, $W_1^{(k)}, \ldots, W_{n-1}^{(k)}$ are distinct base distributions. Let $\set{i_1, \ldots, i_{n-1}} \subset [L]$ denote the indices of the base distributions that are equal to $W_1^{(k)}, \ldots, W_{n-1}^{(k)}$ under the inductive hypothesis. By the event $E$, $\be_{i_1}, \ldots, \be_{i_{n-1}}, \btau_n, \ldots, \btau_L$ are linearly independent. Hence, $\aff(\be_{i_1}, \ldots, \be_{i_{n-1}}, \btau_{n+1}, \ldots, \btau_{L})$ gives a hyperplane with an associated open halfspace $\bH_{i_1, \ldots, i_{n-1}}$ such that $\btau_n \in \bH_{i_1, \ldots, i_{n-1}}$. We claim that there is $\be_j \not \in \set{\be_{i_1}, \ldots, \be_{i_{n-1}}}$ such that $\be_j \in \bH_{i_1, \ldots, i_{n-1}}$. Suppose not. Then, $\be_1, \ldots, \be_L \in \bH_{i_1, \ldots, i_{n-1}}^c$ and $\btau_n \in \bH_{i_1, \ldots, i_{n-1}}$, which implies that $\btau_n \not \in \Delta_{L-1}$. This is a contradiction, so the claim follows.

Define
\begin{align*}
\blambda_k^{(i_1, \ldots, i_{n-1})} \coloneq \frac{1}{k} \btau_n + [\frac{k-1}{k}] \frac{1}{L-1}(\sum_{s > n} \btau_s + \sum_{s < n} \be_{i_s}).
\end{align*}
There exists an integer $K_n^{(i_1, \ldots, i_{n-1})}$ such that if $k \geq K_n^{(i_1, \ldots, i_{n-1})}$, then for all $\be_j \in \bH_{i_1, \ldots, i_{n-1}}$, $\blambda_k^{(i_1, \ldots, i_{n-1})} \in \co(\be_j, \be_{i_1}, \ldots, \be_{i_{n-1}}, \btau_{n+1}, \ldots, \btau_L)^\circ$. Set
\begin{align*}
K_n \coloneq \max(\max_{\set{i_1, \ldots, i_{n-1}} \subset [L]}( K_n^{(i_1, \ldots, i_{n-1})}), K_{n-1}). 
\end{align*}

Fix $k \geq K_n$. Define
\begin{align*}
\blambda_k & \coloneq\frac{1}{k} \btau_n + [\frac{k-1}{k}]\bar{\btau}^{(k)}_n  \\
& = \frac{1}{k} \btau_n + [\frac{k-1}{k}] \frac{1}{L-1}(\sum_{s > n} \btau_s + \sum_{s < n} \bgamma_s^{(k)}).
\end{align*}
By the inductive hypothesis, $k \geq K_{n-1}$, and Proposition \ref{equiv_opt}, there exists $\set{i_1, \ldots, i_{n-1}} \subset [L]$ such that $\bgamma_j^{(k)} = \be_{i_j}$ for all $j \in [n-1]$. For the sake of abbreviation, let $\bH = \bH_{i_1, \ldots, i_{n-1}}$. Thus, $\btau_n \in \bH$ and there exists $\be_j \in \bH$ such that $\be_j \not \in \set{\be_{i_1}, \ldots, \be_{i_{n-1}}}$. Hence, by our choice of $K_n$, for every $\be_j \in \bH$
\begin{align*}
\blambda_k = \blambda_k^{(i_1, \ldots, i_{n-1})} \in \co(\be_j, \be_{i_1}, \ldots, \be_{i_{n-1}}, \btau_{n+1}, \ldots, \btau_L)^\circ. 
\end{align*}

By event $E$ for all $\be_j \in \bH$, there is a unique $\kappa_j > 0$ and unique $a_{j,1}, \ldots, a_{j,n-1}, a_{j,n+1}, \ldots, a_{j,L} >0$ such that
\begin{align*}
\blambda_k & = \kappa_j \be_j + \sum_{l < n} a_{j,i} \be_{i_l} + \sum_{l > n} a_{j,l} \btau_l \\
& = \kappa_j \be_j + (1-\kappa_j) \tilde{\btau}_j
\end{align*}
where $\tilde{\btau}_j \in \co(\be_{i_1},\ldots, \be_{i_{n-1}}, \btau_{n+1}, \ldots, \btau_L)$ is unique.

We claim that for all $l \neq j$ such that $\set{\be_l, \be_j} \subset \bH$, $\kappa_l \neq \kappa_j$. Suppose to the contrary that there exists $l \neq j$ such that $\set{\be_l, \be_j} \subset \bH$ and $\kappa_l = \kappa_j = \kappa$. Then, 
\begin{align*}
\blambda_k = \kappa \be_l + (1-\kappa) \tilde{\btau}_i = \kappa \be_j + (1- \kappa) \tilde{\btau}_j.
\end{align*}
This implies that $\be_l - \be_j \in \spa(\be_{i_1}, \ldots, \be_{i_{n-1}}, \btau_{n+1}, \ldots, \btau_L)$. Observe that $\set{\be_l, \be_j} \subset \bH$ implies that $\be_l \not \in \set{\be_{i_1}, \ldots, \be_{i_{n-1}}}$ and $\be_j \not \in \set{\be_{i_1}, \ldots, \be_{i_{n-1}}}$. Thus, event $G$ implies that $\be_l - \be_j, \be_{i_1}, \ldots, \be_{i_{n-1}}, \btau_{n+1}, \ldots, \btau_L$ are linearly independent. Therefore, we have a contradiction, establishing the claim.

Consequently, there is a unique $j$ that minimizes $\kappa_j$. Note that for all $\be_l \not \in \bH$, if we write $\be_l = \sum_{m < n} a_m \be_{i_m} + \sum_{m > n} a_m \btau_m  + a_n \blambda_k$, then $\ba_n \leq 0$. Then, by Lemma \ref{single}, $\vec{e}_j$ is the residue of $\vec{\lambda}_k$ with respect to $\bgamma_1^{(k)}, \ldots, \bgamma_{n-1}^{(k)}, \vec{\tau}_{n+1}, \ldots, \vec{\tau}_L$. Therefore, by Proposition \ref{equiv_opt}, $\sol(\frac{1}{k} Q_n + (1- \frac{1}{k}) \bar{Q}_n \, | \,   \set{Q_j }_{j > n} \cup \set{W^{(k)}_j}_{j < n})$ is well-defined and if $W_n \longleftarrow \sol(\frac{1}{k} Q_n + (1- \frac{1}{k}) \bar{Q}_n \, | \,   \set{Q_j }_{j > n} \cup \set{W^{(k)}_j}_{j < n})$, $W^{(k)}_n$ is one of the base distributions. Since $\be_j \in \bH$ implies that $\be_j \not \in \set{\be_{i_1}, \ldots, \be_{i_{n-1}}}$, it follows that $W^{(k)}_1, \ldots, W^{(k)}_n$ are distinct base distributions. This establishes the inductive step.

The result follows from applying Lemma \ref{vertex_test}.
\end{proof}

\section{Estimation}
\label{est_app}

In this section, we present the estimation results of our paper. To begin, in Section \ref{residuehat_results_section}, we present the proof of sufficient conditions under which ResidueHat estimators converge uniformly in probability (Proposition \ref{uniform_convergence}). Second, in Section \ref{demixing_mixed_membership_models_estimation_section}, we prove our main estimation result for demixing mixed membership models (Theorem \ref{demix_estim}). Finally, in Section \ref{partial_label_estimation}, we prove our main estimation result for classification with partial labels (Theorem \ref{partial_label_hat}).

\subsection{ResidueHat Results}
\label{residuehat_results_section}

Let $A_1, A_2, \ldots$ denote positive constants whose values may change from line to line. We introduce the following definitions.

\begin{defn}
Let $\est{F}$ and $\est{H}$ be ResidueHat estimators of $F$ and $H$, respectively, where $F \neq H$ and let $G \longleftarrow \text{Residue}(F \, | \, H)$ and $\est{G} \longleftarrow \text{ResidueHat}(\est{F} \, | \, \est{H})$. If $\est{G}$ is a ResidueHat estimator of order $0$, we say its \emph{distributional ancestors} are $\set{F,H}$ and define $\text{ancestors}(\est{G}) \coloneq \set{F,H}$. If $\est{G}$ is a ResidueHat estimator of the $k$th order, we define its distributional ancestors to be $\text{ancestors}(\est{G}) = \text{ancestors}(\est{F}) \cup \text{ancestors}(\est{H})$. 
\end{defn}
\noindent The constants in our bounds depend on the distributional ancestors.
\begin{defn}
We say that the distribution $F$ satisfies the \emph{support condition} \textbf{(SC)} with respect to $H$ if there exists a distribution $G$ and $\gamma \in [0,1)$ such that $\supp(H) \not \subseteq \supp(G)$ and $F = (1-\gamma) G + \gamma H$.
\end{defn}
\begin{defn}
If
\begin{align*}
\sup_{E \in \sE} |\est{F}(E) - F(E)| \overset{i.p.}{\longrightarrow} 0
\end{align*}
as $\vec{n} \longrightarrow \infty$, we say that $\est{F} \longrightarrow F$ uniformly (or $\est{F}$ converges uniformly to $F$) with respect to $\sE$.
\end{defn}

\begin{defn}
Let $\est{F}$ be a ResidueHat estimator of a distribution $F$. We say that $\est{F}$ satisfies a Uniform Deviation Inequality \textbf{(UDI)} with respect to $\sE$ if for 
all $\epsilon > 0$, there exist constants $A_{1,\epsilon}, A_{2,\epsilon} > 0$ and $\vec{N}$ depending on $\text{ancestors}(\est{F})$ such that if $\vec{n} \geq \vec{N}$, then for all $E \in \sE$
\begin{align*}
|\est{F}(E) - F(E)| < A_{1,\epsilon}\vc + \epsilon
\end{align*}
\noindent with probability at least $1 - A_{2,\epsilon} \sum_{i \in [L]} \frac{1}{n_i}$
\end{defn}
\noindent Henceforth, for the purposes of abbreviation, we will only say that a ResidueHat estimator satisfies a Uniform Deviation Inequality \textbf{(UDI)} and omit ``with respect to $\sE$" because the context makes this clear.  
\begin{defn}
Let $\est{F}$ and $\est{H}$ be ResidueHat estimators. We say that $\est{\kappa}(\est{F} \, | \, \est{H})$ satisfies a Rate of Convergence \textbf{(RC)} with respect to $\sE$ if for all $\epsilon > 0$, there exists constants $A_{1,\epsilon} ,A_{2,\epsilon} > 0$ and $\vec{N}$ depending on $\text{ancestors}(\est{F}) \cup \text{ancestors}(\est{H})$ such that for $\vec{n} \geq \vec{N}$,
\begin{align*}
|\est{\kappa}(\est{F} \, | \, \est{H})  - \kappa^*(F \, | \, H) | \leq A_{1,\epsilon} \vc+\epsilon
\end{align*}
\noindent with probability at least $1 - A_{2,\epsilon} \sum_{i \in [L]} \frac{1}{n_i}$.
\end{defn}
 
Lemma \ref{sup_cond} gives sufficient conditions under which $F$ satisfies \textbf{(SC)} with respect to $H$.
\begin{lemma}
\label{sup_cond}
Let $P_1, \ldots, P_L$ satisfy \textbf{(A$''$)} and let $F,H \in \co(P_1, \ldots, P_L)$ such that $F \neq H$. Then, $F$ satisfies \textbf{(SC)} with respect to $H$.
\end{lemma}
\begin{proof}
Let $A = \argmin(|B| : B \subseteq \set{P_1, \ldots, P_L}, F,H \in \co(B))$. Without loss of generality, suppose that $A = \set{P_1, \ldots, P_K}$. $F$ either lies on the boundary of $\co(P_1, \ldots,P_K)$ or doesn't. If $F$ lies on the boundary of $\co(P_1, \ldots,P_K)$, then $H \in \co(P_1, \ldots, P_K)^\circ$ by minimality of $A$. Then, we pick $G = F$ and $\gamma = 0$ to obtain $F = (1-\gamma) F + \gamma H$. Since $P_1, \ldots, P_L$ satisfy \textbf{(A$''$)}, $\supp(H) \not \subseteq \supp(F)$. 

Now, suppose that $F \in \co(P_1, \ldots, P_K)^\circ$. Let $G \longleftarrow \text{Residue}(F \, | \, H)$; we can write $F = (1-\gamma)G + \gamma H$ for $\gamma \in [0,1)$ since $F \neq H$. Then, by Statement \emph{2} of Lemma \ref{facts} and statement \emph{3} of Proposition \ref{equiv_opt}, $G$ is on the boundary of $\co(P_1, \ldots, P_K)$. Without loss of generality, suppose that $G \in \co(P_1, \ldots, P_{K-1})$. Since $F = (1- \gamma) G + \gamma H \in \co(P_1, \ldots, P_K)^\circ$, and $G \in \co(P_1, \ldots, P_{K-1})$, $H \not \in \co(P_1, \ldots, P_{K-1})$. Since $P_1, \ldots, P_L$ satisfy \textbf{(A$''$)}, $\supp(H) \not \subseteq \supp(G)$. This completes the proof.
\end{proof}
Lemma \ref{vc_type_ineq} gives sufficient conditions under which an estimator $\est{G}$ satisfies a \textbf{(UDI)}.
\begin{lemma}
\label{vc_type_ineq}
Let
\begin{enumerate}
\item $F$ and $H$ be distributions such that $F \neq H$,

\item $G \longleftarrow \text{Residue}(F \, | \, H)$, and

\item $\est{G} \longleftarrow \text{ResidueHat}(\est{F} \, | \, \est{H})$.
\end{enumerate}
If $\est{\kappa}(\est{F} \, | \, \est{H})$ satisfies a \textbf{(RC)}, $\est{H}$ satisfies a \textbf{(UDI)}, and $\est{F}$ satisfies a \textbf{(UDI)}, then $\est{G}$ satisfies a \textbf{(UDI)}. 
\end{lemma}
\begin{proof}
For the sake of abbreviation, let $\est{\kappa} = \est{\kappa}(\est{F} \, | \, \est{H})$, $\kappa^* = \kappa^*(F \, | \, H)$, $\est{\alpha} = \frac{1}{1- \est{\kappa}}$ and $\alpha^* = \frac{1}{1-\kappa^*}$. Let $\epsilon > 0$. We claim that there are constants $A_{1,\epsilon}, A_{2,\epsilon} > 0$ such that for sufficiently large $\vec{n}$,
\begin{align}
\Pr(|\est{\alpha}  - \alpha^* | < A_{1,\epsilon}\vc+\epsilon) & \geq 1-  A_{2,\epsilon} \sum_{i \in [L]} \frac{1}{n_i}. \label{alpha_rate}
\end{align}
\noindent Let $\delta = \frac{\epsilon (1-\kappa^*)^2}{2}$. Since $\est{\kappa}$ satisfies a \textbf{(RC)}, there exists constants $A_{1,\delta},A_{2,\delta} > 0$ such that for large enough $\vec{n}$,
\begin{align*}
|\est{\kappa} - \kappa^*| \leq A_{1,\delta} \vc + \delta
\end{align*}
\noindent with probability at least $1 - A_{2,\delta} \sum_{i \in [L]} \frac{1}{n_i}$. Since $F \neq H$, $\kappa^* < 1$ by Proposition \ref{label_noise_bin_case}, so we can let $\vec{n}$ large enough so that
\begin{align*}
\frac{1}{(1-\kappa^*)(1-\est{\kappa})} \leq 2 \frac{1}{(1-\kappa^*)^2}
\end{align*} 
with high probability. Then, on this same event, for large enough $\vec{n}$,
\begin{align*}
|\frac{1}{1-\kappa^*} - \frac{1}{1-\est{\kappa}}| & \leq  \frac{ A_{1,\delta} \vc + \delta}{(1-\kappa^*)(1-\est{\kappa})} \\
& \leq  2 \frac{A_{1,\delta} \vc+\delta}{(1-\kappa^*)^2} \\ 
& \leq 2 \frac{A_{1,\delta} \vc}{(1-\kappa^*)^2} + \epsilon.
\end{align*}
Thus, we obtain the claim. 

We can write $G= \alpha F + (1-\alpha) H$ with $\alpha \geq 1$. Then, by the triangle inequality, 
\begin{align*}
|\est{G} - G| & = | \est{\alpha} \est{F} + (1-\est{\alpha}) \est{H} -  \alpha F - (1-\alpha) H | \\
& \leq |\est{\alpha} \est{F} - \alpha F | + |(1-\est{\alpha}) \est{H} - (1-\alpha) H| \\
& = |\est{\alpha} \est{F} -\est{\alpha} F + \est{\alpha} F - \alpha F | + |(1-\est{\alpha}) \est{H} - (1-\est{\alpha}) H  + (1-\est{\alpha})H - (1-\alpha) H| \\
& \leq |\est{\alpha}| |\est{F} - F| +  |\est{\alpha} - \alpha | + |1-\est{\alpha}||\est{H} -  H | + |\est{\alpha}- \alpha|.
\end{align*}
\noindent Since $\est{F}$ satisfies a \textbf{(UDI)}, $\est{H}$ satisfies a \textbf{(UDI)}, inequality (\ref{alpha_rate}) holds, and $|\est{\alpha}|$ and $|1-\est{\alpha}|$ are bounded in probability, the result follows by an application of a union bound and picking the $\epsilon$s in the uniform deviation inequalities appropriately for each term. 
\end{proof}

Lemma \ref{rate} gives sufficient conditions under which $\est{\kappa}$ satisfies \textbf{(RC)}.
\begin{lemma}
\label{rate}
Let $F$ and $H$ be distributions such that $F \neq H$. If
\begin{itemize}
\item $F$ satisfies \textbf{(SC)} with respect to $H$,

\item $\est{F}$ satisfies \textbf{(UDI)}, and

\item $\est{H}$ satisfies \textbf{(UDI)}, 

\end{itemize}
then $\est{\kappa}(\est{F} \, | \, \est{H})$ satisfies \textbf{(RC)}.
\end{lemma}
\begin{proof}
For abbreviation, let $\kappa^* = \kappa(F \, | \, H)$ and $\est{\kappa} = \est{\kappa}(\est{F} \, | \, \est{H})$. 

We first prove the upper bound. $F$ satisfies \textbf{(SC)} with respect to $H$, so there exists a distribution $G$ such that $F = (1-\gamma) G + \gamma H$ for some $\gamma \in [0,1)$ and $\supp(H) \not \subseteq \supp(G)$. Therefore, we have that $G$ is irreducible with respect to $H$ and, by Proposition \ref{label_noise_bin_case}, $\kappa^* = \gamma$. 

Let $\delta > 0$ (to be chosen later). Since by hypothesis $\est{F}$ and $\est{H}$ satisfy \textbf{(UDI)}, there exist constants $A_{1,\delta}, A_{2,\delta} > 0$ such that for large enough $\vec{n}$, with probability at least $1 - A_{1,\delta}[\sum_{i \in [L]} \frac{1}{n_i}]$, for all $E \in \sE$,
\begin{align}
|\est{F}(E) - F(E) | & < A_{2,\delta} \vc + \delta \\
|\est{H}(E) - H(E) | & < A_{2,\delta} \vc + \delta. \label{udi_rc_proof}
\end{align}
\noindent Without loss of generality, let $A_{1,\delta}, A_{2,\delta} > 1$. 

Pick $R \in \sE$ such that $H(R) > 0$. By inequality (\ref{udi_rc_proof}), there exists $\vec{N}_1$ such that $\vec{n} \geq \vec{N}_1$ implies that $\est{H}(R) - \vc > 0$ with high probability. This implies that for $\vec{n} \geq \vec{N}_1$, $\est{\kappa}$ is finite. Let $\epsilon > 0$. By definition of $\est{\kappa}$, there exists $E \in \sE$ such that
\begin{align*}
\frac{\epsilon}{2} + \est{\kappa} & \geq \frac{\est{F}(E) + \vc}{(\est{H}(E)- \vc)_+}.
\end{align*}
Since $\est{\kappa}$ is finite, we have that $\est{H}(E) > \vc$ and $H(E) > 0$. Then,
\begin{align*}
\frac{\epsilon}{2} + \est{\kappa} & \geq \frac{\est{F}(E) + \vc}{\est{H}(E)- \vc} \\
& \geq \frac{F(E) - (A_{2,\delta}-1)\vc-\delta}{H(E)+ (A_{2,\delta}-1)\vc+\delta} \\
& \geq \frac{ \gamma H(E) }{H(E)+ (A_{2,\delta}-1)\vc + \delta} - \frac{ (A_{2,\delta}-1)\vc}{H(E)+ (A_{2,\delta}-1)\vc + \delta} - \frac{ \delta}{H(E)+ (A_{2,\delta}-1)\vc + \delta}   \\
& \geq \frac{\gamma H(E)}{H(E)} - \frac{(A_{2,\delta} - 1) \vc+\delta}{H(E)} - \frac{(A_{2,\delta} -1) \vc}{H(E) + (A_{2,\delta}-1)\vc+\delta} -\frac{\delta}{H(E) + (A_{2,\delta}-1)\vc+\delta}\\
& \geq \kappa^* - 2\frac{(A_{2,\delta} - 1) \vc}{H(E)} -2 \frac{\delta}{H(E)}\\
\end{align*} 
where in the second to last inequality we used the elementary fact that if $a,b,c > 0$ and $a \leq b$, then $\frac{a}{b+c} \geq \frac{a}{b} - \frac{c}{b}$. Picking $\delta = \frac{H(E) \epsilon}{4}$, we obtain the upper bound.

The proof of the other direction of the inequality is very similar to the proof of Theorem 2 in \citet{scott2015}. By hypothesis, $F$ satisfies \textbf{(SC)} with respect to $H$, so there exists a distribution $G$ such that $F = (1-\gamma) G + \gamma H$ for some $\gamma \in [0,1)$ and $\supp(H) \not \subseteq \supp(G)$. Therefore, we have that $G$ is irreducible with respect to $H$ and, by Proposition \ref{label_noise_bin_case}, $\kappa^*(F \, | \, H) = \gamma$. For abbreviation, let $\kappa^* = \kappa^*(F \, | \, H)$ and $\est{\kappa} = \est{\kappa} (\est{F} \, | \, \est{H})$. Since $\supp(H) \not \subseteq \supp(G)$, there exists an open set $O$ such that 
\begin{align*}
\frac{F(O)}{H(O)} & = ( 1- \gamma) \frac{G(O)}{H(O)} + \gamma = \kappa^*.
\end{align*}
\noindent Then, since $\sE$ contains a generating set for the standard topology on $\bbR^d$, there exists $E \in \sE$ such that
\begin{align*}
\frac{F(E)}{H(E)} = \kappa^*.
\end{align*}
Let $\delta>0$ such that $\delta \leq \frac{1}{4}H(E)$. Since by hypothesis $\est{F}$ and $\est{H}$ satisfy \textbf{(UDI)}, there exist constants $A_{3,\delta}, A_{4,\delta} > 0$ such that for large enough $\vec{n}$, with probability at least $1 - A_{3,\delta}[\sum_{i \in [L]} \frac{1}{n_i}]$,
\begin{align*}
\est{\kappa} & \leq \frac{F(E) + A_{4,\delta}\vc+\delta }{(H(E) - A_{4, \delta}\vc -\delta)_+} \\
& \leq \frac{F(E) + \epsilon}{(H(E) - \epsilon)_+}
\end{align*}
\noindent where $\epsilon = 2A_{4,\delta} \vc +\delta$. The rest of the proof is identical to the proof of Theorem 2 from \citet{scott2015} and, therefore, we omit it.
\end{proof}
The following theorem gives sufficient conditions under which a ResidueHat estimator satisfies \textbf{(UDI)}. It is the basis of Proposition \ref{uniform_convergence}.
\begin{lemma}
\label{main_est_ineq}
If $P_1, \ldots, P_L$  satisfy \textbf{(A$''$)} and $\est{G}$ is a ResidueHat estimator of order $k$ of a distribution $G \in \co(P_1, \ldots, P_L)$, then $\est{G}$ satisfies \textbf{(UDI)}.
\end{lemma}
\begin{proof}
\sloppy Let $\est{G} \longleftarrow \text{ResidueHat}(\est{F} \, | \, \est{H})$ where $G \longleftarrow \text{ResidueHat}(F \, | \, H)$, $F \neq H$, $F, H \in \co(P_1, \ldots, P_L)$ and $\est{F}, \est{H}$ are ResidueHat estimators of $F$ and $H$ respectively. We use induction on $k$. Suppose $k=0$. Then, $\est{F}$ and $\est{H}$ are empirical distributions. Therefore, the VC inequality applies to $\est{F}$ and $\est{H}$. Consequently, $\est{F}$ and $\est{H}$ satisfy \textbf{(UDI)}. Since $P_1, \ldots, P_L$  satisfy \textbf{(A$''$)}, $F, H \in \co(P_1, \ldots, P_L)$ and $F \neq H$, by Lemma \ref{sup_cond}, $F$ satisfies \textbf{(SC)} with respect to $H$. Then, by Lemma \ref{rate}, $\est{\kappa}(\est{F} \, | \, \est{H})$ satisfies \textbf{(RC)}. Then, all of the assumptions of Lemma \ref{vc_type_ineq} are satisfied, so $\est{G}$ satisfies \textbf{(UDI)}. Note that $G \in \co(P_1, \ldots, P_L)$ by Proposition \ref{equiv_opt}. 

The inductive step $(k > 0)$ follows by similar reasoning. The difference is that instead of applying the VC inequality to $\widehat{F}$ and $\widehat{H}$, we use the fact that $\widehat{F}$ and $\widehat{H}$ are ResidueHat estimators of order $k-1$ and, therefore, satisfy \textbf{(UDI)} by the inductive hypothesis.
\end{proof}
\begin{proof}[Proof of Proposition \ref{uniform_convergence}]
Let $0 < \delta < \epsilon$.	By Lemma \ref{main_est_ineq}, $\est{G}$ satisfies \textbf{(UDI)}. Consequently, there exist constants $A_{1,\delta}, A_{2,\delta} > 0$ such that for large enough $\vec{n}$ with probability at least $1 - A_{1,\delta} \sum_{i \in [L]} \frac{1}{n_i}$, $\est{G}$ satisfies for every $E \in \sE$,
\begin{align*}
|\est{G}(E) - G(E)| & \leq A_{2,\delta} \vc + \delta =  A_{2,\delta} \sum_{i \in [L]}  \epsilon_i(\frac{1}{n_i}) + \delta \longrightarrow \delta < \epsilon.
\end{align*}
\end{proof}
\subsection{Demixing Mixed Membership Models}
\label{demixing_mixed_membership_models_estimation_section}

In this section, we prove our main estimation result for demixing mixed membership models, i.e., Theorem \ref{demix_estim}. First, in Section \ref{facetesthat_algorithm_section}, we present an important lemma for FaceTestHat. Second, in Section \ref{demixhat_algorithm_section}, we present an empirical version of Demix and prove Theorem \ref{demix_estim}.

\subsubsection{The FaceTestHat Algorithm}
\label{facetesthat_algorithm_section}

The following establishes that FaceTestHat behaves as desired.

\begin{lemma}
\label{face_test_hat}
Let $\epsilon \in (0,1)$. For all $j \in [K]$, let $Q_j = \vec{\eta}^T_j \vec{P}$ and $\vec{\eta}_j \in \si_K$ such that every $\vec{\eta}_j$ lies in the relative interior of the same face of $\si_K$. Let $P_1,\ldots, P_K$ satisfy \textbf{(A$''$)}, and $Q_1, \ldots, Q_K \in \co(P_1, \ldots, P_K)$ be distinct. Let $\est{Q}_i$ be a ResidueHat estimate of $Q_i$ $\forall i \in [K]$. 
\begin{enumerate}
\item With probability tending to $1$ as $\vec{n} \longrightarrow \infty$, if FaceTestHat($\est{Q}_1,\cdots,\est{Q}_K \, | \, \epsilon$) returns $1$, then $\vec{\eta}_1, \ldots, \vec{\eta}_K$  are in the relative interior of the same face. 

\item Let $\kappa^*_{i,j} = \kappa^*(Q_i \, | \, Q_j)$. If $\vec{\eta}_1, \ldots, \vec{\eta}_K$  are in the relative interior of the same face and $\min_{i,j} \kappa^*_{i,j} > \epsilon$, then with probability tending to $1$ as $\vec{n} \longrightarrow \infty$, FaceTestHat($\est{Q}_1,\cdots,\est{Q}_K \, | \, \epsilon$) returns $1$.
\end{enumerate}

\end{lemma}
\begin{proof}
Let $\epsilon > 0$, $\kappa^*_{i,j} = \kappa^*(Q_i \, | \, Q_j )$ and $\est{\kappa}_{i,j} = \est{\kappa}( \est{Q}_i \, | \, \est{Q}_j)$. Since $P_1,\ldots, P_K$ satisfy \textbf{(A$''$)} and $Q_i \neq Q_j$, by Lemma \ref{sup_cond}, $Q_i$ satisfies \textbf{(SC)} wrt $Q_j$. Since $\est{Q}_i$ and $\est{Q_j}$ are ResidueHat estimators, $\est{Q}_i$ and $\est{Q}_j$ satisfy \textbf{(UDI)} (Lemma \ref{main_est_ineq}). Then, by Lemma \ref{rate} $\est{\kappa}_{i,j}$ satisfies \textbf{(RC)}.
\begin{enumerate}

\item We prove the contrapositive. Suppose that $Q_1, \ldots, Q_K$ are not in the relative interior of the same face. Then, by Proposition \ref{face_test}, $\text{FaceTest}(Q_1,\ldots,Q_K)$ returns $0$, which occurs if and only if there exist $i \neq j$ such that $\kappa^*_{i,j} = 0$. Since $\est{\kappa}_{i,j}$ satisfies \textbf{(RC)}, as $\vec{n} \longrightarrow \infty$, with probability tending to 1, $\est{\kappa}_{i,j} \longrightarrow 0$. This completes the proof.

\item If $\min_{i,j} \kappa^*_{i,j} > \epsilon$, then as $\vec{n} \longrightarrow \infty$, with probability tending to 1, $\min_{i,j} \est{\kappa}_{i,j} > \epsilon$.
\end{enumerate}
\end{proof}

\subsubsection{The DemixHat Algorithm}
\label{demixhat_algorithm_section}

The DemixHat algorithm (see Algorithm \ref{demix_hat_alg}) differs from the Demix algorithm in that \emph{(i)} it requires the specification of a constant $\epsilon \in (0,1)$ and \emph{(ii)} it only uses the two-sample $\kappa^*$ operator. In the interest of clarity, we state the population version of the algorithm DemixHat, which we call Demix2. The only difference between Demix and Demix2 is that line \ref{multi_sample_kappa_demix} in Demix has been replaced with lines \ref{last_loop_start_demix2}-\ref{last_loop_end_demix2} in Demix2.

\begin{algorithm}[t]
\caption{Demix2($S_1, \ldots, S_K$)}
\textbf{Input: }$S_1, \ldots, S_K$ are distributions
\label{demix2_alg}
\begin{algorithmic}[1]
\IF{$K = 2$}
\RETURN $(\text{Residue}(S_1 \, | \, S_2),  \text{Residue}(S_2 \, | \, S_1))^T$ \label{demix2_res_1}
\ELSE
\STATE $(R_2, \ldots, R_K)^T \longleftarrow \text{FindFace}(S_1, \ldots,S_K)$
\STATE $(Q_1, \ldots, Q_{K-1})^T \longleftarrow \text{Demix2}(R_2, \ldots, R_K)$
\FOR{$i = 1, \ldots, K-1$} \label{last_loop_start_demix2}
\STATE $Q_K \longleftarrow \text{Residue}(Q_K \, | \, Q_i)$ \label{last_residue_compute_demix2}
\ENDFOR \label{last_loop_end_demix2}
\RETURN $(Q_1, \ldots, Q_K)^T$
\ENDIF
\end{algorithmic}
\end{algorithm}

Lemma \ref{mod_ident_lemma} establishes that it is possible to replace line \ref{multi_sample_kappa_demix} of the Algorithm \ref{demix_alg} with the sequence of applications of the two-sample $\kappa^*$ in lines \ref{last_loop_start_demix2}-\ref{last_loop_end_demix2} of Algorithm \ref{demix2_alg}, without changing the conclusion of Theorem \ref{demix_identification}.
\begin{lemma}
\label{mod_ident_lemma}
\sloppy Let $\set{i_1, \ldots, i_K} \subset [L]$ be distinct indices. Let $P_1, \ldots,P_L$ be jointly irreducible, $(Q_1, \ldots, Q_{K-1})$ be a permutation of $(P_{i_1}, \ldots, P_{i_{K-1}})$ and $Q^1_K \in \co(P_{i_1}, \ldots, P_{i_K})^\circ$. Define the sequence 
\begin{align*}
Q_K^i & \longleftarrow \text{Residue}(Q_K^{i-1} \, | \, Q_{i-1});
\end{align*}
\noindent then, $Q_K^K = P_{i_K}$.
\end{lemma}
\begin{proof}
Relabel the distributions so that $Q_j = P_j$. Let $\vec{\mu}_i$ denote the mixture proportion of $Q_K^i$ and $\vec{e}_j$ the mixture proportion of $P_j$. Write $\bmu_1 = \sum_{i=1}^K \alpha_i \be_i$. We claim that $\bmu_k = \frac{\sum_{i \geq k} \alpha_i \be_i}{\sum_{i \geq k} \alpha_i}$ for all $k \leq K$. We prove this inductively. The base case $k=1$ follows since $\sum_{i \geq 1} \alpha_i = 1$. Next, we prove the inductive step. Suppose that $\bmu_{k-1} = \frac{\sum_{i \geq k-1} \alpha_i \be_i}{\sum_{i \geq k-1} \alpha_i}$. By Proposition \ref{equiv_opt}, the mixture proportion of $Q_K^k$, $\bmu_k$, is the residue of $\bmu_{k-1}$ with respect to $\be_{k-1}$. By statement 1 of Lemma \ref{facts}, we can write
\begin{align*}
\bmu_k & = \be_{k-1} + \alpha^* (\bmu_{k-1} - \be_{k-1}) \\
& = \frac{[\sum_{i \geq k-1} \alpha_i (1-\alpha^*) + \alpha^* \alpha_{k-1}] \be_{k-1} + \alpha^* \sum_{i \geq k} \alpha_i \be_i }{\sum_{i \geq k-1} \alpha_i}
\end{align*}
where
\begin{align*}
\alpha^* = \frac{1}{1-\kappa^*(\bmu_{k-1} \, | \, \be_{k-1})}
\end{align*} 
and we have used the inductive hypothesis $\bmu_{k-1} = \frac{\sum_{i \geq k-1} \alpha_i \be_i }{\sum_{i \geq k-1} \alpha_i}$. $\alpha^*$ is the value of the following optimization problem (in statement 1 of Lemma \ref{facts}):
\begin{align*}
\max(\alpha \geq 1 \, | \, \exists G, G = \bmu_{k-1} + \alpha(\be_{k-1} - \bmu_{k-1})).
\end{align*}
Inspection of the above optimization problem reveals that $\alpha^* = \frac{\sum_{i \geq k-1} \alpha_i}{\sum_{i \geq k} \alpha_i}$. Plugging this into the above equation gives $\bmu_k = \frac{\sum_{i \geq k} \alpha_i \be_i}{\sum_{i \geq k} \alpha_i}$. This establishes the claim.

Setting $k = K$, it follows that $\mu_K = \frac{\alpha_K \be_K}{\alpha_K} = \be_K$.
\end{proof}

\begin{cor}
\label{demix2_identification}
Let $P_1, \ldots, P_L$ be jointly irreducible and $\vec{\Pi}$ have full column rank. Then, with probability $1$, Demix2$(\tilde{\vec{P}})$ returns a permutation of  $\vec{P}$.
\end{cor}

\begin{proof}[Proof of Theorem \ref{demix_estim}]
Note that every estimator of a distribution in the DemixHat algorithm is a ResidueHat estimator since \emph{(i)} the Demix2 algorithm \ref{demix2_alg} only considers distributions that are in $\co(P_1, \ldots, P_L)$ and \emph{(ii)} only computes Residue($F \, | \, H$) if $F \neq H$. To see why \emph{(ii)} is true, consider: the Demix2 algorithm computes Residue($\cdot \, | \, \cdot$) at lines  \ref{demix2_res_1} and \ref{last_residue_compute_demix2} in Demix2 and line \ref{findface_res_line} in FindFace. In the proof of Theorem \ref{demix_identification}, we showed that $S_1, \ldots, S_K$ are always linearly independent and therefore distinct. This implies that in lines \ref{demix2_res_1} and \ref{last_residue_compute_demix2} in Demix2 and line \ref{findface_res_line} in FindFace, the residue function is called on distinct distributions. Thus, every estimator of a distribution of the DemixHat algorithm satisfies the assumptions of Lemma \ref{main_est_ineq}. 

First, we argue that the order of the ResidueHat estimators is bounded; this implies that the constants in the uniform deviation inequalities associated with the ResidueHat estimators are bounded. We give a very loose bound. DemixHat calls itself at most $L-1$ times and in each call recurses on at most $L-1$ ResidueHat estimators and calculates at most $L-1$ more ResidueHat estimators. Therefore, each ResidueHat estimator has order at most $(L-1)^3$.

Second, let $A_i$ denote the event that DemixHat recurses on $i$ distributions lying in the relative interior of an $i$-face in the $(L-i)$th recursive call. We show that the event $\cap_{i =2}^{L-1} A_i$ occurs with probability tending to $1$ as $\vec{n} \longrightarrow \infty$. Consider $A_{L-1}$. Let $\est{R}_i^{(n)}$ denote the estimate of the $i$th distribution in line \ref{R_hat_demixhat} in the $n$th iteration of the for loop in Algorithm \ref{findface_hat_alg} and let $R_i^{(n)}$ denote the corresponding distribution. Let $\kappa^*_{i,j,n} = \kappa^*(R^{(n)}_i \, | \, R^{(n)}_j)$. From the proof of Theorem \ref{demix_identification}, there exists an integer $N_1 \geq 0$ such that for $n \geq N_1$, $R_i^{(n)}$ lies in the relative interior of the same face for all $i = 2, \ldots, {L}$. Further, using the notation from the proof of Theorem \ref{demix_identification}, we have that the mixture proportions of the $R_i^{(n)}$s, i.e., the $\bmu_i^{(n)}$s, converge to a common $\blambda$ on this face, i.e., for all $i =2,\ldots, L$, $\norm{\bmu_i^{(n)} - \blambda} \longrightarrow 0$. Thus, by statement 3 of Lemma \ref{continuity}, for all $i \neq j \in [L] \setminus \set{1}$ $\kappa^*(\bmu_i^{(n)} \, | \, \bmu_j^{(n)}) \longrightarrow 1$. Hence, there exists $N_2 \geq N_1$ such that $\kappa^*(\bmu_i^{(N_2)} \, | \, \bmu_j^{(N_2)}) > \epsilon$ for all $i \neq j$. By statement 1 of Lemma \ref{face_test_hat} and a union bound argument, with probability increasing to $1$, $\text{FaceTestHat}(\est{R}_2^{(n)}, \ldots, \est{R}_L^{(n)} \, | \, \epsilon)$ returns $0$ for all $n < N_1$ since $R_2^{(n)}, \ldots, R_L^{(n)}$ are not on the relative interior of the same face. Thus, with probability tending to $1$, FaceTest does not make the mistake to return $1$ before the distributions $R_2^{(n)}, \ldots, R_L^{(n)}$ are on the relative interior of the same face. By statement 2 of Lemma \ref{face_test_hat}, with probability tending to $1$ as $\bn \longrightarrow \infty$, $\text{FaceTestHat}(\est{R}_2^{(N_2)}, \ldots, \est{R}_L^{(N_2)} \, | \, \epsilon)$ returns $1$. Hence, with probability increasing to $1$, the event $A_{L-1}$ occurs. Applying the same argument to $A_i$ for $i < L-1$ and taking the union bound shows that $\cap_{i =2}^{L-1} A_i$ occurs with probability tending to $1$ as $\bn \longrightarrow \infty$. 

Now, we can complete the proof. Under the assumptions of Theorem \ref{demix_identification}, there is a permutation $\sigma$ such that for each distribution $Q_i$ estimated by $\est{Q}_i$, $P_{\sigma(i)} = Q_i$. By Proposition \ref{uniform_convergence}, as $\vec{n} \longrightarrow \infty$, $\est{Q}_i$ converges uniformly to $Q_i$. The result follows.
\end{proof}

\subsection{Classification with Partial Labels}
\label{partial_label_estimation}

In this section, we prove Theorem \ref{partial_label_hat}. To begin, we briefly sketch an argument that one can reduce any instance of a partial label model satisfying \textbf{(B3)} and \textbf{(A)} to an instance of a partial label model that also satisfies \textbf{(D)}. Let $J = \set{i : \bPi^+_{i,:} = \vec{e}_j^T \text{ for some } j \in [L]} = \set{j_1, \ldots, j_k}$, the set of indices of contaminated distributions that are equal to some base distribution. Compute $\text{Residue}(\p{i} \, | \, \p{{j_1}})$ for $i \in [L] \setminus J$ if there is $l$ such that $\bPi^+_{i,l} = \bPi^+_{j_1, l} = 1$. Replace $\p{i}$ with $\text{Residue}(\p{i} \, | \, \p{{j_1}})$ (and call it $\p{i}$ for simplicity of presentation). Update $\bPi^+$ and remove $j_1$ from $J$. Repeat this procedure until $J$ is empty. Then, there will be $(L - |J|)$ $\p{i}$ lying in a $(L - |J|)$-face of $\Delta_L$ that are not equal to any of the base distributions and the other contaminated distributions will be equal to base distributions. Then, it suffices to solve the instance of the partial label model on the $(L - |J|)$-face, which satisfies \textbf{(D)}.

Next, we introduce VertexTestHat (Algorithm \ref{vertex_test_hat_alg}), an empirical version of VertexTest, and prove that it satisfies a useful consistency property.

\begin{algorithm}[t]
\caption{$\text{VertexTestHat}(\bPi^+, (\ed{\tilde{P}}_1, \ldots, \ed{\tilde{P}}_M)^T, (\est{Q}_1, \ldots, \est{Q}_L)^T)$}
\begin{algorithmic}[1]
\label{vertex_test_hat_alg}
\STATE Form the matrix $\est{\vec{M}}_{i,j} \coloneq \est{\kappa}(\ed{\tilde{P}}_i \, | \, \est{Q}_j)$
\STATE Let $|\bPi^+|$ denote the number of nonzero entries in $\bPi^+$
\STATE Form the matrix $\est{\vec{Z}}$ by making the $|\bPi^+|$ largest entries of $\est{\vec{M}}$ equal to $1$ and the rest of its entries equal to $0$
\STATE Use any algorithm that finds a permutation matrix $\vec{C}$ such that $\widehat{\vec{Z}} \vec{C} = \bPi^+$ (if it exists)
\IF{such a permutation matrix $\vec{C}$ exists}
\RETURN $(1, \vec{C}^T)$
\ELSE
\RETURN $(0, \vec{0})$
\ENDIF
\end{algorithmic}
\end{algorithm}

\begin{lemma}
\label{vertex_hat_lemma}
\sloppy Suppose that $P_1, \ldots, P_L$ satisfy \textbf{(A$''$)}, $\vec{\Pi}$ has full column rank, the columns of $\bPi^+$ are unique and $\bPi^+$ satisfies \textbf{(D)}. Let $\est{Q}_1, \ldots, \est{Q}_L$ be ResidueHat estimators of $Q_1, \ldots, Q_L$, respectively. Suppose that  $(Q_1, \ldots, Q_L)$ is a permutation of $(P_1, \ldots, P_L)$. Then, with probability tending to $1$ as $\vec{n} \longrightarrow \infty$, $\text{VertexTestHat}(\bPi^+, (\ed{\tilde{P}}_1, \ldots, \ed{\tilde{P}}_M)^T, (\est{Q}_1, \ldots, \est{Q}_L)^T)$ returns a permutation matrix $\vec{C}$ such that $\forall i$, $\vec{C}_{i,:} (\est{Q}_1, \ldots, \est{Q}_L)^T$ is a ResidueHat estimator of $P_i$.
\end{lemma}
\begin{proof}
Define $\est{\kappa}_{i,j} \coloneq \est{\kappa}(\ed{\tilde{P}}_i \, | \, \est{Q}_j)$ and $\kappa^*_{i,j} \coloneq \kappa^*(\p{i} \, | Q_j)$. We claim that $\est{\kappa}_{i,j}$ satisfies a \textbf{(RC)}. Since $P_1, \ldots, P_L$ satisfy \textbf{(A$''$)} and by assumption \textbf{(D)} $Q_j \neq \p{i}$, by Lemma \ref{sup_cond}, $\p{i}$ satisfies \textbf{(SC)} wrt $Q_j$. Since $\ed{\tilde{P}}_i$ is an empirical distribution, $\ed{\tilde{P}}_i$  satisfies a \textbf{(UDI)}. Since $\est{Q}_j$ is a ResidueHat estimator, $\est{Q}_j$ satisfies a \textbf{(UDI)} by Lemma \ref{main_est_ineq}. Therefore, the hypotheses of Lemma \ref{rate} are satisfied and $\est{\kappa}_{i,j}$ satisfies a \textbf{(RC)}. 

Form the matrix $Z_{i,j} = \ind{\kappa^*(\p{i} \, | \, Q_j) > 0}$ as in Algorithm \ref{vertex_test_alg}. Since $Q_1, \ldots, Q_L$ are a permutation of $P_1, \ldots, P_L$, $\vec{Z}$ is formed by permuting the columns of $\bPi^+$ appropriately. Thus, there are $|\bPi^+|$ $(i,j)$ pairs such that $\kappa_{i,j} > 0$ and the rest are such that $\kappa_{i,j}=0$. Then, using Lemma \ref{rate} and a union bound, with probability tending to $1$ as $\vec{n} \longrightarrow \infty$, $\est{\kappa}_{i,j}$ is among the $|S|$ largest values in the matrix $\est{M}$ if and only if $\kappa_{i,j} > 0$. On this event, $\text{VertexTestHat}(\bPi^+, (\ed{\tilde{P}}_1, \ldots, \ed{\tilde{P}}_M)^T, (\est{Q}_1, \ldots, \est{Q}_L)^T)$ and $\text{VertexTest}(\bPi^+, (\p{1}, \ldots, \p{M})^T, (Q_1, \ldots, Q_L)^T)$ return the same output. Since $(Q_1, \ldots, Q_L)$ is a permutation of $(P_1, \ldots, P_L)$ by hypothesis, by Lemma \ref{vertex_test} $\text{VertexTest}(\bPi^+, (\p{1}, \ldots, \p{M})^T, (Q_1, \ldots, Q_L)^T)$ returns $(1, \vec{C}^T)$ such that $\vec{C}^T (Q_1, \ldots, Q_L)^T = \vec{P}$. The result follows.
\end{proof}

\begin{proof}[Proof of Theorem \ref{partial_label_hat}]
Let $(\est{Q}_1, \ldots, \est{Q}_L) \longleftarrow \text{DemixHat}(\ed{\tilde{P}}_1, \ldots, \ed{\tilde{P}}_M \, | \, \epsilon)$. By Theorem \ref{demix_estim}, w.p. tending towards $1$ as $\bn \longrightarrow \infty$, there exists a permutation $\sigma: [L] \longrightarrow [L]$ such that for every $i \in [L]$,
\begin{align*}
\sup_{E \in \sE} |\est{Q}_i(E) - P_{\sigma(i)}(E)| < \delta.
\end{align*}
From the proof of Theorem \ref{demix_estim}, each $\est{Q}_i$ is a ResidueHat estimator. The assumptions of Lemma \ref{vertex_hat_lemma} are satisfied. The result follows immediately from  Lemma \ref{vertex_hat_lemma}.
\end{proof}

\section{Previous Results}
\label{prev_app}

\begin{lemma}[Lemma A.1 \citep{blanchard2014}]
\label{A_1}
The maximum operation in the definition of $\kappa^*$ and $\widehat{\kappa}$ (lines (\ref{multisample_kappa}) and (\ref{multisample_kappa_estimator}), respectively) is well-defined, that is, the outside supremum is attained at at least one point. 
\end{lemma}

\begin{lemma}[Lemma B.1 \citep{blanchard2014}]
\label{B_1}
If $\Pi$ satisfies \textbf{(B1)}, then $\bpi_1, \ldots, \bpi_L$ are linearly independent. If $P_1, \ldots, P_L$ are jointly irreducible, then they are linearly independent. If $\bpi_1, \ldots, \bpi_L$ are linearly independent and $P_1, \ldots, P_L$ are linearly independent, then $\p{1}, \ldots, \p{L}$ are linearly independent.
\end{lemma}

\section{Experiments}
\label{experiments_section}

In this Section, we perform experiments that suggest that joint irreducibility of $P_1, \ldots, P_L$ is a reasonable assumption. In particular, our experiments suggest that on the datasets in question,  \textbf{(A$''$)} holds (which is a strictly stronger condition than joint irreducibility). We consider three datasets: classes 1, 2, and 3 of MNIST \citep{lecun1998}, the Iris dataset \citep{fisher1936}, and the Breast Cancer Wisconsin (Diagnostic) Data Set \citep{Dua:2017}. We use the Spectral Support Estimation algorithm \citep{vito2010, rudi2014} to estimate the support of each class in each dataset. We split each dataset into training, validation, and test sets, applying the algorithm to the training set, using the validation set to pick the hyperparameters, and evaluating the performance on the test set. We average our results over $60$ trials where in each trial we randomly permute the dataset, thus altering the training, validation, and test sets. Let $\widehat{S}_i$ denote an estimate of the support of class $i$. Tables \ref{cancer_sup_results}, \ref{iris_sup_results}, and \ref{mnist_sup_results} display an estimate of the probability that a point sampled from $P_i$ belongs to the estimate of the support $\widehat{S}_i$. They indicate that the Spectral Support Estimation has reasonably good performance in producing $\widehat{S}_i$s containing the support of the associated class. Tables  \ref{cancer_sep_results}, \ref{iris_sep_results}, and \ref{mnist_sep_results} use the $\widehat{S}_i$ to estimate the quantity $\Pr_{\bx \sim P_i}(\bx \in \cup_{j \neq i} \supp(P_j))$, which must be strictly less than $1$ for \textbf{(A$''$)}  to hold. We find that our estimates are considerably less than $1$, which suggests that joint irreducibility holds on these datasets.

\begin{table}[H]
    \begin{minipage}{.5\linewidth}
      \centering
\scalebox{0.75}{\begin{tabular}{lrrr}
\hline
         &   $i= 1$ &   $i = 2$  \\
\hline
$\widehat{\Pr}_{\bx \sim P_i}(\bx \in \widehat{S}_i)$ &      0.87 &      0.89 \\
\hline
\end{tabular}}
	\caption{Cancer Support Results.}
	\label{cancer_sup_results}
    \end{minipage}%
    \begin{minipage}{.5\linewidth}
      \centering
\scalebox{0.75}{
\begin{tabular}{lrrr}
\hline
         &   $i= 1$ &   $i = 2$  \\
\hline
$\widehat{\Pr}_{\bx \sim P_i}(\bx \in \cup_{j \neq i}  \widehat{S}_j)$ &       0.18 &      0.38 \\
\hline
\end{tabular}}
	\caption{Cancer Separability Results.}
	\label{cancer_sep_results}
    \end{minipage} 
\end{table}

\begin{table}[H]
    \begin{minipage}{.5\linewidth}
      \centering
\scalebox{0.75}{\begin{tabular}{lrrr}
\hline
         &   $i= 1$ &   $i = 2$  & $i = 3$ \\
\hline
$\widehat{\Pr}_{\bx \sim P_i}(\bx \in \widehat{S}_i)$ &       0.86 &       0.84 & 0.84 \\
\hline
\end{tabular}}
	\caption{Iris Support Results.}
	\label{iris_sup_results}
    \end{minipage} 
    \begin{minipage}{.5\linewidth}
      \centering
\scalebox{0.75}{\begin{tabular}{lrrr}
\hline
         &   $i= 1$ &   $i = 2$  & $i = 3$ \\
\hline
$\widehat{\Pr}_{\bx \sim P_i}(\bx \in \cup_{j \neq i} \widehat{S}_j)$ &       0.0 &       0.17 & 0.19 \\
\hline
\end{tabular}}
	\caption{Iris Separability Results.}
	\label{iris_sep_results}
    \end{minipage}%
\end{table}

\begin{table}[H]
    \begin{minipage}{.5\linewidth}
      \centering
\scalebox{0.75}{
\begin{tabular}{lrrr}
\hline
         &   $i= 1$ &   $i = 2$  & $i = 3$ \\
\hline
$\widehat{\Pr}_{\bx \sim P_i}(\bx \in  \widehat{S}_i)$ &       0.98 &         0.87 & 0.83 \\
\hline
\end{tabular}}
	\caption{MNIST Support Results.}
	\label{mnist_sup_results}
    \end{minipage}%
    \begin{minipage}{.5\linewidth}
      \centering
\scalebox{0.75}{\begin{tabular}{lrrr}
\hline
         &   $i= 1$ &   $i = 2$  & $i = 3$ \\
\hline
$\widehat{\Pr}_{\bx \sim P_i}(\bx \in \cup_{j \neq i} \widehat{S}_j)$ &      0.08 &       0.17 & 0.14  \\
\hline
\end{tabular}}
	\caption{MNIST Separability Results.}
	\label{mnist_sep_results}
    \end{minipage} 
\end{table}

\bibliography{references}

\end{document}